\documentclass[11pt]{article}

\usepackage[utf8]{inputenc} 

\pdfoutput=1

\usepackage{algorithm}
\usepackage{algorithmic}
\usepackage{amsmath}
\usepackage{amssymb}
\usepackage{amsthm,amsfonts}
\usepackage{array}
\usepackage{authblk}
\usepackage{babel}
\usepackage{booktabs}
\usepackage{caption}
\usepackage{colortbl}
\usepackage{enumitem}
\usepackage[T1]{fontenc}    
\usepackage{framed}
\usepackage{fullpage}
\usepackage{graphicx}
\usepackage[colorlinks, linkcolor=red, anchorcolor=blue, citecolor=blue]{hyperref}
\usepackage{indentfirst}
\usepackage[utf8]{inputenc}
\usepackage{lipsum}        
\usepackage{makecell}
\usepackage{mathrsfs}
\usepackage{mathtools}

\usepackage[framemethod=default]{mdframed}
\usepackage{microtype}     
\usepackage{multicol}
\usepackage{multirow}
\usepackage{natbib}
\usepackage{nicefrac}       
\usepackage{tablefootnote}
\usepackage{url}            
\usepackage{wrapfig,lipsum}
\usepackage{xcolor}         
\usepackage{xspace}

\usepackage{diagbox}

\usepackage{doi}
\usepackage{tikz}
\usetikzlibrary{decorations.pathreplacing}
\usepackage{amsmath,amsfonts,bm}

\newcommand{\hanze}[1]{{\textcolor{orange}{#1} }}

\def\eqref#1{(\ref{#1})}

\def\1{\bm{1}}

\def\rvb{{\mathbf{b}}}

\def\rvv{{\mathbf{v}}}

\def\rvx{{\mathbf{x}}}
\def\rvy{{\mathbf{y}}}
\def\rvz{{\mathbf{z}}}

\def\EE{{\mathbb{E}}}

\def\vzero{{\bm{0}}}
\def\vone{{\bm{1}}}
\def\vmu{{\bm{\mu}}}
\def\vtheta{{\bm{\theta}}}

\def\vb{{\bm{b}}}

\def\vv{{\bm{v}}}

\def\vx{{\bm{x}}}
\def\vy{{\bm{y}}}
\def\vz{{\bm{z}}}
\def\vmu{{\bm{\mu}}}

\def\mI{{\bm{I}}}

\DeclareMathAlphabet{\mathsfit}{\encodingdefault}{\sfdefault}{m}{sl}
\SetMathAlphabet{\mathsfit}{bold}{\encodingdefault}{\sfdefault}{bx}{n}

\newcommand{\grad}{\ensuremath{\nabla}}

\newcommand{\E}{\mathbb{E}}

\newcommand{\R}{\mathbb{R}}

\newcommand{\KL}[2]{\mathrm{KL}\left(#1 \big\| #2\right)}
\newcommand{\FI}[2]{\mathrm{FI}\left(#1 \| #2\right)}
\newcommand{\TVD}[2]{\mathrm{TV}\left(#1, #2\right)}

\newcommand{\der}{\mathrm{d}}

\DeclareMathOperator*{\argmin}{arg\,min}

\makeatletter
\def\thickhline{%
  \noalign{\ifnum0=`}\fi\hrule \@height \thickarrayrulewidth \futurelet
   \reserved@a\@xthickhline}
\def\@xthickhline{\ifx\reserved@a\thickhline
               \vskip\doublerulesep
               \vskip-\thickarrayrulewidth
             \fi
      \ifnum0=`{\fi}}
\makeatother

\newlength{\thickarrayrulewidth}
\newtheorem{lemma}{Lemma}[section]
\newtheorem{theorem}[lemma]{Theorem}
\newtheorem{corollary}[lemma]{Corollary}

\newtheorem{remark}{Remark}[section]

\newtheorem{definition}{Definition}

\newcommand*\samethanks[1][\value{footnote}]{\footnotemark[#1]}

\newcommand{\const}{C}

\title{Faster Sampling via Stochastic Gradient Proximal Sampler}

\author[$\dagger$]{\normalsize Xunpeng Huang\thanks{Mail to \href{xhuangck@connect.ust.hk}{xhuangck@connect.ust.hk}, \href{dzou@cs.hku.hk} {dzou@cs.hku.hk}}}
\author[$\S$]{Difan Zou\samethanks}
\author[$\P$]{Yian Ma}
\author[$\dagger$]{Hanze Dong}
\author[$\ddag$]{Tong Zhang}

\affil[$\dagger$]{Hong Kong University of Science and Technology}
\affil[$\S$]{The University of Hong Kong}
\affil[$\P$]{University of California San Diego}
\affil[$\ddag$]{University of Illinois Urbana-Champaign}

\begin{document}

\date{}
\maketitle

\begin{abstract}
Stochastic gradients have been widely integrated into Langevin-based methods to improve their scalability and efficiency in solving large-scale sampling problems. However, the proximal sampler, which exhibits much faster convergence than Langevin-based algorithms in the deterministic setting ~\cite{lee2021structured}, has yet to be explored in its stochastic variants. 
In this paper, we study the Stochastic Proximal Samplers (SPS) for sampling from non-log-concave distributions. We first establish a general framework for implementing stochastic proximal samplers and establish the convergence theory accordingly. 
We show that the convergence to the target distribution can be guaranteed as long as the second moment of the algorithm trajectory is bounded and restricted Gaussian oracles can be well approximated. 
We then provide two implementable variants based on Stochastic gradient Langevin dynamics (SGLD) and Metropolis-adjusted Langevin algorithm (MALA), giving rise to SPS-SGLD and SPS-MALA. We further show that SPS-SGLD and SPS-MALA can achieve $\epsilon$-sampling error in total variation (TV) distance within $\tilde{\mathcal{O}}(d\epsilon^{-2})$ and $\tilde{\mathcal{O}}(d^{1/2}\epsilon^{-2})$ gradient complexities, which outperform the best-known result by at least an $\tilde{\mathcal{O}}(d^{1/3})$ factor. This enhancement in performance is corroborated by our empirical studies on synthetic data with various dimensions, demonstrating the efficiency of our proposed algorithm.

\end{abstract}

\section{Introduction}
\label{sec:intro}

Sampling from a target distribution $p_*\propto \exp(-f)$ is a fundamental problem in many research fields such as statistics~\cite{neal1993probabilistic}, scientific computing~\cite{robert1999monte}, and machine learning ~\cite{bishop2006pattern}. Here, $f\colon \R ^d\rightarrow \R$ is referred to as the negative log-density function or energy function of $p_*$.  To solve this problem, the Langevin-based sampling algorithms,  based on discretizing the continuous-time Langevin dynamics, are the most popular choices, including  Unadjusted Langevin Algorithm (ULA) \citep{neal1992bayesian,roberts1996exponential}, Underdamped Langevin Dynamic (ULD)~\citep{cheng2018underdamped,ma2021there,mou2021high}. These algorithms have been extensively investigated both theoretically and empirically. 
Notably, Langevin-based algorithms are usually biased, i.e., the stationary distribution of ULA and ULD (which are also Markov processes), will be different from the target distribution $p_*$, and the error is governed by the discretization step size. 
Thus, Metropolis-adjusted Langevin Algorithm (MALA)~\citep{roberts2002langevin,xifara2014langevin} was designed to resolve this issue.

To achieve the unbiasedness for sampling, 
Proximal sampler, similar to proximal point methods in convex optimization, has been recently developed in \citet{lee2021structured}. In particular, the core idea of the proximal sampler is to first construct a joint distribution 
\begin{equation}
    \label{def:joint_target}
    p_*(\vx,\vy) \propto \exp\big(-f(\vx)-\left\|\vx-\vy\right\|^2/(2\eta)\big)
\end{equation}
whose $\vx$-marginal distribution is the same as $p_*$.
Then, the iterations follow from the two stages:
\begin{itemize}[nosep, leftmargin=*]
    \item From a given $\vx$, sample $\vy|\vx \sim p_{*}(\vy|\vx) = \mathcal{N}(\vx, \mI)$.
    \item From a given $\vy$, sample $\vx|\vy \sim p_{*}(\vx|\vy)$ satisfying 
    \begin{equation*}
        p_*(\vx|\vy)\propto \exp\big(-f(\vx)-\left\|\vx-\vy\right\|^2/(2\eta)\big).
    \end{equation*}
\end{itemize}
It can be noted that the second stage can be easily implemented even in the non-log-concave setting (i.e., $f(\vx)$ is nonconvex), since the target distribution, i.e., $p_*(\vx|\vy)$, is strongly log-concave when $\eta$ is properly small.
Under this condition, the proximal sampler achieves a linear convergence rate for different criteria~\citep{chen2022improved} when the proximal oracle can be accessed.


Despite the impressive performance of proximal samplers in the deterministic setting, where full access to the function $f(\vx)$ and its gradient $\nabla f(\vx)$ is available, their behavior remains largely unexplored in the stochastic setting. In this context, we can only access a stochastic version of $f$ and $\nabla f(\vx)$ at each step.  This is particularly relevant in scenarios where the target distribution $p_*$ is formulated as the posterior of a stochastic process based on multiple observations or training data points. In such cases, the negative log-density function takes the finite-sum form: $f(\vx) = \frac{1}{n}\sum_{i=1}^n f_i(\vx)$, where $n$ denotes the number of observations and $f_i(\vx)$ denotes the corresponding negative log-density function\footnote{
We consider the average for consistency with~\citet{raginsky2017non,zou2021faster}.
}. To reduce the high per-step computational complexity for calculating the full gradient, the mini-batch stochastic gradient has become a standard choice. In the realm of Langevin-based algorithms, extensive research has been conducted on their stochastic counterparts. Various stochastic gradient Langevin algorithms, including stochastic gradient Langevin dynamics (SGLD) \citep{welling2011bayesian} and stochastic gradient ULD (SG-ULD) \citep{cheng2018underdamped}, have been developed. Moreover, the convergence guarantees of these algorithms have been well-established for both log-concave and non-log-concave target distributions.

However, to the best of our knowledge, no prior attempts have been made to study the stochastic gradient proximal sampler, encompassing both algorithm design and theoretical analysis. Consequently, there exists a considerable gap in understanding how the proximal sampler can be effectively adapted to the stochastic setting and what convergence rates can be achieved. This unexplored research question impedes the broader application of the proximal sampler in various tasks, hindering its full potential utilization.


In this paper, we aim to systematically answer this question by providing a comprehensive study of the stochastic gradient proximal sampler. First, we provide a framework for implementing stochastic proximal samplers, the idea is to replace the original joint target distributions with a randomized one:
\begin{equation*}
    p_*(\vx,\vy|\rvb) \propto \exp\left(-f_{\rvb}(\vx)-\left\|\vx-\vy\right\|^2/(2\eta)\right),
\end{equation*}
where $\rvb$ is the stochastic mini-batch that is randomly sampled in different iterations. The two-stage alternating sampling process for $p_*(\vy|\vx,\rvb)$ (a Gaussian-type distribution) and $\rvx$ from $p_*(\vx|\vy,\rvb)$ (sampling a log-concave distribution) will be performed accordingly. By applying different numerical samplers for $p_*(\vx|\vy,\rvb)$, we are able to design various stochastic proximal samplers. Then, we develop the theory to characterize the convergence of the stochastic proximal samplers. The core of our analysis is to sharply quantify the error propagation across multiple iterations. 
In particular, the sampling error within one step stems from (1) inexact target $p_*(\vx|\vy,\rvb)$ caused by stochastic mini-batch; (2) inexact sampling for $p_*(\vx|\vy,\rvb)$ caused by numerical samplers. 
Then, by designing proper initialization when sampling from $p_*(\vx|\vy,\rvb)$, the error propagation can be controlled by the second moment of particles' underlying distributions rather than requiring the stationary points of $f$ as previous analysis~\cite{altschuler2023faster}.
When $p_*$ only satisfies LSI, its negative log-density $f$ will even be nonconvex, which means finding an $\epsilon$-approximate stationary points requires $\mathcal{O}(\epsilon^{-4})$ oracles with stochastic gradient descent, which is unacceptable in sampling tasks.
Besides, by controlling the second moment bound, we provide the gradient complexity expectation for the convergence, which is stronger than a high probability convergence shown in~\citet{altschuler2023faster}.
 Based on our theory, we can develop the convergence guarantees for a variety of stochastic proximal samplers, when the target distribution is log-smooth and satisfies Log-Sobolev Inequality (LSI).  We summarize the main contributions of this paper as follows:


\begin{itemize}[leftmargin=*]
    \item We propose a framework for implementing stochastic proximal samplers. We then provide a general theory to characterize the convergence of stochastic proximal samplers for a general class of target distributions (that can be non-log-concave). We show that with feasible choices of the mini-batch size and learning rate, the stochastic proximal samplers provably converge to the target distributions with a small total variation (TV) distance. Notably, compared with~\citet{altschuler2023faster}, our framework is more practical since it does not require the stationary point information of $f$ and replaces the high probability convergence results with expectation ones. 
    \item Based on the developed framework, we consider two implementations of stochastic proximal samplers using SGLD and warm-started MALA for sampling $p_*(\vx|\vy,\rvb)$, giving rise to SPS-SGLD and SPS-MALA algorithms. We prove that in order to achieve $\epsilon$ sampling error in TV distance, the gradient complexities of SPS-SGLD and SPS-MALA are $\tilde{\mathcal{O}}(d\epsilon^{-2})$ and $\tilde{\mathcal{O}}(d^{1/2}\epsilon^{-2})$ respectively. Compared with the state-of-the-art $\tilde {\mathcal{O}}(d^{4/3}\epsilon^{-2})$ results achieved by CC-SGLD \citep{das2023utilising}, the developed stochastic proximal samplers are faster by at least an $\tilde{\mathcal{O}}(d^{1/3})$ factor.
    \item We conduct experiments to compare SGLD with SPS-SGLD, where the latter one is implemented by using SGLD to sample $p_*(\vx|\vy, \vb)$ in the stochastic proximal sampler framework. 
    Empirical results show that SPS-SGLD consistently achieves better sampling performance than vanilla SGLD for various problem dimensions.
\end{itemize}



\if
Markov Chain Monte Carlo (MCMC), serving as a fundamental class of algorithms for sampling from an unnormalized target distribution $p_*\propto \exp(-f)$, has been extensively used in statistics~\cite{neal1993probabilistic}, scientific computing~\cite{robert1999monte}, and machine learning ~\cite{bishop2006pattern}. Within this class, a series of sampling algorithms have been developed based on the Langevin dynamics 
~\citep{langevin1908theory}, including  Unadjusted Langevin Algorithm (ULA) \citep{neal1992bayesian,roberts1996exponential}, Underdamped Langevin Dynamic (ULD)~\citep{cheng2018underdamped,ma2021there,mou2021high}, Metropolis-adjusted Langevin Algorithm (MALA)~\citep{roberts2002langevin,xifara2014langevin}. In particular, these methods can be viewed as performing different types of discretizations for the continuous-time Langevin dynamics, which have been widely investigated both theoretically and empirically.
Moreover, to adapt the target $p_*$ to a broader distribution class, some gradient-based MCMC methods move beyond the continuous Langevin dynamics. 
For example, proximal samplers~\cite{lee2021structured,chen2022improved,altschuler2023faster} constructs a new unbiased Markov chain that extends the fast convergence of MALA from strongly log-concave to isoperimetric properties, i.e., the Poincare inequality (PI) or log Sobolev inequality (LSI).
Additionally, diffusion-based Monte Carlos~\cite{huang2023monte,huang2024faster} discretize the reverse SDE of an Ornstein–Uhlenbeck process starting with $p_*$ which only requires weak or even no isoperimetric properties.
Despite many advantages in various aspects shown in gradient-based MCMC, their heavy dependency on the first-order oracle, i.e., $\grad f$, of the target negative log density is still a concern.

\hanze{ Maybe highlight the goodness of MALA PS here?
1. MCMC is important

2. MALA/PS converges fast

3. large-scale MCMC needs stochastic version, stochastic MALA/PS is missing

4. SPS
}

In large-scale machine learning problems, the target negative log density, denoted by $f$, is often expressed as a finite sum, i.e., $f(\vx) = n^{-1}\sum_{i=1}^n f_i(\vx)$, which involves the average of all $n$ training data points.
When $n$ is extremely large, computing the exact full gradient of $f$ becomes computationally infeasible.
In such cases, the stochastic gradient can be used as an unbiased estimation of $\grad f(\vx)$ in updates of Langevin-based MCMCs, such as stochastic gradient Langevin dynamic (SGLD)~\cite{welling2011bayesian} and stochastic versions of ULD.
Besides the implementation, the theoretical properties of SGLD, especially for the convergence of different criteria, e.g., Wasserstein $2$ distance and KL divergence, are well studied when $p_*$ is strongly log-concave~\cite{dalalyan2019user,durmus2019analysis}, dissipative~\cite{raginsky2017non,zou2021faster} or even isoperimetric~\cite{das2023utilising}.
The analysis is also extended to stochastic ULD~\cite{cheng2018underdamped} and provides the convergence of Wasserstein $2$ distance in the strongly log-concave assumption.

Although gradient-based MCMCs with unbiased Markov operators usually have a faster convergence rate, whether stochastic gradients can be adapted remains unknown.
For MALA, the stochastic gradient will make the Metropolis–Hastings filter complicated and hardly unbounded.
For the proximal samplers, replacing the full gradient with a stochastic can be considered a problem of the feasibility of using an inexact restricted Gaussian oracle.
To maintain the feasibility, the proximal sampler often requires the exact value of the global minimum or stationary points of negative log density $f$~\cite{altschuler2023faster}, or their regularized versions~\cite{chen2022improved}, which makes proximal samplers hard to implement. 
Therefore, in this paper, we aim to solve the following problem:
\begin{quote}
\emph{Can we develop stochastic gradient variants of proximal samplers that surpass the performance of Langevin-based MCMC methods, like SGLD, when dealing with non-log concave targets $p_*$?} 
\end{quote}

We first provide a framework of stochastic proximal samplers by replacing the original joint target distribution with a randomized one, i.e.,
\begin{equation*}
    p_*^\prime(\vx,\vy|\rvb) \propto \exp\left(-f_{\rvb}(\vx)-\frac{\left\|\vx-\vy\right\|^2}{2\eta}\right)
\end{equation*}
where $\rvb$ is the stochastic mini-batch at each iteration.
Then, we alternatively draw $\rvy$ from $p_*^\prime(\vy|\vx,\rvb)$ (a Gaussian-type distribution) and $\rvx$ from $p_*^\prime(\vx|\vy,\rvb)$ with Langevin-based MCMCs.
Compared with standard proximal samplers, the approximate errors are mainly from the (1) inexact target $p_*^\prime(\vx|\vy,\rvb)$ caused by stochastic mini-batch; (2) inexact sampling oracle caused by Langevin-based MCMCs.
Notably, with a benign error propagation guaranteed by the chain rule of KL divergence, the approximate error can be well controlled by choosing a batch with a relatively large size, and the inexact sampling oracle will not dominate the total error with proper initialization and sufficient runs of Langevin-based MCMCs.
\hanze{??? 
Specifically, though the iteration number required to sample from $p_*^\prime(\vx|\vy,\rvb)$ relies on the conditional $\vy$ and is randomized, from an expectation view, it can be controlled by the second moment bound of particles' underlying distribution rather than suffering from the dependency on the maximum norm of iterative particles as in previous analysis.}
With this framework, the unbiased Markov chain from proximal samplers can achieve fast convergence even when $p_*$ is non-log-concave, and the good properties of Langevin-based MCMCs are combined for solving a sampling subproblem with strongly log-concave targets, i.e., $p_*^\prime(\vx|\vy,\rvb)$.
We then summarize the main contributions of this paper as follows:
\begin{itemize}
    \item We propose a new framework for implementing stochastic proximal samplers easily and flexibly, which allows updating particles approximately sampled from $p_*^\prime(\vx|\vy,\rvb)$ (inexact sampling oracles) with different Langevin-based MCMCs, and provides the provable gradient complexity expectation when the total variation (TV) distance is convergent, and the target $p_*$ satisfies log Sobolev inequality (an isoperimetric assumption beyond strong log-concavity).
    \item We provide two detailed implementations of stochastic proximal samplers with SGLD and warm-started MALA for solving the inner sampling subproblems w.r.t. $p_*^\prime(\vx|\vy,\rvb)$. To achieve the TV distance convergence, the gradient complexity of these two implementations will be $\tilde{\mathcal{O}}(d\epsilon^{-2})$ and $\tilde{\mathcal{O}}(d^{1/2}\epsilon^{-2})$ which are faster than state-of-the-art results by an $\tilde{\mathcal{O}}(d^{1/3})$ factor at least.
    \item We implement the stochastic proximal framework with an SGLD inner sampler (ProxSGLD) and conduct experiments to compare it with the vanilla SGLD. ProxSGLD will be significantly better than vanilla SGLD on high-dimensional synthetic data for all learning rates in our settings.
\end{itemize}
\fi
\section{Related Work}
\label{sec:rewo}
This section primarily introduces related work by dividing current gradient-based MCMCs into two categories.
The first one is based on discretizing the continuous Langevin dynamics. For the second type, including proximal samplers, the SDE of particles varies a lot.
Beyond the sampling algorithms, we will also introduce the usage of the proximal operator in optimization and how it relates to the sampling.

\noindent \textbf{Stochastic Gradient Langevin-based Algorithms.} 
To implement Langevin-based MCMCs with stochastic gradient oracles, the first work is stochastic gradient Langevin dynamic (SGLD)~\citet{welling2011bayesian}. \citet{dalalyan2019user} further establishes the convergence guarantee of SGLD in Wasserstein-2 distance for strongly log-concave targets. 
Besides,~\citet{durmus2019analysis} analyzes SGLD from a composite optimization perspective and obtains the convergence of the KL divergence.
To adapt SGLD to a broader class of target distributions beyond log-concavity,~\citet{raginsky2017non,xu2018global} extend the theoretical analysis of SGLD to distributions satisfying dissipative conditions and proves the convergence when using large mini-batch size. This result has been further improved by \citet{zou2021faster}, which establishes the convergence guarantee of SGLD for sampling non-log-concave distributions for arbitrary mini-batch size. More recently, \citet{das2023utilising} develops non-asymptotic Center Limit Theorems to quantify the approximate Gaussianity of the noise introduced by the random batch-based stochastic approximations used in SGLD and its variants, which leads to the best known convergence rate, i.e., $\tilde{\mathcal{O}}(d^{1.5}\epsilon^{-2})$ and $\tilde{\mathcal{O}}(d^{4/3}\epsilon^{-2})$, for distributions satisfying isoperimetric conditions.



\noindent\textbf{Non-Langevin-based Algorithms.}
There are a number of sampling algorithms are designed based on other Markov processes beyond Langevin. To name a few, Hamiltonian Markov Carlo (HMC) \citep{Neal2010MCMC} is designed by simulating the particles' trajectory in the Hamiltonian's system; diffusion-based MCMCs~\citep{huang2023monte,huang2024faster} discretize the reverse process of an Ornstein–Uhlenbeck process that initializes at $p_*$; proximal samplers alternatively sample the marginal distributions of a joint distribution. \citet{dong2022particle} focus on ODE-based sampling.

In theory, the convergence rate of HMC has been established in ~\citet{bou2020coupling, mangoubi2017rapid, mangoubi2018dimensionally, lee2018algorithmic,chen2022optimal, durmus2017convergence, chen2020fast}; which achieves smaller sampling error than ULA for sampling both strongly log-concave and non-log-concave targets. \citet{chen2014stochastic,zou2021convergence} further develops a class of stochastic gradient HMC methods and proves the convergence rates in the strongly log-concave setting. The convergence rates of diffusion-based MCMCs are studied in~\cite{huang2023monte,huang2024faster}, which are demonstrated to be faster than ULA and can be applied to more general settings (e.g., beyond isoperimetric). For the proximal sampler,~\citet{lee2021structured, chen2022improved} provide its linear convergence rate for different criteria under strongly log-concave or isoperimetric conditions when the exact proximal oracle exists.
~\citet{liang2022proximal, altschuler2023faster} further extend the convergence results to some inexact proximal oracles. 

Notably, existing theory for non-
Langevin-based algorithms are mostly developed in the deterministic setting, while the algorithmic implementation and theoretical analysis in the stochastic setting remain largely understudied,
especially when the target distribution is non-log-concave. Our paper provides the first attempts to study the proximal sampler's theoretical and empirical behaviors with only stochastic gradient oracles, which paves the way for exploring other non-Langevin-based algorithms in the stochastic setting.

\noindent\textbf{Applications of the Proximal Operator.} Before applying the proximal operator to the sampling algorithms, it is introduced in optimization by the proximal point method~\cite{lemarechal2009extension,lemarechal1978nonsmooth,liang2021proximal, liang2023unified, mifflin1982modification,rockafellar1976monotone,wolfe2009method}. 
The proximal point method for minimizing the objective function $f$ is the iteration of the proximal mapping
\begin{equation*}
    \mathrm{prox}_{\eta f}(\vy)\coloneqq \argmin_{\vx\in\R^d} \big\{f(\vx)+ \|\vx-\vy\|^2/(2\eta)\big\}
\end{equation*}
with proper choice of $\eta$.
Using the correspondence $f$ and $\exp(-f)$ between optimization and sampling, the proximal sampler can be viewed as a sampling counterpart of the proximal point method in optimization~\cite{rockafellar1976monotone}.

\section{Proposed Framework}
\label{sec:framw}
This section will first introduce the notations commonly used in the following sections.
Then, we will specify the assumptions that the target distribution $p_*$ is required in our algorithms and analysis.
After that, the proposed framework and some fundamental properties, such as the error propagation control when sampling from an inexact conditional density $p_*^\prime(\vx|\vy)$, will be shown.

\begin{table*}[t]
    \scriptsize
    \centering
    \begin{tabular}{ccccc}
    \toprule
     Results & Algorithm & Assumptions & Metric & Complexity \\
     \midrule
     \citet{raginsky2017non} & SGLD & Dissipative, Component Smooth & $\mathcal{W}_2$ & $\tilde{\mathcal{O}}(\mathrm{poly}(d) \epsilon^{-4})$\\
     \midrule
     \citet{zou2021faster} & SGLD & Dissipative, Warm Start, Component Smooth & $\mathrm{TV}$ & $\tilde{\mathcal{O}}(d^4 \epsilon^{-2})$\\
     \midrule
     \citet{das2023utilising} & AB-SGLD & LSI, Finite-Sum, Smooth  &  $\mathrm{TV}$ & $\tilde{\mathcal{\mathcal{O}}}(d^{3/2} \epsilon^{-2})$\\
     \midrule
     \citet{das2023utilising} & CC-SGLD & LSI, $6^{\mathrm{th}}$ moment, Smooth &  $\mathrm{TV}$ & $\tilde{\mathcal{O}}(d^{4/3} \epsilon^{-2})$\\
     \midrule
     Theorem~\ref{thm:conv_gra_comp_innerSGLD_informal} & SPS-SGLD & LSI, Finite-Sum, Component Smooth & $\mathrm{TV}$ & \textcolor{red}{$\tilde{\mathcal{O}}(d\epsilon^{-2})$}\\
     \midrule
     Theorem~\ref{thm:conv_gra_comp_innerMALA_informal} & SPS-MALA & LSI, Finite-Sum, Component Smooth & $\mathrm{TV}$ & \textcolor{red}{$\tilde{\mathcal{O}}(d^{1/2}\epsilon^{-2})$}\\
     \bottomrule 
    \end{tabular}
    \caption{Comparison with prior works for SGLD. $d$ and $\epsilon$ mean the dimension and error tolerance. Note that we do not list the assumptions about the stochastic gradient since they vary greatly in different references, which will be discussed in our detailed theorems.
    The results of our theorem based on~\ref{con_ass:var_bound} and $\sigma^2 = \Theta(1)$. 
    Compared with the state-of-the-art result, the sampling methods with the stochastic proximal sampler have a better convergence rate with an $\tilde{\mathcal{O}}(d^{1/3})$ factor at least.} 
    \label{tab:comp_old}
\end{table*}

\subsection{Notations and Assumptions}
\label{sec:no_and_ass}
We suppose the target distribution, i.e., $p_*\propto \exp(-f)$ with a finite sum negative log-density, which means
\begin{equation}
    \label{def:total_loss}
    f(\vx)\coloneqq \frac{1}{n}\sum_{i=1}^n f_i(\vx)\quad \mathrm{where}\quad \forall i,  f_i\colon \R^d\rightarrow \R.
\end{equation}
We use letters, e.g., $\vx$ and $\rvx$, to denote vectors and random vectors in $\R^d$ except for letters $\vb$ and $\rvb$, which denote sets and randomized sets.
The function $f_{\vb}$ denotes the energy function deduced by mini-batch $\vb$, i.e.,
\begin{equation}
    \label{def:minibatch_loss}
    f_{\vb}(\vx)\coloneqq \frac{1}{|\vb|}\sum_{i\in \vb}f_i(\vx)\quad \mathrm{where}\quad \vb\subseteq \{1,2,\ldots\, n\},
\end{equation}
and $\grad f_{\vb}$ is the corresponding mini-batch gradient. 
The notation $|\cdot|$ denotes the $L_1$ norm or the number of elements when the inner notation is a vector or a set, respectively.
The Euclidean norm (vector) and its induced norm (matrix) are denoted by $\|\cdot \|$.
For distributions $p$ and $q$, we use $\TVD{p}{q}$ and $\KL{p}{q}$ to denote their TV distance and KL divergence, respectively.

Then, we show the assumptions required for $p_*$:
\begin{enumerate}[label=\textbf{[A{\arabic*}]},nosep]
    \item  
    \label{con_ass:lips_loss}(Component Smooth) For any $i\in\{1,2,\ldots, n\}$, the gradient of $f_i$ is $L$-smooth, which means
    \begin{equation*}
        \left\|\grad f_i(\vx) - \grad f_i(\vy)\right\| \le L\left\|\vx-\vy\right\|.
    \end{equation*}
    \item \label{con_ass:lsi}(Log-Sobolev Inequality) The target distribution $p_*$ satisfies the following inequality 
    \begin{equation*}
        \begin{aligned}
            &\E_{p_*}\left[g^2\log g^2\right]-\E_{p_*}[g^2]\log \E_{p_*}[g^2]
            \le \frac{2}{\alpha_*} \E_{p_*}\left\|\grad g\right\|^2
        \end{aligned}
    \end{equation*}
    with a constant $\alpha_*$ for all smooth function $g\colon \R^d\rightarrow \R$ satisfying $\E_{p_*}[g^2]<\infty$.
    \item \label{con_ass:var_bound}(Bounded Variance) For any $\vx\in\R^d$, the variance of stochastic gradients is bounded, i.e., 
    \begin{equation*}
        \frac{1}{n}\sum_{i=1}^n \left\|\grad f_i(\vx)-\grad f(\vx)\right\|^2\le \sigma^2.
    \end{equation*}
\end{enumerate}
The component smoothness of the finite sum loss, i.e.,~\ref{con_ass:lips_loss}, is also required in~\citet{raginsky2017non,zou2021faster}.
\ref{con_ass:lsi} is a kind of isoperimetric condition~\cite{vempala2019rapid} which is strictly weaker than the strongly log-concave assumption and even the dissipative assumption~\cite{raginsky2017non}.
Besides, it implies the target distribution $p_*$ to have a finite second moment $M$ satisfying $M=\mathcal{O}(d)$, which is demonstrated in Appendix~\ref{sec:not_ass_0x}.
~\ref{con_ass:var_bound} recovers the standard
uniformly bounded variance assumption, i.e., $\sigma=\Theta(1)$, following from~\citet{nemirovski2009robust,ghadimi2012optimal,ghadimi2013stochastic}, and sampling references sometimes allow $\sigma^2 = \Theta(d)$, e.g., ~\citet{raginsky2017non, dalalyan2019user, das2023utilising}.
Both of these cases will be considered in our analysis.

\begin{algorithm}[t]
    \caption{Stochastic Proximal Sampler}
    \label{alg:sps}
    \begin{algorithmic}[1]
            \STATE {\bfseries Input:} The negative log density $f$ of the target distribution, the initial particle $\rvx_0$ drawn from $p_0$;
            \FOR{$k=0$ to $K-1$}
                \STATE Sample $\hat{\rvx}_{k+1/2}$ from $\hat{p}_{k+1/2|k}(\cdot|\rvx_k)$;
                \STATE Draw the mini-batch $\rvb_k$ from $\{1,2,\ldots, n\}$;
                \STATE Sample $\hat{\rvx}_{k+1}$ from  $\hat{p}_{k+1|k+1/2,b}(\cdot|\hat{\rvx}_{k+1/2},\rvb_k)$;
            \ENDFOR
            \STATE {\bfseries Return:} $\hat{\rvx}_{K}$.
    \end{algorithmic}
\end{algorithm}

\subsection{Stochastic Proximal Sampler}
The stochastic proximal sampler (SPS) framework is shown in Alg.~\ref{alg:sps}.
With the common notations introduced in Section~\ref{sec:no_and_ass}, we will explain $\hat{p}_{k+1/2|k}(\cdot|\rvx_k)$ and $\hat{p}_{k+1|k+1/2,b}(\cdot|\rvx_{k+1/2},\rvb_k)$, that are similar to standard proximal samplers.
Considering a joint target distribution 
\begin{equation}
    \label{def:joint_target_rand}
    p_*(\vx,\vy) \propto \exp\bigg(-f_{\rvb}(\vx)-\frac{\left\|\vx-\vy\right\|^2}{2\eta}\bigg)
\end{equation}
that is defined by the randomized mini-batch $\rvb$ and the outer loop step size $\eta$, then Alg.~\ref{alg:sps} samples from $p_*^\prime(\vy|\vx)$ and $p_*^\prime(\vx|\vy)$ alternatively.
Specifically, at iteration $k$, suppose $\vx=\vx_k$, $\vy=\vx_{k+1/2}$ and $\eta=\eta_k$, the conditional probability density $p_*^\prime(\vx_{k+1/2}|\vx_k)$ is equivalent to 
\begin{equation}
    \small
    \label{def:transition_kernel_stage1}
    p_{k+\frac{1}{2}|k}(\vx^\prime|\vx) \propto \exp\left(-\frac{\left\|\vx^\prime - \vx\right\|^2}{2\eta_k}\right),
\end{equation}
which can be exactly implemented by Line 3 of Alg.~\ref{alg:sps} due to its Gaussianity.
Besides, suppose $\vx=\vx_{k+1}$ and $\vy=\vx_{k+1/2}$, the transition kernel $p_*^\prime(\vx_{k+1}|\vx_{k+1/2})$ can be reformulated as
\begin{equation}
    \small
    \label{def:transition_kernel_stage2}
    \begin{aligned}
        p_{k+1|k+\frac{1}{2},b}(\vx^\prime|\vx, \vb) \propto \exp\left(-f_\vb(\vx^\prime)-\frac{\left\|\vx^\prime - \vx\right\|^2}{2\eta_k}\right),
    \end{aligned}
\end{equation}
which is desired to be implemented with Line 5 of Alg.~\ref{alg:sps}.
Rather than exactly sampling from a target distribution, e.g., $p_{k+1|k+\frac{1}{2},b}(\vx^\prime|\vx,\vb)$, most samplers can only generate approximate samples that are close to the target ones in real practice.
Therefore, we consider a Markov process $\{\hat{\rvx}_k\}$ whose underlying distribution is defined as $\hat{p}_k$.
Given the same initialization $\hat{p}_0 = p_0$, we denote the two empirical transition kernels as $\hat{p}_{k+\frac{1}{2}|k} \coloneqq p_{k+\frac{1}{2}|k}$ and $\hat{p}_{k+1|k+\frac{1}{2},b}(\cdot|\vx,\vb)$ that satisfies
\begin{equation}
    \label{def:transition_kernel_norgo}
    \begin{aligned}
        \KL{\hat{p}_{k+1|k+\frac{1}{2},b}(\cdot|\vx,\vb)}{{p}_{k+1|k+\frac{1}{2},b}(\cdot|\vx,\vb)} \le \delta_k.
    \end{aligned}
\end{equation}
Here we assume that the conditional distribution of $\hat{\rvx}_{k+1}$ given $\hat{\rvx}_{k+1/2}$ is close to the ideal conditional distribution  $p_{k+1|k+1/2,b}(\vx^\prime|\vx, \vb)$ with up to $\delta_k$ approximation error in KL divergence. 
In fact, as the distribution  $p_{k+1|k+1/2,b}(\vx^\prime|\vx, \vb)$ is strongly log-concave when $\eta_k$ is properly chosen, the condition Eq.~\ref{def:transition_kernel_norgo} can be achieved by applying standard numerical samplers such as SGLD and MALA with provable guarantees (detailed implementations will be discussed in the next section). 

Then, the following theorem characterizes the error propagation across multiple steps and provides general results on the sampling error achieved by Alg.~\ref{alg:sps}.

\begin{theorem}
    \label{thm:comp_conv_sps}
    Suppose Assumption~\ref{con_ass:lips_loss}-\ref{con_ass:var_bound} hold, and Alg.~\ref{alg:sps} satisfies:
    \begin{itemize}[leftmargin=*]
        \item We have $\eta_k\le \frac{1}{2L}$ for all $k\in\{0,1,\ldots, K-1\}$.
        \item The initial particle $\hat{\rvx}_0$ is drawn from the standard Gaussian distribution on $\R^d$.
        \item Line 5 is implemented by some specific inner sampler, achieving
            \begin{equation*}
                \KL{\hat{p}_{k+1|k+\frac{1}{2},b}(\cdot|\vx,\vb)}{{p}_{k+1|k+\frac{1}{2},b}(\cdot|\vx,\vb)} \le \delta_k
            \end{equation*}
        for all $k\in\{0,1,\ldots, K-1\}$.
    \end{itemize}
    Then, we have
    \begin{equation}
        \label{ineq:sps_tv_bound}
        \begin{aligned}
            \TVD{\hat{p}_{K}}{p_*} \le \sqrt{\frac{1}{2}\sum_{i=0}^{K-1} \delta_i} +  \sigma \sqrt{\sum_{i=0}^{K-1} \frac{\eta_i}{2|\rvb_i|}} 
            + \sqrt{\frac{(1+L^2)d}{4\alpha_*}} \cdot \prod_{i=0}^{K-1} \left(1+\alpha_* \eta_i\right)^{-1}.
        \end{aligned}
    \end{equation}
\end{theorem}
Theorem~\ref{thm:comp_conv_sps} provides the general upper bound of the TV distance between the underlying distribution of particles returned by Alg.~\ref{alg:sps} and the target distribution $p_*$. 
The first term in Eq~\ref{ineq:sps_tv_bound} represents the accumulated error of the inexact sampling from $p_{k+1|k+\frac{1}{2},b}(\cdot|\vx,\vb)$, i.e., Line 5 of Alg~\ref{alg:sps}.
The second term represents the approximation error using stochastic gradients, and the last term represents the error from deterministic proximal samplers.
To achieve an $\epsilon$-TV distance to the target distribution $p_*$, one may have to choose a small error tolerance of inexact sampling, i.e., $\delta_k=\epsilon^2$, to control the first term of Eq~\ref{ineq:sps_tv_bound}.
Besides, it still requires a large enough mini-batch size, i.e., $|\rvb_i|=\Theta(1/(\sigma \epsilon)^2)$ and the mixing time, i.e., $\sum_{i=0}^{K-1}\eta_i = \Theta(\log(1/\epsilon))$, to make the last two terms of Eq~\ref{ineq:sps_tv_bound} small, respectively.

Notably, the implementation of the proximal sampler in~\citet{altschuler2023faster} also allows inexact sampling from $p_{k+1|k+\frac{1}{2},b}(\cdot|\vx,\vb)$ in the second stage update, and requires the underlying distribution of returned particles, i.e., $\hat{p}_{k+1|k+\frac{1}{2},b}(\cdot|\vx,\vb)$ to satisfy Eq.~\ref{def:transition_kernel_norgo} with a small $\delta_k$.
However, they only consider the deterministic setting, i.e., $\vb=\{1,2,\ldots,n\}$, and requires initializing Line 5 of Alg.~\ref{alg:sps} with certain stationary points $\vx_*$ of $f$.
Hence, directly applying their analysis may require finding stationary points in each iteration, as the function $f_{\vb}$ changes, which may take substantially more time.
This is because, when $p_*$ only satisfies LSI, the function $f_{\vb}$ may not be convex.
Finding an $\epsilon$-approximate stationary point of a general non-convex function requires $\mathcal{O}(\epsilon^{-4})$~\cite{nesterov2013introductory} for stochastic gradient descent, which is unacceptable in sampling algorithms.
Therefore, the implementation of~\citet{altschuler2023faster} still remains a concern without exact information, or even only with inexact information, about the stationary points of $f$. 

In our analysis, combining proper Langevin-based MCMC with a $\hat{\rvx}_{k+1/2}$ mean Gaussian-type initialization, the gradient complexity for achieving Eq.~\ref{def:transition_kernel_norgo} will only depend on $\log\|\hat{\rvx}_{k+1/2}\|^2$ rather than stationary points $\vx_*$, which will be explicitly shown in the next section. Considering the expected gradient complexity, it requires to characterize $\E_{\hat{p}_{k+1/2}}[\log\|\hat{\rvx}_{k+1/2}\|^2]$, which can be readily upper bounded by $\log[\E_{\hat{p}_{k+1/2}}[\|\hat{\rvx}_{k+1/2}\|^2]]$. This implies that we further need to control the second moment of the particles. This is conducted in the following lemma.
\begin{lemma}
    \label{lem:2ndmoment_bound}
    Suppose Assumption~\ref{con_ass:lips_loss}-\ref{con_ass:var_bound} hold, and the second moment of the underlying distribution of $\hat{\rvx}_{k}$ is $M_k$, then we have 
    \begin{equation*}
        M_{k+1} \le 24 M_k + 4\eta_k \delta_k + \frac{24\eta_k^2 \sigma^2}{|\rvb|}+ 28M  + 24\eta_k d.
    \end{equation*}
\end{lemma}

This bound may seem to be large as $M_k$ exhibit an exponential increasing rate. However, we remark that only $\log(M_k)$ will appear in our calculation of the gradient complexity rather than $M_k$ itself. Then, let $K$ be the number of total steps, which can be chose to be $\tilde{\mathcal{O}}(L/\alpha^*)$, then $M_K$ will be controlled by $\exp(K)$ and so that $\log(M_k)$ can be controlled by $K=\tilde{\mathcal{O}}(L/\alpha^*)$, which will not heavily affect the total gradient complexity.


\section{Implementations of SPS}
\label{sec:imple}

This section mainly focuses on the detailed implementation of the SPS. 
Specifically, since the target $\hat{p}_{k+1/2|k}$ of Line 3 of Alg.~\ref{alg:sps} is a Gaussian-type distribution shown as Eq.~\ref{def:transition_kernel_stage1}, we can obtain the sample exactly. Then, the key step is to numerically sample from the distribution $p_{k+1|k+1/2,b}$ to ensure that the distribution of the approximate samples, i.e., $\hat{p}_{k+1|k+1/2,b}$ satisfies Eq.~\ref{def:transition_kernel_norgo}. 
In particular, we will implement this step, i.e., Line 5 of Alg.~\ref{alg:sps} using two inner samplers: stochastic gradient Langevin dynamics (SGLD) and warm-started Metropolis-adjusted Langevin Algorithm (MALA), which give rise to two stochastic proximal sampling algorithms. In what follows, we will introduce the implementation details of these two algorithms and prove their gradient complexities, i.e., the desired number of stochastic gradient calculations to guarantee $\epsilon$ sampling error.

\begin{algorithm}[t]
    \caption{Inner Stochastic Gradient Langevin Dynamics: $\mathsf{InnerSGLD}(\vx_0, \vb, \eta, \delta)$}
    \label{alg:sgld_inner}
    \begin{algorithmic}[1]
            \STATE {\bfseries Input:} The output particle $\vx_0$ of Alg.~\ref{alg:sps} Line 3, the selected mini-batch $\vb$, the step size of outer loop $\eta$, the required accuracy of the inner loop $\delta$;
            \STATE Initialized the returned particle $\overline{\rvz}=\vzero$;
            \STATE Draw the initial particle $\rvz_0$ from $\mathcal{N}(\vx_0,\eta\cdot \mI)$
            \FOR{$s=0$ to $S-1$}
                \STATE Draw the mini-batch $\rvb_s$ from $\vb$;
                \STATE Update the particle $$\rvz^\prime_s \gets \rvz_{s}+\sqrt{2\tau_s\cdot \left(1-\frac{\tau_s}{4\eta}\right)^{-1}}\xi \quad $$
                where $\xi \sim \mathcal{N}(\vzero, \mI)$;
                \STATE Update the particle $$\rvz_{s+1} \gets \rvz_s^\prime - \tau_s \cdot \left(\grad f_{\rvb_s}\left(\rvz_s^\prime\right)+\eta^{-1}\cdot \left(\rvz_s^\prime-\vx_0\right)\right);$$
                \IF {$s> S^\prime$}
                    \STATE Update the returned particle: 
                    $$\overline{\rvz} \gets \overline{\rvz} + \rvz_s^\prime/{(S-S^\prime+1)};$$
                \ENDIF 
            \ENDFOR
            \STATE {\bfseries Return:} $\overline{\rvz}$.
    \end{algorithmic}
\end{algorithm}

\subsection{SGLD Inner Sampler}
\label{sec:sgld_inner}
We consider implementing Line 5 of Alg.~\ref{alg:sps} with SGLD inner sampler shown in Alg.~\ref{alg:sgld_inner}, and name it SPS-SGLD. 
We point out that the particle update of Alg.~\ref{alg:sgld_inner} is slightly different from the standard SGLD update.
In particular, our update is performed with two steps and returns a trajectory average, computed using the last $S-S'$ iterations, rather than a single particle. The first step of the update, i.e., Line 6 of Alg. ~\ref{alg:sgld_inner} performs the diffusion via the Gaussian process, and the second step, i.e., Line 7 of Alg.~\ref{alg:sgld_inner} updates the particle via drift term $\grad\log \hat{p}_{k+1|k+1/2,b}$. 
With this implementation, we show the gradient complexity for approaching the target $p_*$ in the following theorem.
\begin{theorem}
    \label{thm:conv_gra_comp_innerSGLD_informal}
    Suppose~\ref{con_ass:lips_loss}-\ref{con_ass:var_bound} hold. 
    With proper parameter settings at the following levels
    \begin{equation*}
        \begin{aligned}
            &\eta_k = \Theta(L^{-1}),\quad K = \tilde{\Theta}(\kappa),\quad \delta_k = \tilde{\Theta}(\kappa^{-1} \epsilon^2),\\
            &\mathrm{and}\quad b_o = \min\left\{\tilde{\Theta}(\alpha_*^{-1}\sigma^2\epsilon^{-2}),n\right\},
        \end{aligned}
    \end{equation*}
    where $\kappa = L/\alpha_*$ for Alg.~\ref{alg:sps}, if we choose Alg.~\ref{alg:sgld_inner} as the inner sampler shown in Line 5 Alg.~\ref{alg:sps}, set 
    \begin{equation*}
        \small
        \begin{aligned}
            &\tau = \min\left\{ \tilde{\Theta}(\kappa^{-1}\epsilon^2(d+\sigma^2)^{-1}), \frac{1}{36}\right\},\ \tau^\prime = \min\left\{\tilde{\Theta}(L^{-1}\tau), \frac{1}{36}\right\},\\
            & S^\prime = \tilde{\Theta}(L^{-1}\tau^{-1}),\quad \tau_s = \tau \quad \mathrm{when}\quad s\in[0,S^\prime],\\
            &S = \tilde{\Theta}(S^\prime+(\tau^\prime)^{-1}),\quad \tau_s = \tau^\prime \quad \mathrm{when}\quad s\in[S^\prime+1, S-1],
        \end{aligned}
    \end{equation*}
    and inner minibatch sizes satisfy $|\rvb_s|=1$, for all $s\in \{0,1,\ldots S-1\}$, the distribution of returned particles $\hat{p}_K$ in Alg.~\ref{alg:sps} satisfies $\TVD{\hat{p}_{K}}{p_*}<3\epsilon$. 
    In this condition, the expected gradient complexity will be $\tilde{\Theta}(\kappa^3 (d+\sigma^2)\epsilon^{-2})$.
\end{theorem}
Due to the space limitation, we only show an informal result in this section, and the formal version will be deferred to Theorem~\ref{thm:conv_gra_comp_innerSGLD} in Appendix~\ref{app_sec:innersgld}.
Theorem~\ref{thm:conv_gra_comp_innerSGLD_informal} provides an $\tilde{\mathcal{O}}(d\epsilon^2)$ gradient complexity regardless of $\sigma^2=\Theta(d)$ or $\sigma^2=\Theta(1)$.
When $\sigma^2=\Theta(d)$, the state-of-the-art results are $\tilde{\mathcal{O}}(d^{3/2}\epsilon^{-2})$ and $\tilde{\mathcal{O}}(d^{4/3}\epsilon^{-2})$ under stronger variance assumptions~\cite{das2023utilising}.  
Compared with those results provided in~\citet{das2023utilising}, our SPS-SGLD is faster by at least an $\tilde{\mathcal{O}}(d^{1/3})$ factor with strictly weaker assumptions.
When $\sigma^2=\Theta(1)$, the gradient complexity provided in~\citet{das2023utilising} will become $\tilde{\mathcal{O}}(d\epsilon^2)$ which is the same as our results.

Notably, in the proof of Theorem~\ref{thm:conv_gra_comp_innerSGLD}, we demonstrate that, with the Gaussian type initialization shown in Line 3 of Alg.~\ref{alg:sgld_inner},  the relative Fisher information gap between the underlying distribution of $\rvz_0$ and the target distribution $\hat{p}_{k+1|k+1/2,b}$ can be upper bounded with a factor $\log (\|\vx_0\|^2 + \|\grad f_{\vb}(\vzero)\|^2)$ which is independent of stationary points of $f$ and can be controlled by second moment with Lemma~\ref{lem:2ndmoment_bound} and variance of stochastic gradients from an expectation perspective.
This means the SPS-SGLD can be easily implemented without initialization issues in previous work, e.g.,~\citet{altschuler2023faster}. 

\subsection{Warm-started MALA Inner Sampler}

\begin{algorithm}[t]
    \caption{Inner Metropolis-adjusted Langevin algorithm: $\mathsf{InnerMALA}(\vx_0, \vb, \eta, \delta)$}
    \label{alg:mala_inner}
    \begin{algorithmic}[1]
            \STATE {\bfseries Input:} The output particle $\vx_0$ of Alg.~\ref{alg:sps} Line 3, the selected mini-batch $\vb$, the step size of outer loop $\eta$, the required accuracy of the inner loop $\delta$;
            \STATE Draw the initial sampler $\rvz_0$ from $\mathsf{InnerULD}(\vx_0, \vb, \eta)$ by Alg.~\ref{alg:ULD_inner}
            \FOR{$s=0$ to $S-1$}
                \STATE Draw $\rvz^\prime_s$ from $\mathcal{N}(\rvz_s - \tau_s \cdot \grad g(\rvz_s), 2\tau_s \mI)$;
                \STATE Define the threshold $p$ to be 
                \begin{equation*}
                    p \coloneqq \min\left\{1, \frac{\exp\left(g(\rvz_s)+\varphi(\rvz_s^\prime;\rvz_s,\tau_s)\right)}{\exp\left(g(\rvz_s^\prime)+\varphi(\rvz_s;\rvz_s^\prime,\tau_s)\right)}\right\};
                \end{equation*}
                \STATE Draw the sample $p^\prime$ uniformly from $[0,1]$;
                \IF {$p^\prime \le p$}
                    \STATE Update the particle $\rvz_{s+1}\gets \rvz_s^\prime$
                \ELSE 
                \STATE Update the particle $\rvz_{s+1}\gets \rvz_s$
                \ENDIF
            \ENDFOR
            \STATE {\bfseries Return:} $\rvz_S$.
    \end{algorithmic}
\end{algorithm}

We consider implementing Line 5 of Alg.~\ref{alg:sps} with warm-started MALA  inner sampler shown in Alg.~\ref{alg:mala_inner}, and name it SPS-MALA where the functions $g(\vz)$ and $\psi(\vz';\vz,\tau)$ are defined as follows:
\begin{equation*}
    \small
    \begin{aligned}
        &g(\vz) \coloneqq -\log p_{k+1|k+\frac{1}{2},b}(\vz|\vx_0, \vb)=  f_{\vb}(\vz) + \frac{\left\|\vz-\vx_0\right\|^2}{2\eta},\\
        &\varphi(\vz^\prime; \vz,\tau)\coloneqq \frac{\left\|\vz^\prime - \left(\vz - \tau\grad g(\vz)\right)\right\|^2}{4\tau}.
    \end{aligned}
\end{equation*}
Inspired by~\citet{altschuler2023faster}, SPS-MALA requires InnerULD to provide warm starts, i.e., Line 2 of Alg~\ref{alg:mala_inner}, where we defer the implementation of InnerULD to Appendix~\ref{sec:not_ass_0x}.
Compared with general initialization, the gradient complexity MALA can be improved from $\tilde{\mathcal{O}}(d)$ to $\tilde{\mathcal{O}}(d^{1/2})$ with warm starts, and ULD can provide warm starts within $\tilde{\mathcal{O}}(d^{1/2})$ gradient complexity.
It means InnerMALA will be faster than InnerSGLD by an $\tilde{\mathcal{O}}(d^{1/2})$ factor to achieve the KL convergence, i.e., Eq.~\ref{def:transition_kernel_norgo}.
Hence, SPS-MALA can be expected to improve the dimensional dependence of SPS-SGLD. 
With this implementation, i.e., Alg.~\ref{alg:mala_inner}, the TV distance convergence
of Alg.~\ref{alg:sps} can be presented in the following:
\begin{theorem}
    \label{thm:conv_gra_comp_innerMALA_informal}
    Suppose \ref{con_ass:lips_loss}-\ref{con_ass:var_bound} hold. 
    With proper parameter settings at the following levels
    \begin{equation*}
        \begin{aligned}
            &\eta_k = \Theta(L^{-1}),\quad K = \tilde{\Theta}(\kappa),\quad \delta_k = \tilde{\Theta}(\kappa^{-1} \epsilon^2),\\
            &\mathrm{and}\quad b_o = \min\left\{\tilde{\Theta}(\alpha_*^{-1}\sigma^2\epsilon^{-2}),n\right\},
        \end{aligned}
    \end{equation*}
    where $\kappa = L/\alpha_*$ for Alg.~\ref{alg:sps}, if we choose Alg.~\ref{alg:mala_inner} as the inner sampler shown in Line 5 of Alg.~\ref{alg:sps}, set 
    \begin{equation*}
        \gamma = \Theta(L^{1/2}),\quad \tau=\tilde{\Theta}(L^{-1/2}d^{-1/2}),\quad \mathrm{and}\quad S= \tilde{\Theta}(d^{1/2}).
    \end{equation*}
    for Alg.~\ref{alg:ULD_inner}, and
     \begin{equation*}
        \begin{aligned}
        &\tau = \tilde{\Theta}(L^{-1}d^{-1/2}),\quad  \mathrm{and}\quad S=\tilde{\Theta}(d^{1/2})
        \end{aligned}
    \end{equation*}
    for Alg.~\ref{alg:mala_inner}, then the underlying distribution of returned particles $\hat{p}_K$ in Alg.~\ref{alg:sps} satisfies $\TVD{\hat{p}_{K}}{p_*}<3\epsilon$. 
    In this condition, the expected gradient complexity will be $\tilde{\Theta}\left(\kappa^3 d^{1/2}\sigma^2\epsilon^{-2}\right)$.
\end{theorem}

\begin{figure*}[t]
    \centering
    \begin{tabular}{cccc}
          \includegraphics[width=3.2cm]{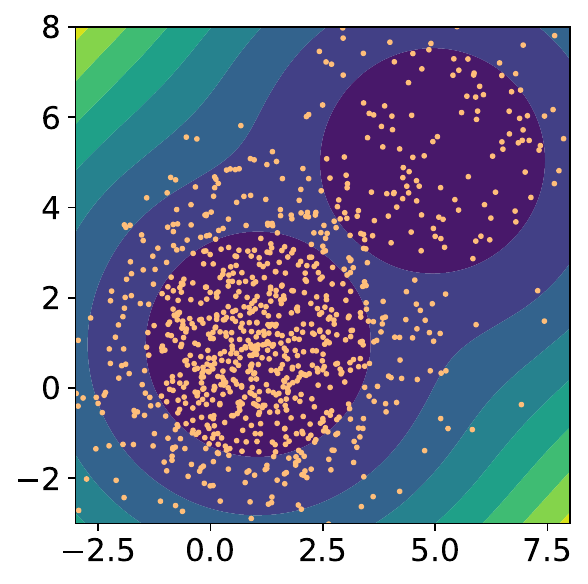}&
          \includegraphics[width=3.2cm]{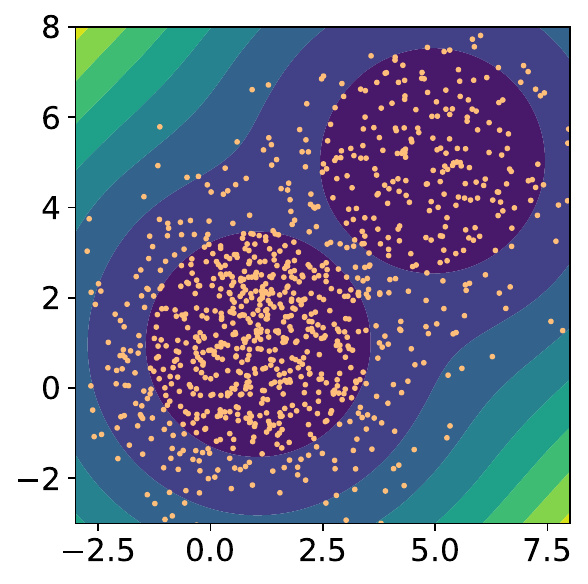}&
          \includegraphics[width=3.2cm]{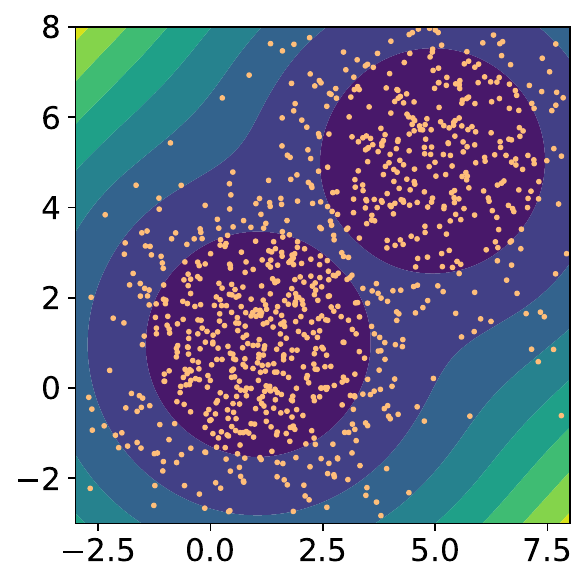}&
          \includegraphics[width=4cm]{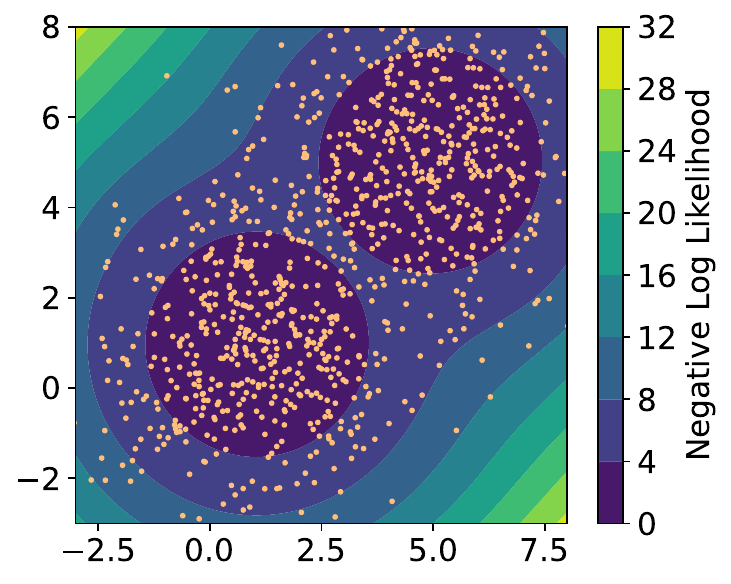}\\
          \begin{small} $\mathrm{GradNum} = 250$ \end{small}& 
          \begin{small} $\mathrm{GradNum} = 500$ \end{small}&
          \begin{small} $\mathrm{GradNum} = 1000$ \end{small} &
          \begin{small} $\mathrm{GradNum} = 2000$ \end{small} \\
          \includegraphics[width=3.2cm]{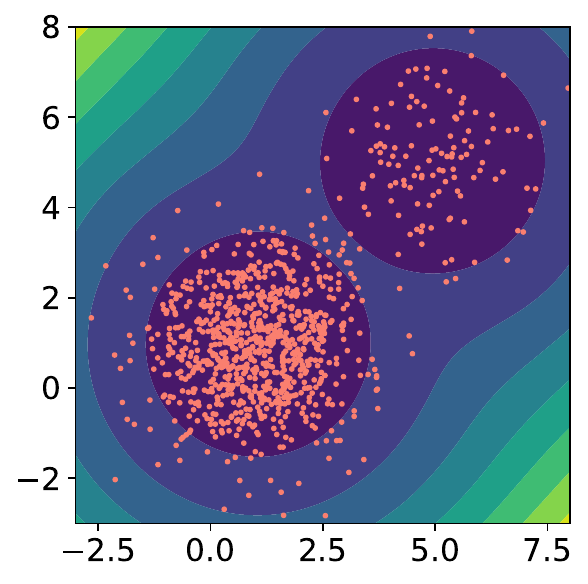}&
          \includegraphics[width=3.2cm]{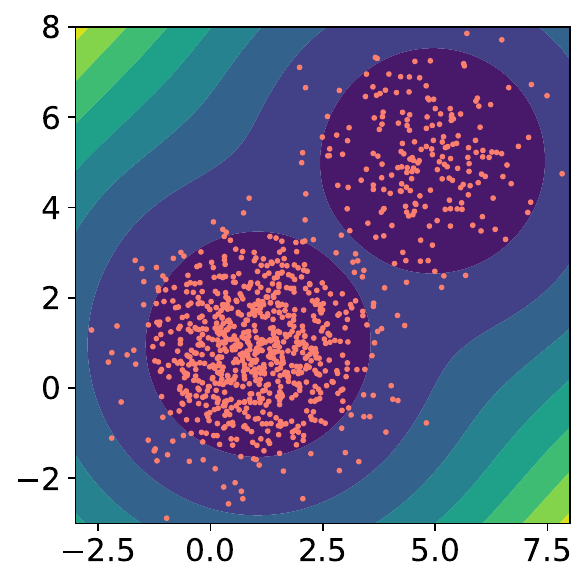}&
          \includegraphics[width=3.2cm]{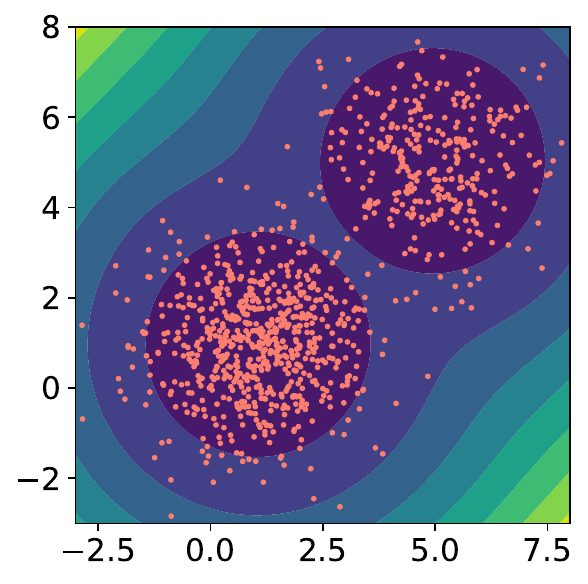}&
          \includegraphics[width=4cm]{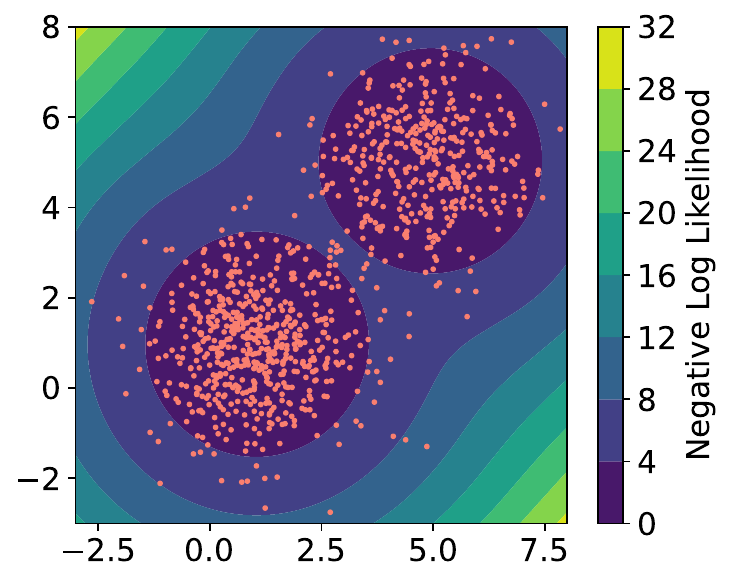}\\
          \begin{small} $\mathrm{GradNum} = 250$ \end{small}& 
          \begin{small} $\mathrm{GradNum} = 500$ \end{small}&
          \begin{small} $\mathrm{GradNum} = 1000$ \end{small} &
          \begin{small} $\mathrm{GradNum} = 2000$ \end{small}
    \end{tabular}
    \caption{\small The background of all graphs is the projection of the negative log density on a $2$d plane, and nodes are the projection of particles returned by different algorithms on the same plane. The first two rows show the distribution of particles' projection after different iterations of SGLD and SPS-SGLD with their optimal step sizes when $d=10$.
    }
    \label{fig:illustration_exp}
\end{figure*}

Due to the space limitation, we only show an informal result in this section, and the formal version will be deferred to Theorem~\ref{thm:conv_gra_comp_innerMALA} in Appendix~\ref{app_sec:innermala}.
Theorem~\ref{thm:conv_gra_comp_innerMALA_informal} provides gradient complexities of $\tilde{\mathcal{O}}(d^{1/2}\epsilon^2)$ and $\tilde{\mathcal{O}}(d^{3/2}\epsilon^2)$ for cases when $\sigma^2=\Theta(1)$ and $\sigma^2=\Theta(d)$, respectively.
When $\sigma^2=\Theta(1)$, the state-of-the-art result is $\tilde{\mathcal{O}}(d\epsilon^{-2})$ under the lin-growth assumption~\cite{das2023utilising}.
Compared with the result provided in~\citet{das2023utilising}, our SPS-MALA is faster by an $\tilde{\mathcal{O}}(d^{1/2})$ factor with strictly weaker assumptions. 
However, the efficiency of SPS-MALA will be greatly affected by the variance, i.e., $\sigma^2$ in~\ref{con_ass:var_bound}, through the mini-batch size of Alg~\ref{alg:sps}.
Even when $\sigma^2=\Theta(d)$, the complexity of SPS-MALA will become $\tilde{\mathcal{O}}(d^{3/2}\epsilon^2)$, which is the same as AB-SGLD shown in Table.~\ref{tab:comp_old} with weaker assumptions.

Besides, it should be noted that~\citet{altschuler2023faster} provides high probability convergence of the TV distance with an $\tilde{O}(n\kappa d^{1/2})$ gradient complexity, while requiring the stationary points of $f$.
Compared with this result, we have an additional $\tilde{O}(\kappa)$ factor besides replacing the number of training data $n$ to the $\tilde{\Theta}(\alpha_*^{-1}\sigma^2\epsilon^{-2})$ batch size.
This factor comes from our proof techniques of removing the dependency of stationary points for SPS framework by upper bounding second moments during the entire Alg.~\ref{alg:sps}, which is demonstrated in Section~\ref{sec:sgld_inner}.

\section{Experiments}
\label{sec:exp}
In this section, we will first provide our experimental settings. 
Then, for a fair comparison with SGLD, we implement the proximal sampler with SPS-SGLD and show their sampling performance with different dimensions.
\begin{figure}[t]
    \centering
    \begin{tabular}{cc}
         \includegraphics[width=6.8cm]{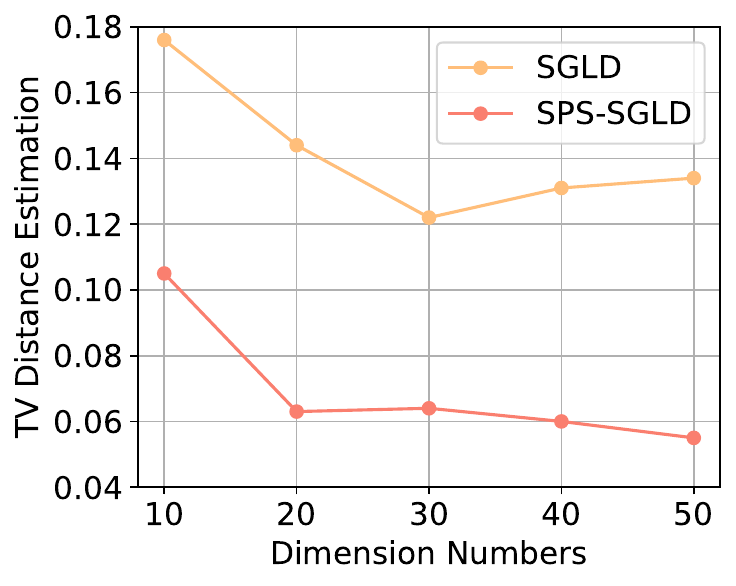} &
         \includegraphics[width=6.8cm]{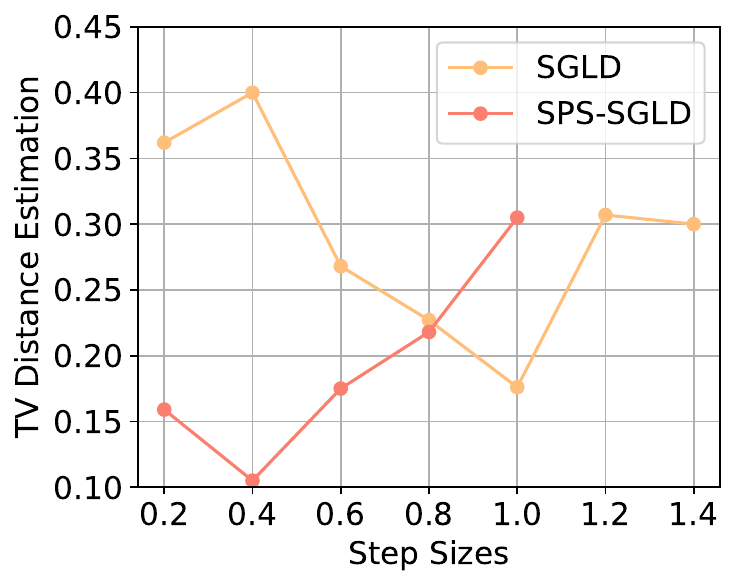}
    \end{tabular}
    \caption{\small The graph in the left column shows the TV distance estimation, i.e., $\mathrm{TV}(\hat{p}_K, p_*)$ when SGLD and SPS-SGLD chose their optimal hyper-parameters under different dimensions. The graph in the right column denotes the TV distance estimation when SGLD and SPS-SGLD chose different step sizes and $d=10$.}
    \label{fig:tv_diff}
\end{figure}

\noindent\textbf{Experimental Settings.} Here, we consider the component $e^{-f_i}$ shares a similar definition in~\citet{zou2019sampling}, i.e., $e^{-f_i(\vx)}\coloneqq e^{-{\left\|\vx - \vb - \vmu_i\right\|^2}/2} + e^{-{\left\|\vx - \vb + \vmu_i\right\|^2}/2}$,
where the number of input data $n=100$, the dimension $d\in\{10, 20, 30, 40, 50\}$, the bias vector $\vb = (3,3,\ldots, 3) \cdot $, and the data input $\sqrt{d/10}\cdot \vmu_i\sim \mathcal{N}(\overline{\vmu}, \mI_{d\times d})$ with $\overline{\vmu} = (2,2,\ldots, 2)$.
Here, we require the input data to shrink with the growth of $d$, which keeps the distances between different modes for each $e^{-f_i}$.
Since~\citet{zou2019sampling} had proven the function $f_i$ is dissipative, which implies the LSI property of $e^{-f_i}$ and $e^{-f}$, we omit the discussion about the property of $f_i$ in this section. 

For the common hyper-parameter settings of SGLD and SPS-SGLD, we fix the number of stochastic gradient oracles as $12000$ and the mini-batch size for each iteration as $1$.
We enumerate the step size of SGLD and the inner step size of SPS-SGLD from $0.2$ to $1.4$. 
Besides, the inner loops' iterations and the outer loops' step sizes are grid-searched with $[20, 40, 80]$ and $[1.0, 4.0,10.0]$.
Besides, we use the formulation $\mathrm{TV}(\hat{p}_K, p_*)\coloneqq \frac{1}{2d}\sum_{i=1}^d \mathrm{TV}(\hat{p}_{K,i}, p_{*,i})$
to estimate total variation distances between the target distribution and the underlying distribution of returned particles, where $\hat{p}_{K,i}$ and $p_{*,i}$ are the marginal distributions of the $i$-th coordinate.
For $1$d distributions, their densities can be approximated by the histogram of particles.

\noindent\textbf{Experimental Results.} We first show the optimal TV distance to the target distribution $p_*$ obtained by SGLD and SPS-SGLD under different dimensions in the left column of  Fig.~\ref{fig:tv_diff}.
Since we consider different problems when using different dimensions, the sampling error does not necessarily increase when $d$ increases.
It can be clearly observed that the optimal TV distance of SPS-SGLD is at least $0.5$ smaller than that of SGLD in all our dimension settings, which means SPS-SGLD presents a significantly better performance in this synthetic task.
Specifically, we investigate the changes in the TV distance with the growth of step sizes for both SPS-SGLD and SGLD, and show the results in the right column Fig.~\ref{fig:tv_diff}.
Although the absolute values of these two algorithms vary a lot, their changing trends are very similar.
When the step size is small, both SPS-SGLD and SGLD describe the local landscape of a single mode well. 
With the growth of step sizes, they can gradually cover all modes, whereas SPS-SGLD achieves a lower TV distance since it can cover modes and keep the local landscape well with a smaller step size.
Besides, we provide show distributions of particles' projections under different stochastic gradient oracles when $d=10$ and the optimal step sizes are chosen in Fig.~\ref{fig:illustration_exp}.
According to the contour of the projected negative log density of $p_*$, we note that SPS-SGLD can cover all modes with a more accurate variance estimation compared with SGLD.
It demonstrates that SPS-SGLD generates more reasonable samples with different stochastic gradient oracles from another perspective.

\section{Conclusion}
\label{sec:conclu}
This paper is the first study about adapting stochastic gradient oracles to unbiased samplers to draw samples from unnormalized non-log-concave target distributions, i.e., $p_*\propto e^{-f}$.
Specifically, we provide a framework named stochastic proximal samplers (SPS) to remove the unrealistic requirement about stationary points of $f$ in previous implementations~\cite{altschuler2023faster}.
Furthermore, compared with biased samplers SGLD and its variants, two implementations of the SPS framework can converge to the target distribution $p_*$ with a lower gradient complexity with an $\tilde{O}(d^{1/3})$ factor at least, and this improvement is validated by our experiments conducted on synthetic data.

\newpage

\bibliographystyle{apalike}
\bibliography{0_contents/ref}  





\newpage
\appendix

\section{Additional Notations and Assumptions in Appendix}
\label{sec:not_ass_0x}

For the convenience of analysis, we define three Markov processes, i.e., $\{\rvx_k\}$, $\{\tilde{\rvx}_k\}$ and $\{\hat{\rvx}_k\}$, as follows.
For the process $\{\rvx_k\}$, we suppose its initialization $\rvx_0$  is drawn from the standard Gaussian of $\R^d$.
There are two transition kernels in this process. 
The first provides the conditional probability of $\rvx_{k+1/2}$ when $\rvx_{k}$ is given and can be presented as the same as Eq~\ref{def:transition_kernel_stage1}, i.e.,
\begin{equation*}
    p_{k+\frac{1}{2}|k}(\vx^\prime|\vx) \propto \exp\left(-\frac{\left\|\vx^\prime - \vx\right\|^2}{2\eta_k}\right).
\end{equation*}
The second transition kernel denotes the conditional probability of $\rvx_{k+1}$ when $\rvx_{k+1/2}$ and a stochastic mini-batch $\rvb$ is given and can be presented as the same as Eq~\ref{def:transition_kernel_stage2}, i.e.,
\begin{equation*}
    \begin{aligned}
        p_{k+1|k+\frac{1}{2},b}(\vx^\prime|\vx, \vb) \propto \exp\left(-f_\vb(\vx^\prime)-\frac{\left\|\vx^\prime - \vx\right\|^2}{2\eta_k}\right).
    \end{aligned}
\end{equation*}
For the process $\{\tilde{\rvx}_k\}$, we suppose the initialization $\tilde{\rvx}_0$ shares the same distribution as $\rvx_0$, and the transition kernel is defined as 
\begin{equation}
    \label{def:transition_kernel_fullgrad}
    \begin{aligned}
        \tilde{p}_{k+\frac{1}{2}|k} \coloneqq p_{k+\frac{1}{2}|k}\quad \mathrm{and}\quad \tilde{p}_{k+1|k+\frac{1}{2},b}(\vx^\prime|\vx,\vb) = p_{k+1|k+\frac{1}{2},b}(\cdot|\vx, \{1,2,\ldots, N\}).
    \end{aligned}
\end{equation}
For the third process $\{\hat{\rvx}_k\}$, it presents the actual Markov process obtained by implementing Alg~\ref{alg:sps}.
That is to say, the initialization $\hat{\rvx}_0$ shares the same distribution as $\rvx_0$.
The transition kernel satisfies
$\hat{p}_{k+\frac{1}{2}|k} \coloneqq p_{k+\frac{1}{2}|k}$ and 
\begin{equation*}
    \begin{aligned}
        \KL{\hat{p}_{k+1|k+\frac{1}{2},b}(\cdot|\vx,\vb)}{{p}_{k+1|k+\frac{1}{2},b}(\cdot|\vx,\vb)} \le \delta_k.
    \end{aligned}
\end{equation*}
It should be noted that the transition kernel $\hat{p}_{k+1|k+\frac{1}{2},b}(\cdot|\vx,\vb)$ does not have a explicit form.
Instead, it depends on the sampling process at Line 5 of Alg~\ref{alg:sps}.
Although no explicit form is required, it still should be a good approximation of $p_{k+1|k+\frac{1}{2},b}(\vx^\prime|\vx, \vb)$.
At last, to simplify the notation, we denote $\varphi_{\sigma^2}$ as the density function of the Gaussian distribution $\mathcal{N}(\vzero, \sigma^2\mI)$.

\paragraph{Assumption~\ref{con_ass:lsi} implies a bounded second moment:}
\begin{lemma}
Assume that density $p_*$ satisfies assumption~\ref{con_ass:lsi} that for any smooth function $g(\vx)$ satisfying $\E_{p_*}[g^2]<\infty$:
\[
\E_{p_*}\left[g^2\log g^2\right]-\E_{p_*}[g^2]\log \E_{p_*}[g^2]
\le \frac{2}{\alpha_*} \E_{p_*}\left\|\grad g\right\|^2.
\]
Then density $p_*$ has the following variance bound:
\[
\E_{\vx\sim p^*}[\|\vx-\E[\vx]\|^2] \leq 2 d/\alpha^*.
\]
\end{lemma}
\begin{proof}
Consider a target distribution $p^*$ that follows~\ref{con_ass:lsi} and for the simplicity of notation denote a constant $\const=1/(2\alpha^*)$.
We then follow the Herbst argument and take the test function in \ref{con_ass:lsi} to be $g(\vx)=e^{t f(\vx)/2}$, for an arbitrary $t>0$ and a function $f$ so that $\|\nabla f(\vx)\|\leq 1$. 
We obtain from the substitution that
\begin{align*}
\E_{p_*}\left[ t f(\vx) e^{t f(\vx)} \right] 
        - \E_{p_*}[e^{t f(\vx)}]\log \E_{p_*}[e^{t f(\vx)}]
        \le \const \E_{p_*} \left[ t^2 e^{t f(\vx)} \left\|\grad f(\vx)\right\|^2 \right]
        \le \const \E_{p_*} \left[ t^2 e^{t f(\vx)} \right].
\end{align*}
Denote $F(t)=\E_{p_*}\left[e^{tf(\vx)}\right]$. We rewrite the above inequality as a differential inequality:
\[
t F'(t) \leq F(t) \log F(t) + \const t^2 F(t),
\]
or equivalently:
\[
\frac{\der}{\der t} \left( \frac{1}{t} \log F(t) \right) \leq \const.
\]
Taking $t\rightarrow 0$, we know that the initial condition is $\frac{1}{t} \log F(t) \rightarrow \E_{p_*}[f(\vx)]$.
Therefore, along the entire trajectory
\[
\frac{1}{t} \log F(t) \leq \E_{p_*}[f(\vx)] + \const \cdot t.
\]
Plugging in the definition of $F(t)$, that is
\[
\E_{p_*}\left[e^{tf(\vx)}\right] \leq \exp\left( t \E_{p_*}[f(\vx)] + \const \cdot t^2 \right).
\]

By Markov's inequality, we obtain that for $\vx\sim p^*$ and for any $t>0$: 
\[
P\left( f(\vx) - \E[f(\vx)] > \lambda \right) \leq \exp(\const t^2 - \lambda t).
\]
Optimizing over $t$ gives 
\[
P\left( f(\vx) - \E[f(\vx)] > \lambda \right) \leq \exp\left(- \frac{\lambda^2}{4\const} \right).
\]

Taking $f(\vx) = \langle \vx, \vtheta \rangle$, for any $\|\vtheta\|=1$, gives the standard subGaussian tail bound:
\[
P( | \langle \vx - \E[\vx], \vtheta \rangle | > \lambda ) \leq 2 \exp\left(- \frac{\lambda^2}{4\const} \right), \forall \|\vtheta\|=1,
\]
which means that random vector $\vx\sim p^*$ is $\sqrt{2\const}$-subGaussian.
This also implies that $\vx\sim p^*$ is $\sqrt{2\const\cdot d}$-norm-subGaussian, leading to the following moment bound:
\[
(\E[\|\vx-\E[\vx]\|^p])^{1/p} \leq \sqrt{2 p \const \cdot d}.
\]
We read off the second moment bound from the above inequality: $\E_{\vx\sim p^*}[\|\vx-\E[\vx]\|^2] \leq 4 \const \cdot d = 2 d/\alpha^*$.
\end{proof}

\paragraph{Implementation of InnerULD:}
\begin{algorithm}[t]
    \caption{Inner underdamped Langevin algorithm: $\mathsf{InnerULD}(\vx_0, \vb, \eta, \delta)$}
    \label{alg:ULD_inner}
    \begin{algorithmic}[1]
            \STATE {\bfseries Input:} The output particle $\vx_0$ of Alg.~\ref{alg:sps} Line 3, the selected mini-batch $\vb$, the step size of outer loop $\eta$, the required accuracy of the inner loop $\delta$;
            \STATE Initialize the particle with $\rvz_0 \gets \vx_0$ and the velocity $\rvv_0$ is sampled from $\mathcal{N}(\vzero,\mI)$; 
            \FOR{$s=0$ to $S-1$}
                \STATE Draw sample $(\rvz_{s+1}, \rvv_{s+1})$ from the following Gaussian distribution $ \mathcal{N}\left(g^\prime(\rvz_s,\rvv_s), \Sigma\right).
                $
            \ENDFOR
            \STATE {\bfseries Return:} $\rvz_S$.
    \end{algorithmic}
\end{algorithm}

Specifically, The closed form of the update of ULD shown in Line 4 of Alg.~\ref{alg:ULD_inner} satisfies $g^\prime \colon \R^d\times \R^d \rightarrow \R^d\times \R^d$ defined as
\begin{equation*}
    \begin{aligned}
        g^\prime(\vz,\vv) \coloneqq \left(\vz+\gamma^{-1}(1-a)\vv- \gamma^{-1}\left(\tau - \gamma^{-1}(1-a)\right)\grad g(\vz), a\vv - \gamma^{-1}(1-a)\grad g(\vz)\right),
    \end{aligned}
\end{equation*}
where $a\coloneqq \exp(-\gamma \tau)$, and 
\begin{equation*}
    \Sigma \coloneqq \left[
        \begin{matrix}
            \frac{2}{\gamma}\left(\tau - \frac{2}{\gamma}(1-a)+\frac{1}{2\gamma}(1-a^2)\right)\cdot \mI_d & \frac{2}{\gamma}\left(\frac{1}{2}-a+a^2\right)\cdot \mI_d\\
            \frac{2}{\gamma}\left(\frac{1}{2}-a+a^2\right)\cdot \mI_d & (1-a^2)\cdot \mI_d
        \end{matrix}
    \right].
\end{equation*}
Such an iteration corresponds to the discretization of the following SDE
\begin{equation*}
    \begin{aligned}
        \der \rvz_t =&  \rvv_t\der t,\\
        \der \rvv_t =& -\grad g(\rvz_s; \vx_0,\vb,\eta)\der t-\gamma \rvv_t\der t+\sqrt{2\gamma}\der B_t,
    \end{aligned}
\end{equation*}
where $B_t$ is a standard $d$-dimensional Brownian motion.
This update is introduced in several references, including~\citet{cheng2018underdamped,altschuler2023faster}.

\section{Lemmas for SPS Framework}

\begin{lemma}[variant of data-processing inequality]
    \label{lem:kl_chain_rule}
    Consider four random variables, $\rvx, \rvz, \tilde{\rvx}, \tilde{\rvz}$, whose underlying distributions are denoted as $p_x, p_z, q_x, q_z$.
    Suppose $p_{x,z}$ and $q_{x,z}$ denotes the densities of joint distributions of $(\rvx,\rvz)$ and $(\tilde{\rvx},\tilde{\rvz})$, which we write in terms of the conditionals and marginals as
    \begin{equation*}
        \begin{aligned}
        &p_{x,z}(\vx,\vz) = p_{x|z}(\vx|\vz)\cdot p_z(\vz)=p_{z|x}(\vz|\vx)\cdot p_{x}(\vx)\\
        &q_{x,z}(\vx,\vz)=q_{x|z}(\vx|\vz)\cdot q_z(\vz) = q_{z|x}(\vz|\vx)\cdot q_x(\vx).
        \end{aligned}
    \end{equation*}
    then we have
    \begin{equation*}
        \begin{aligned}
            \KL{p_{x,z}}{q_{x,z}} = & \KL{p_z}{q_z} + \E_{\rvz\sim p_z}\left[\KL{p_{x|z}(\cdot|\rvz)}{q_{x|z}(\cdot|\rvz)}\right]\\
            = & \KL{p_x}{q_x}+\E_{\rvx \sim p_x}\left[\KL{p_{z|x}(\cdot|\rvx)}{q_{z|x}(\cdot|\rvx)}\right]
        \end{aligned}
    \end{equation*}
    where the latter equation implies
    \begin{equation*}
        \KL{p_x}{q_x}\le \KL{p_{x,z}}{q_{x,z}}.
    \end{equation*}
\end{lemma}
\begin{proof}
    According to the formulation of KL divergence, we have
    \begin{equation*}
        \begin{aligned}
            \KL{p_{x,z}}{q_{x,z}} = &\int p_{x,z}(\vx, \vz) \log \frac{p_{x,z}(\vx, \vz)}{q_{x,z}(\vx, \vz)}\der (\vx, \vz)\\
            = & \int p_{x,z}(\vx,\vz)\left(\log \frac{p_x(\vx)}{q_x(\vx)} + \log \frac{p_{z|x}(\vz|\vx)}{q_{z|x}(\vz|\vx)}\right) \der(\vx,\vz)\\
            = & \int p_{x,z}(\vx, \vz)\log \frac{p_x(\vx)}{q_x(\vx)}\der(\vx, \vz)+\int p_x(\vx) \int p_{z|x}(\vz|\vx)\log \frac{p_{z|x}(\vz|\vx)}{q_{z|x}(\vz|\vx)}\der \vz \der\vx\\
            = & \KL{p_x}{q_x}+\E_{\rvx \sim p_x}\left[\KL{p_{z|x}(\cdot|\rvx)}{q_{z|x}(\cdot|\rvx)}\right] \ge \KL{p_x}{q_x},
        \end{aligned}
    \end{equation*}
    where the last inequality follows from the fact
    \begin{equation*}
        \KL{p_{z|x}(\cdot|\vx)}{\tilde{p}_{z|x}(\cdot|\vx)}\ge 0\quad \forall\ \vx.
    \end{equation*}
    With a similar technique, it can be obtained that
    \begin{equation*}
        \begin{aligned}
            \KL{p_{x,z}}{q_{x,z}} = &\int p_{x,z}(\vx, \vz) \log \frac{p_{x,z}(\vx, \vz)}{q_{x,z}(\vx, \vz)}\der (\vx, \vz)\\
            = & \int p_{x,z}(\vx,\vz)\left(\log \frac{p_z(\vz)}{q_z(\vz)} + \log \frac{p_{x|z}(\vx|\vz)}{q_{x|z}(\vx|\vz)}\right) \der(\vx,\vz)\\
            = & \int p_{x,z}(\vx, \vz)\log \frac{p_z(\vz)}{q_z(\vz)}\der(\vx, \vz)+\int p_z(\vz) \int p_{x|z}(\vx|\vz)\log \frac{p_{x|z}(\vx|\vz)}{q_{x|z}(\vx|\vz)}\der \vz \der\vx\\
            = & \KL{p_z}{q_z}+\E_{\rvz \sim p_z}\left[\KL{p_{x|z}(\cdot|\rvz)}{\tilde{p}_{x|z}(\cdot|\rvz)}\right].
        \end{aligned}
    \end{equation*}
    Hence, the proof is completed.
\end{proof}

\begin{lemma}[strong log-concavity and smoothness of inner target functions]
    \label{lem:sc_sm_inner_targets}
    Using the notations presented in Section~\ref{sec:not_ass_0x}, for any $k\in\{0, 1,\ldots, K-1\}$, $\vx\in\R^d$, and $\vb\subseteq \{1,2,\ldots, n\}$, suppose $\eta_k < 1/L$, then the target distributions of inner loops, i.e., $p_{k+1|k+1/2, b}(\cdot | \vx, \vb)$, satisfy
    \begin{equation*}
        (-L+\eta_k^{-1})\cdot \mI \preceq - \grad^2_{\vx^\prime} \log p_{k+1|k+1/2,b}(\vx^\prime|\vx,\vb) \preceq (L+\eta_k^{-1})\cdot \mI
    \end{equation*}
\end{lemma}
\begin{proof}
    For any $k\in\{0, 1,\ldots, K-1\}$, $\vx\in\R^d$, and $\vb\subseteq \{1,2,\ldots, n\}$, we have
    \begin{equation*}
        p_{k+1|k+\frac{1}{2}, \vb}(\vx^\prime|\vx, \vb) = C(\vb,\eta_k,\vx)^{-1}\cdot \exp\left(-f_\vb(\vx^\prime)-\frac{\left\|\vx^\prime - \vx\right\|^2}{2\eta_k}\right),
    \end{equation*}
    which implies
    \begin{equation*}
        - \grad^2_{\vx^\prime} \log p_{k+1|k+1/2,b}(\vx^\prime|\vx,\vb) = \grad^2 f_{\vb}(\vx^\prime) + \eta_k^{-1}\cdot \mI.
    \end{equation*}
    Since we have~\ref{con_ass:lips_loss}, it has
    \begin{equation*}
        (-L+\eta_k^{-1})\cdot \mI \preceq \grad^2 f_{\vb}(\vx^\prime) + \eta_k^{-1}\cdot \mI \preceq (L+\eta_k^{-1})\cdot \mI.
    \end{equation*}
    Hence, the proof is completed.
\end{proof}

\begin{lemma}
    \label{lem:kl_error_cmp_fullgrad}
    Using the notations presented in Section~\ref{sec:not_ass_0x}, for any $k\in\{0,1,\ldots,K-1\}$, $\vx\in\R^d$ and $\vb\subseteq \{1,2,\ldots, n\}$, suppose it has $\eta<1/L$, then we have
    \begin{equation*}
        \KL{\tilde{p}_{k+1|k+\frac{1}{2},b}(\cdot|\vx,\vb)}{p_{k+1|k+\frac{1}{2},b}(\cdot|\vx,\vb)}\le \frac{1}{2(\eta^{-1}-L)}\cdot \E_{\rvx^\prime\sim \tilde{p}_{k+1|k+\frac{1}{2},b}}\left[\left\|\grad f(\rvx^\prime)- \grad f_{\vb}(\rvx^\prime)\right\|^2\right]
    \end{equation*}
\end{lemma}
\begin{proof}
    We abbreviate $p_{k+1|k+1/2,b}(\cdot|\vx,\vb)$ and $\tilde{p}_{k+1|k+1/2,b}(\cdot|\vx,\vb)$ as $p$ and $\tilde{p}$ for convenience.
    According to the definition of $p$, i.e., Eq~\ref{def:transition_kernel_stage2}, and $\tilde{p}$, i.e., Eq~\ref{def:transition_kernel_fullgrad}, we have
    \begin{equation*}
        \begin{aligned}
            &p(\vx^\prime) = C(\vb,\eta,\vx)^{-1}\cdot \exp\left(-f_\vb(\vx^\prime)-\frac{\left\|\vx^\prime - \vx\right\|^2}{2\eta}\right)\\
            &\tilde{p}(\vx^\prime) = C(\eta,\vx)^{-1}\cdot \exp\left(-f(\vx^\prime)-\frac{\left\|\vx^\prime - \vx\right\|^2}{2\eta}\right). 
        \end{aligned}
    \end{equation*}
    According to Lemma~\ref{lem:sc_sm_inner_targets}, we have
    \begin{equation*}
        \begin{aligned}
            -\grad^2 \log p(\vx^\prime)\succeq \left(-L + \eta^{-1}\right)\mI, 
        \end{aligned}
    \end{equation*}
    which means the density function $p$ is strongly log-concave when $\eta<1/L$.
    According to Lemma~\ref{lem:strongly_lsi}, the density function $p$ satisfies LSI with a constant $(\eta^{-1}-L)$.
    Then, with the definition of LSI, we have
    \begin{equation*}
        \begin{aligned}
            \KL{\tilde{p}}{p}\le &\frac{1}{2(\eta^{-1}-L)}\cdot \E_{\rvx^\prime\sim \tilde{p}}\left[\left\|\grad\log \frac{\tilde{p}(\rvx^\prime)}{p(\rvx^\prime)}\right\|^2\right]= \frac{1}{2(\eta^{-1}-L)}\cdot \E_{\rvx^\prime\sim \tilde{p}}\left[\left\|\grad f(\rvx^\prime)- \grad f_{\vb}(\rvx^\prime)\right\|^2\right]
        \end{aligned}
    \end{equation*}
    Hence, the proof is completed.
\end{proof}

\begin{lemma}
    \label{lem:traj_error_cmp_fullgrad}
    Using the notations presented in Section~\ref{sec:not_ass_0x} and considering Alg~\ref{alg:sps}, if $\eta_i\le 1/(2L)$ for all $i\in\{0,1,\ldots, K-1\}$, then we have
    \begin{equation*}
        \TVD{\tilde{p}_{K}}{p_{K}}\le \sigma \sqrt{\sum_{i=0}^{K-1} \frac{\eta_i}{2|\rvb_i|}}
    \end{equation*}
    where $|\cdot|$ denotes the sample size in each mini-batch loss.
\end{lemma}
\begin{proof}
    According to Pinsker's inequality, we have
    \begin{equation*}
        \TVD{p_{K}}{\tilde{p}_{K}} \le \sqrt{\frac{1}{2}\KL{\tilde{p}_{K}}{p_{K}}}.
    \end{equation*}
    Let $p_{k+1, k+1/2,b}$ and $\tilde{p}_{k+1, k+1/2,b}$ denote the density of joint distribution of $(\rvx_{k+1}, \rvx_{k+1/2},\rvb_k)$ and $(\tilde{\rvx}_{k+1},\tilde{\rvx}_{k+1/2},\tilde{\rvb}_k)$ respectively, which we write in term of the conditionals and marginals as
    \begin{equation*}
        \begin{aligned}
            p_{k+1, k+\frac{1}{2} ,b}(\vx^\prime,\vx, \vb) = &p_{k+1|k+\frac{1}{2},b}(\vx^\prime|\vx,\vb)\cdot p_{k+\frac{1}{2},b}(\vx, \vb)=p_{k+\frac{1}{2},b|k+1}(\vx,\vb|\vx^\prime)\cdot p_{k+1}(\vx^\prime)\\
            \tilde{p}_{k+1, k+\frac{1}{2},b}(\vx^\prime,\vx, \vb) = &\tilde{p}_{k+1|k+\frac{1}{2},b}(\vx^\prime|\vx,\vb)\cdot \tilde{p}_{k+\frac{1}{2},b}(\vx, \vb) = \tilde{p}_{k+\frac{1}{2},b|k+1}(\vx,\vb|\vx^\prime)\cdot \tilde{p}_{k+1}(\vx^\prime).
        \end{aligned}
    \end{equation*}
    In this condition, we have
    \begin{equation*}
        \begin{aligned}
            &\KL{\tilde{p}_{k+1}}{p_{k+1}} \le \KL{\tilde{p}_{k+1,k+\frac{1}{2},b}}{p_{k+1,k+\frac{1}{2},b}}\\
            & = \KL{\tilde{p}_{k+\frac{1}{2},b}}{p_{k+\frac{1}{2},b}} + \E_{(\tilde{\rvx},\tilde{\rvb})\sim \tilde{p}_{k+\frac{1}{2},b}}\left[\KL{\tilde{p}_{k+1|k+\frac{1}{2},b}(\cdot |\tilde{\rvx},\tilde{\rvb})}{p_{k+1|k+\frac{1}{2},b}(\cdot |\tilde{\rvx},\tilde{\rvb})}\right]\\
            & = \KL{\tilde{p}_{k+\frac{1}{2}}}{p_{k+\frac{1}{2}}} + \E_{\tilde{\rvx}\sim \tilde{p}_{k+\frac{1}{2}}}\left[\KL{\tilde{p}_{b|k+\frac{1}{2}}(\cdot|\tilde{\rvx})}{p_{b|k+\frac{1}{2}}(\cdot|\tilde{\rvx})}\right]\\
            &\quad + \E_{(\tilde{\rvx},\tilde{\rvb})\sim \tilde{p}_{k+\frac{1}{2},b}}\left[\KL{\tilde{p}_{k+1|k+\frac{1}{2},b}(\cdot |\tilde{\rvx},\tilde{\rvb})}{p_{k+1|k+\frac{1}{2},b}(\cdot |\tilde{\rvx},\tilde{\rvb})}\right],
        \end{aligned}
    \end{equation*}
    which follows from Lemma~\ref{lem:kl_chain_rule}.
    Respectively, for the first and the second equation, we plug 
    \begin{equation*}
        \rvx \coloneqq \tilde{\rvx}_{k+1},\ \rvz\coloneqq \big(\tilde{\rvx}_{k+\frac{1}{2}},\tilde{\rvb}_k\big),\ \tilde{\rvx}\coloneqq \rvx_{k+1}\ \mathrm{and}\ \tilde{\rvz}\coloneqq \big(\rvx_{k+\frac{1}{2}},\rvb_k\big)
    \end{equation*}
    and 
    \begin{equation*}
        \rvx = \tilde{\rvx}_{k+\frac{1}{2}},\ \rvz\coloneqq \tilde{\rvb}_k,\ \tilde{\rvx}\coloneqq \rvx_{k+\frac{1}{2}}\ \mathrm{and}\ \tilde{\rvz}\coloneqq \rvb_k,
    \end{equation*}
    to Lemma~\ref{lem:kl_chain_rule}.
    Here, we should note the choice of $\tilde{\rvb}_k$ is introduced as an auxiliary random variable, which is independent with the update of $\tilde{\rvx}_k$ for all $k\in\{0,1,\ldots, K-1\}$. 
    Then, by requiring
    \begin{equation}
        \label{eq:sample_stra_b}
        \tilde{p}_{b|k+\frac{1}{2}}(\cdot|\vx)= p_{b|k+\frac{1}{2}}(\cdot|\vx)=p_b\quad \forall \vx\in\R^d \quad \mathrm{and}\quad \eta_k\le 1/(2L).
    \end{equation}
    we have
    \begin{equation}
        \label{ineq:trans_ker_error_stage2}
        \begin{aligned}
            &\KL{\tilde{p}_{k+1}}{p_{k+1}} \le  \KL{\tilde{p}_{k+\frac{1}{2}}}{p_{k+\frac{1}{2}}} + \E_{(\tilde{\rvx},\tilde{\rvb})\sim \tilde{p}_{k+\frac{1}{2},b}}\left[\KL{\tilde{p}_{k+1|k+\frac{1}{2},b}(\cdot |\tilde{\rvx},\tilde{\rvb})}{p_{k+1|k+\frac{1}{2},b}(\cdot |\tilde{\rvx},\tilde{\rvb})}\right]\\
            &\le  \KL{\tilde{p}_{k+\frac{1}{2}}}{p_{k+\frac{1}{2}}} + \frac{1}{2\cdot (\eta_k^{-1}-L)}\cdot \E_{(\tilde{\rvx},\tilde{\rvb})\sim \tilde{p}_{k+\frac{1}{2},b}}\left[\E_{\rvx^\prime\sim \tilde{p}_{k+1|k+\frac{1}{2},b}(\cdot|\tilde{\rvx},\tilde{\rvb})}\left\|\grad f(\rvx^\prime)-\grad f_{\tilde{\rvb}}(\rvx^\prime)\right\|^2\right]\\
            &\le \KL{\tilde{p}_{k+\frac{1}{2}}}{p_{k+\frac{1}{2}}} + \eta_k \cdot \E_{(\tilde{\rvx},\tilde{\rvb})\sim \tilde{p}_{k+\frac{1}{2},b}}\left[\E_{\rvx^\prime\sim \tilde{p}_{k+1|k+\frac{1}{2},b}(\cdot|\tilde{\rvx},\tilde{\rvb})}\left\|\grad f(\rvx^\prime)-\grad f_{\tilde{\rvb}}(\rvx^\prime)\right\|^2\right]
        \end{aligned}
    \end{equation}
    where the second inequality follows from Lemma~\ref{lem:kl_error_cmp_fullgrad} and the last inequality follows from the choice of step size satisfies $\eta_k$.
    
    Then, we consider the upper bound for the second term of RHS of Eq~\ref{ineq:trans_ker_error_stage2} and have
    \begin{equation}
        \label{eq:stochastic_error_form}
        \begin{aligned}
            &\E_{(\tilde{\rvx},\tilde{\rvb})\sim \tilde{p}_{k+\frac{1}{2},b}}\left[\E_{\rvx^\prime\sim \tilde{p}_{k+1|k+\frac{1}{2},b}(\cdot|\tilde{\rvx},\tilde{\rvb})}\left\|\grad f(\rvx^\prime)-\grad f_{\tilde{\rvb}}(\rvx^\prime)\right\|^2\right]\\
            &= \int \tilde{p}_{k+1,k+\frac{1}{2},b}(\vx^\prime, \tilde{\vx},\tilde{\vb})\cdot \left\|\grad f(\vx^\prime) - \grad f_{\tilde{\vb}}(\vx^\prime)\right\|^2 \der (\vx^\prime, \tilde{\vx},\tilde{\vb}).
        \end{aligned}
    \end{equation}
    The density $\tilde{p}_{k+1,k+\frac{1}{2},b}(\vx^\prime, \tilde{\vx},\tilde{\vb})$ of the joint distribution satisfies
    \begin{equation}
        \label{eq:joint_dis_split}
        \begin{aligned}
            \tilde{p}_{k+1,k+\frac{1}{2},b}(\vx^\prime,\vx,\vb)= &\tilde{p}_{k+1|k+\frac{1}{2},b}(\vx^\prime|\vx,\vb)\cdot \tilde{p}_{k+\frac{1}{2},b}(\vx,\vb)\\
            = & \tilde{p}_{k+1|k+\frac{1}{2}}(\vx^\prime|\vx)\cdot \tilde{p}_{k+\frac{1}{2}}(\vx)\cdot \tilde{p}_{b|k+\frac{1}{2}}(\vb|\vx)\\
            = & \tilde{p}_{k+1|k+\frac{1}{2}}(\vx^\prime|\vx)\cdot  \tilde{p}_{k+\frac{1}{2}}(\vx)\cdot p_b(\vb),
        \end{aligned}
    \end{equation}
    where the second equation establishes since the choice of $\tilde{\rvb}_k$ will not affect the update of $\tilde{\rvx}_k$ shown in Eq~\ref{def:transition_kernel_fullgrad}.
    Besides, the last inequality follows from Eq~\ref{eq:sample_stra_b} and the fact that the choice of $\rvb_k$ is independent with the choice of $\rvx_k$ shown in Line 4 of Alg~\ref{alg:sps}.
    Combining Eq~\ref{eq:stochastic_error_form} and Eq~\ref{eq:joint_dis_split}, we have
    \begin{equation}
        \label{eq:stochastic_error_form_fin}
        \begin{aligned}
            &\E_{(\tilde{\rvx},\tilde{\rvb})\sim \tilde{p}_{k+\frac{1}{2},b}}\left[\E_{\rvx^\prime\sim \tilde{p}_{k+1|k+\frac{1}{2},b}(\cdot|\tilde{\rvx},\tilde{\rvb})}\left\|\grad f(\rvx^\prime)-\grad f_{\tilde{\rvb}}(\rvx^\prime)\right\|^2\right]\\
            &= \sum_{\vb\subseteq{1,2,\ldots,n}}\int \tilde{p}_{k+1|b}(\vx^\prime) p_b(\vb)\left\|\grad f(\vx^\prime) - \grad f_{\tilde{\vb}}(\vx^\prime)\right\|^2 \der \vx^\prime\\
            &=\int \tilde{p}_{k+1}(\vx^\prime)\E_{\rvb_k}\left[\left\|\grad f(\vx^\prime) - \grad f_{\rvb_k}(\vx^\prime)\right\|\right]\der \vx^\prime \le \frac{\sigma^2}{|\rvb_k|},
        \end{aligned}
    \end{equation}
    where the last inequality follows from~\ref{con_ass:var_bound} and Lemma~\ref{lem:minibatch_var}.
    Hence, Eq~\ref{ineq:trans_ker_error_stage2} satisfies
    \begin{equation}
        \label{ineq:trans_ker_error_stage2_fin}
        \begin{aligned}
            \KL{\tilde{p}_{k+1}}{p_{k+1}} \le &\KL{\tilde{p}_{k+\frac{1}{2}}}{p_{k+\frac{1}{2}}} +\sigma^2 \cdot \frac{\eta_k }{|\rvb_k|}.
        \end{aligned}
    \end{equation}
    
    Then, consider the first stage of the update, we have
    \begin{equation}
        \label{ineq:trans_ker_error_stage1}
        \begin{aligned}
            \KL{\tilde{p}_{k+\frac{1}{2}}}{p_{k+\frac{1}{2}}}\le &\KL{\tilde{p}_k}{p_k}+ \E_{\tilde{\rvx}\sim \tilde{p}_k}\left[\KL{\tilde{p}_{k+\frac{1}{2}|k}(\cdot|\tilde{\rvx})}{p_{k+\frac{1}{2}|k}(\cdot|\tilde{\rvx})}\right]= \KL{\tilde{p}_k}{p_k},
        \end{aligned}
    \end{equation}
    where the first inequality follows from Lemma~\ref{lem:kl_chain_rule} by setting
    \begin{equation*}
        \rvx = \tilde{\rvx}_{k+\frac{1}{2}},\ \rvz\coloneqq \tilde{\rvx}_k,\ \tilde{\rvx}\coloneqq \rvx_{k+\frac{1}{2}}\ \mathrm{and}\ \tilde{\rvz}\coloneqq \rvx_k,
    \end{equation*}
    and the second equation establishes since $\{\rvx_k\}$ and $\tilde{\rvx}_k$ share the same update in the first stage shown in Eq~\ref{def:transition_kernel_stage1} and Eq~\ref{def:transition_kernel_fullgrad}.

    Combining Eq~\ref{ineq:trans_ker_error_stage2_fin} and Eq~\ref{ineq:trans_ker_error_stage1}, we have
    \begin{equation*}
        \KL{\tilde{p}_{k+1}}{p_{k+1}} \le  \KL{\tilde{p}_k}{p_k} + \sigma^2\cdot \frac{\eta_k}{|\rvb_k|},
    \end{equation*}
    which implies
    \begin{equation*}
        \KL{\tilde{p}_{K}}{p_{K}} \le \sigma^2 \cdot \sum_{i=0}^{K-1} \frac{\eta_i}{|\rvb_i|}
    \end{equation*}
    with the telescoping sum.
    Hence, the proof is completed.
\end{proof}

\begin{lemma}
    \label{lem:traj_error_cmp_norgo}
    Using the notations presented in Section~\ref{sec:not_ass_0x}, we have
    \begin{equation*}
        \TVD{\hat{p}_{K}}{p_{K}}\le \sqrt{\frac{1}{2}\sum_{i=0}^{K-1} \delta_i}
    \end{equation*}
    where $\delta$ denotes the error tolerance of approximate conditional densities shown in Eq~\ref{def:transition_kernel_norgo}.
\end{lemma}
\begin{proof}
    According to Pinsker's inequality, we have
    \begin{equation*}
        \TVD{\hat{p}_{k+1}}{p_{k+1}} \le \sqrt{\frac{1}{2}\KL{\hat{p}_{k+1}}{p_{k+1}}}.
    \end{equation*}
    Let $p_{k+1, k+1/2,b}$ and $\hat{p}_{k+1, k+1/2,b}$ denote the density of joint distribution of $(\rvx_{k+1}, \rvx_{k+1/2},\rvb_k)$ and $(\hat{\rvx}_{k+1},\hat{\rvx}_{k+1/2},\hat{\rvb}_k)$ respectively, which we write in term of the conditionals and marginals as
    \begin{equation*}
        \begin{aligned}
            p_{k+1, k+\frac{1}{2} ,b}(\vx^\prime,\vx, \vb) = &p_{k+1|k+\frac{1}{2},b}(\vx^\prime|\vx,\vb)\cdot p_{k+\frac{1}{2},b}(\vx, \vb)=p_{k+\frac{1}{2},b|k+1}(\vx,\vb|\vx^\prime)\cdot p_{k+1}(\vx^\prime)\\
            \hat{p}_{k+1, k+\frac{1}{2},b}(\vx^\prime,\vx, \vb) = &\hat{p}_{k+1|k+\frac{1}{2},b}(\vx^\prime|\vx,\vb)\cdot \hat{p}_{k+\frac{1}{2},b}(\vx, \vb) = \hat{p}_{k+\frac{1}{2},b|K+1}(\vx,\vb|\vx^\prime)\cdot \hat{p}_{k+1}(\vx^\prime).
        \end{aligned}
    \end{equation*}

    In this condition, we have
    \begin{equation*}
        \begin{aligned}
            &\KL{\hat{p}_{k+1}}{p_{k+1}} \le \KL{\hat{p}_{k+1,k+\frac{1}{2},b}}{p_{k+1,k+\frac{1}{2},b}}\\
            & = \KL{\hat{p}_{k+\frac{1}{2},b}}{p_{k+\frac{1}{2},b}} + \E_{(\hat{\rvx},\hat{\rvb})\sim \hat{p}_{k+\frac{1}{2},b}}\left[\KL{\hat{p}_{k+1|k+\frac{1}{2},b}(\cdot |\hat{\rvx},\hat{\rvb})}{p_{k+1|k+\frac{1}{2},b}(\cdot |\hat{\rvx},\hat{\rvb})}\right]\\
            & = \KL{\hat{p}_{k+\frac{1}{2}}}{p_{k+\frac{1}{2}}} + \E_{\hat{\rvx}\sim \hat{p}_{k+\frac{1}{2}}}\left[\KL{\hat{p}_{b|k+\frac{1}{2}}(\cdot|\hat{\rvx})}{p_{b|k+\frac{1}{2}}(\cdot|\hat{\rvx})}\right]\\
            &\quad + \E_{(\hat{\rvx},\hat{\rvb})\sim \hat{p}_{k+\frac{1}{2},b}}\left[\KL{\hat{p}_{k+1|k+\frac{1}{2},b}(\cdot |\hat{\rvx},\hat{\rvb})}{p_{k+1|k+\frac{1}{2},b}(\cdot |\hat{\rvx},\hat{\rvb})}\right]
        \end{aligned}
    \end{equation*}
    where the first and the second equations are established by plugging 
    \begin{equation*}
        \rvx \coloneqq \hat{\rvx}_{k+1},\ \rvz\coloneqq \big(\hat{\rvx}_{k+\frac{1}{2}},\hat{\rvb}_k\big),\ \tilde{\rvx}\coloneqq \rvx_{k+1}\ \mathrm{and}\ \tilde{\rvz}\coloneqq \big(\rvx_{k+\frac{1}{2}},\rvb_k\big)
    \end{equation*}
    and 
    \begin{equation*}
        \rvx = \hat{\rvx}_{k+\frac{1}{2}},\ \rvz\coloneqq \hat{\rvb}_k,\ \tilde{\rvx}\coloneqq \rvx_{k+\frac{1}{2}}\ \mathrm{and}\ \tilde{\rvz}\coloneqq \rvb_k,
    \end{equation*}
    to Lemma~\ref{lem:kl_chain_rule}, respectively.
    Then, by requiring
    \begin{equation}
        \label{eq:sample_stra_b_prac}
        \hat{p}_{b|k+\frac{1}{2}}(\cdot|\vx)= p_{b|k+\frac{1}{2}}(\cdot|\vx)=p_b\quad \forall \vx\in\R^d,
    \end{equation}
    we have
    \begin{equation}
        \label{ineq:trans_ker_error_stage2_prac}
        \begin{aligned}
            &\KL{\hat{p}_{k+1}}{p_{k+1}} \le  \KL{\hat{p}_{k+\frac{1}{2}}}{p_{k+\frac{1}{2}}} + \E_{(\hat{\rvx},\hat{\rvb})\sim \hat{p}_{k+\frac{1}{2},b}}\left[\KL{\hat{p}_{k+1|k+\frac{1}{2},b}(\cdot |\hat{\rvx},\hat{\rvb})}{p_{k+1|k+\frac{1}{2},b}(\cdot |\hat{\rvx},\hat{\rvb})}\right]\\
            &\le  \KL{\hat{p}_{k+\frac{1}{2}}}{p_{k+\frac{1}{2}}} + \delta_k
        \end{aligned}
    \end{equation}
    where the last inequality follows from Eq~\ref{def:transition_kernel_norgo}.
    Besides, considering the first stage of the update, we have
    \begin{equation}
        \label{ineq:trans_ker_error_stage1_prac}
        \begin{aligned}
            \KL{\hat{p}_{k+\frac{1}{2}}}{p_{k+\frac{1}{2}}}\le &\KL{\hat{p}_k}{p_k}+ \E_{\hat{\rvx}\sim \hat{p}_k}\left[\KL{\hat{p}_{k+\frac{1}{2}|k}(\cdot|\hat{\rvx})}{p_{k+\frac{1}{2}|k}(\cdot|\hat{\rvx})}\right]= \KL{\hat{p}_k}{p_k},
        \end{aligned}
    \end{equation}
    where the first inequality follows from Lemma~\ref{lem:kl_chain_rule} by setting
    \begin{equation*}
        \rvx = \hat{\rvx}_{k+\frac{1}{2}},\ \rvz\coloneqq \hat{\rvx}_k,\ \tilde{\rvx}\coloneqq \rvx_{k+\frac{1}{2}}\ \mathrm{and}\ \tilde{\rvz}\coloneqq \rvx_k,
    \end{equation*}
    and the second equation establishes since $\rvx_k$ and $\hat{\rvx}_k$ share the same update in the first stage shown in Eq~\ref{def:transition_kernel_stage1} and Eq~\ref{def:transition_kernel_norgo}.
    Combining Eq~\ref{ineq:trans_ker_error_stage2_prac} and Eq~\ref{ineq:trans_ker_error_stage1_prac}, we have
    \begin{equation*}
        \KL{\hat{p}_{k+1}}{p_{k+1}} \le  \KL{\hat{p}_k}{p_k} + \delta_k,
    \end{equation*}
    which implies
    \begin{equation*}
        \KL{\tilde{p}_{K}}{p_{K}} \le \sum_{i=0}^{K-1} \delta_i
    \end{equation*}
    with the telescoping sum.
    Hence, the proof is completed.
\end{proof}

\begin{lemma}
    \label{thm:tv_bound_rgo}
    Suppose Assumption~\ref{con_ass:lips_loss}-\ref{con_ass:var_bound} hold, and Alg.~\ref{alg:sps} satisfy:
    \begin{itemize}
        \item The step sizes have $\eta_i\le 1/(2L)$ for all $i\in\{0,1,\ldots, K-1\}$.
        \item The initial particle $\hat{\rvx}_0$ is drawn from the standard Gaussian distribution on $\R^d$.
        \item The transition kernel at Line 5 of Alg.~\ref{alg:sps}, i.e., $\hat{p}_{k+1|k+\frac{1}{2},b}(\cdot|\vx,\vb)$, satisfies Eq~\ref{def:transition_kernel_norgo} and $\delta_k=0$.
    \end{itemize}
    Then, we have
    \begin{small}
    \begin{equation*}
        \TVD{\hat{p}_{K}}{p_*} \le \sigma \sqrt{\sum_{i=0}^{K-1} \frac{\eta_i}{2|\rvb_i|}} + \sqrt{\frac{(1+L^2)d}{4\alpha_*}} \cdot \prod_{i=0}^{K-1} \left(1+\alpha_* \eta_i\right)^{-1}. 
    \end{equation*}
    \end{small}
\end{lemma}
\begin{proof}
    When $\delta_k = 0$, the Markov process $\{\hat{\rvx}_k\}$ shares the same underlying distribution as the Markov process $\{{\rvx}_k\}$ .
    We consider to upper bound the total variation distance between $p_{K}$ and $p_*$ which satisfies
    \begin{equation}
        \label{ineq:init_ineq}
        \TVD{p_{K}}{p_*} \le \TVD{p_{K}}{\tilde{p}_{K}}+\TVD{\tilde{p}_{K}}{p_*}.
    \end{equation}
    According to Lemma~\ref{lem:traj_error_cmp_fullgrad}, by requiring $\eta_i\le 1/(2L)$ for all $i\in\{0,1,\ldots, K-1\}$, we have
    \begin{equation}
        \label{ineq:exact_tem1_upb}
        \TVD{p_{K}}{\tilde{p}_{K}}\le \sigma \sqrt{\sum_{i=0}^{K-1} \frac{\eta_i}{2|\rvb_i|}}.
    \end{equation}
    Besides, for $\TVD{\tilde{p}_{K}}{p_*}$ in Eq~\ref{ineq:init_ineq}, we have
    \begin{equation}
        \label{ineq:exact_tem2_upb}
        \begin{aligned}
            &\TVD{\tilde{p}_{K}}{p_*} \le \sqrt{\frac{1}{2}\KL{\tilde{p}_{K}}{p_*}}\\
            &\le \sqrt{\frac{1}{2}\KL{\tilde{p}_0}{p_*}}\cdot \prod_{i=0}^{K-1} \left(1+\alpha_* \eta_i\right)^{-1} \le \sqrt{\frac{(1+L^2)d}{4\alpha_*}} \cdot \prod_{i=0}^{K-1} \left(1+\alpha_* \eta_i\right)^{-1}
        \end{aligned}
    \end{equation}
    where the first inequality follows from Pinsker's inequality, the second follows from Lemma~\ref{lem:thm3_chen2022improved}, and the last inequality follows from Lemma~\ref{lem:init_error_bound} when we set $p_0$ as the standard Gaussian in $\R^d$.
    Finally, plugging Eq~\ref{ineq:exact_tem1_upb} and Eq~\ref{ineq:exact_tem2_upb} to Eq~\ref{ineq:init_ineq}, the proof is completed.
\end{proof}

\begin{proof}[Proof of Theorem~\ref{thm:comp_conv_sps}]
    Using the notations presented in Section~\ref{sec:not_ass_0x}, we consider to upper bound the total variation distance between $\hat{p}_{K+1}$ and $p_*$ which satisfies
    \begin{equation*}
        \TVD{\hat{p}_{K}}{p_*} \le \TVD{\hat{p}_{K}}{p_{K}}+\TVD{p_{K}}{p_*}.
    \end{equation*}
    According to Lemma~\ref{lem:traj_error_cmp_norgo}, we have
    \begin{equation}
        \label{ineq:inexact_rgo_bound}
        \TVD{\hat{p}_{K}}{p_{K}}\le \sqrt{\frac{1}{2}\sum_{i=0}^{K-1} \delta_i}.
    \end{equation}
    Besides, we have
    \begin{equation}
        \label{ineq:batch_randomness_bound}
        \TVD{p_{K}}{p_*} \le \sigma \sqrt{\sum_{i=0}^{K-1} \frac{\eta_i}{2|\rvb_i|}} + \sqrt{\frac{(1+L^2)d}{4\alpha_*}} \cdot \prod_{i=0}^{K-1} \left(1+\alpha_* \eta_i\right)^{-1}
    \end{equation}
    with Lemma~\ref{thm:tv_bound_rgo}. 
    Here, we should note the gradient complexity of Alg~\ref{alg:sps} will be dominated by Line 5, i.e., the inner sampler which requires $\mathrm{GC}(|\rvb_k|, \delta_k)$ at the $k$-th iteration. 
    Therefore, the total gradient complexity will be 
    \begin{equation*}
        \mathcal{O}\left(\sum_{i=0}^{K-1} \mathrm{GC}(|\rvb_i|, \delta_i)\right)
    \end{equation*}
    and the proof is completed.
\end{proof}

\section{Theorems for Different Implementations}

\subsection{Stochastic Gradient Langevin Dynamics Inner Samplers}
\label{app_sec:innersgld}

\begin{lemma}
    \label{lem:prop2_durmus2019analysis}
    Using the notations presented in Alg~\ref{alg:sgld_inner}, asume~\ref{con_ass:lips_loss}-\ref{con_ass:var_bound}, for any $\tau_s\in(0,\frac{1}{36}]$, we have 
    \begin{equation*}
        \begin{aligned}
            2\tau_s \cdot \KL{q^\prime_{s}}{p_{k+1|k+\frac{1}{2},b}(\cdot|\vx_0, \vb)}\le & \left(1-\frac{\tau_s}{4\eta}\right)\cdot  W_2^2(q_s, p_{k+1|k+\frac{1}{2},b}(\cdot|\vx_0,\vb))\\
            & - W_2^2(q_{s+1},p_{k+1|k+\frac{1}{2},b}(\cdot|\vx_0,\vb))+  \frac{4\tau_s^2 \sigma^2}{|\rvb_s|} + \frac{6\tau_s^2 d}{\eta}
        \end{aligned}
    \end{equation*}
    where $q_s$, $q^\prime_s$ and $q_*$ denotes underlying distribution of $\rvz_s$, $\rvz_s^\prime$ and the ideal output particles.
\end{lemma}
\begin{proof}
    This proof only considers the KL divergence behavior for the $k$-th inner sampling subproblem, i.e., Line 5 of Alg~\ref{alg:sps}.
    The target distribution of the inner loop, i.e., $p_{k+1|k+1/2, b}(\cdot|\vx_0,\vb)$ will be abbreviated as
    \begin{equation*}
        q_{*}(\vz)\coloneqq C_q^{-1}\cdot \exp(-g(\vz)) = C_q^{-1}\cdot  \exp\left(-f_{\vb}(\vz) - \frac{\left\|\vz - \vx_0\right\|^2}{2\eta}\right).
    \end{equation*}
    Since $\mathsf{InnerSGLD}$ sample mini-batch $\rvb_s$ from $\vb$ for all $s\in \{1,2,\ldots, S\}$, we define 
    \begin{equation*}
        g_{\vb_s}(\vz) \coloneqq - \frac{1}{|\vb_s|}\sum_{i\in \vb_s} f_i(\vz) - \frac{\left\|\vz - \vx_0\right\|^2}{2\eta}.
    \end{equation*}
    Combining Lemma~\ref{lem:sc_sm_inner_targets} and the choice of the step size, i.e., $\eta\le 1/2L$, we have
    \begin{equation*}
        (2\eta)^{-1}\cdot \mI \preceq \grad^2 g(\vz) = \grad^2 q_*(\vz) \preceq (3/2\eta)\cdot \mI.
    \end{equation*}
    Suppose the underlying distribution of $\rvz_s$ and $\rvz_s^\prime$ are $q_s$ and $q_s^\prime$ respectively.
    Besides, the KL divergence between $q_s$ and $q_*$ is
    \begin{equation*}
        \begin{aligned}
            \KL{q_s}{q_*} = \int q_s(\vz) \log \frac{q_s(\vz)}{q_*(\vz)} \der \vz = \underbrace{\int q_s(\vz) \log q_s(\vz)\der \vz}_{\mathcal{H}(q_s)} + \underbrace{\int q_s(\vz)\left( g(\vz)+ \log C_q\right) \der \vz}_{\mathcal{E}(q_s)}.
        \end{aligned}
    \end{equation*}
    Then we consider the dynamics of entropy $\mathcal{H}$ and energy $\mathcal{E}$ functionals with the iteration presented as
    \begin{equation*}
        \begin{aligned}
            &\rvz^\prime_{s} = \rvz_{s}+\sqrt{2\tau_s\cdot \left(1-\frac{\tau_s}{4\eta}\right)^{-1}}\xi \quad \mathrm{where}\quad \xi \sim \mathcal{N}(\vzero, \mI),\\
            & \rvz_{s+1} = \rvz^\prime_{s} - \tau_s \grad g_{\rvb_s}\left(\rvz^\prime_{s}\right).
        \end{aligned}
    \end{equation*}

    \paragraph{Energy functional dynamics}
    We start with the following inequality
    \begin{equation*}
        \begin{aligned}
            &W_2^2(q_{s+1},q_*) \le \E_{(\rvz_s^\prime,\rvz_*)\sim \gamma^\prime_{s}}\left[\E_{\rvz_{s+1}\sim q^\prime_{s+1|s}(\cdot|\rvz_s^\prime)}\left\|\rvz_{s+1}- \rvz\right\|^2\right],
        \end{aligned}
    \end{equation*}
    where $\gamma^\prime_{s}$ denotes the optimal coupling between the densities $q^\prime_{s}$ and $q_*$, and $q^\prime_{s+1|s}(\cdot|\vz_s^\prime)$ denotes the density function for $\rvz_{s+1}$ when $\rvz_s^\prime = \vz^\prime_s$.
    According to the change of variables, the inner expectation on the RHS satisfies
    \begin{equation}
        \label{ineq:w2d_decomposed_bound}
        \begin{aligned}
            &\E_{\rvz_{s+1}\sim q^\prime_{s+1|s}(\cdot|\vz^\prime_s)}\left\|\rvz_{s+1}- \vz\right\|^2 = \sum_{\vb_s \subseteq \vb} p_b(\vb_s)\cdot \left\|\vz_s^\prime - \tau_s\grad g_{\vb_s}(\vz_s^\prime)-\vz\right\|^2 \\
            & = \left\|\vz_s^\prime - \vz\right\|^2 - 2\tau_s \left<\grad g(\vz_s^\prime), \vz_s^\prime-\vz\right>+ \tau_s^2 \E_{\rvb_s}\left\|\grad g_{\rvb_s}(\vz_s^\prime)\right\|^2\\
            & \le \left(1- \frac{\tau_s}{2\eta}\right)\cdot \left\|\vz_s^\prime- \vz\right\|^2 - 2\tau_s \cdot \left(g(\vz_s^\prime)- g(\vz)\right) + \tau_s^2 \E_{\rvb_s}\left\|\grad g_{\rvb_s}(\vz_s^\prime)\right\|^2,
        \end{aligned}
    \end{equation}
    where the last inequality follows from the strong convexity of $g$, i.e.,
    \begin{equation*}
        g(\vz) - g(\vz_s^\prime)\ge \left<\grad g(\vz_s^\prime), \vz-\vz_s^\prime\right> + \frac{1}{4\eta}\cdot \left\|\vz - \vz_s^\prime\right\|^2.
    \end{equation*}
    Taking the expectation for both sides of Eq~\ref{ineq:w2d_decomposed_bound}, we have
    \begin{equation}
        \label{ineq:w2d_bound}
        \begin{aligned}
            \E_{(\rvz^\prime_{s},\rvz)\sim \gamma^\prime_{s}}\left[\E_{q^\prime_{s+1|s}(\cdot|\rvz_s^\prime)}\left\|\rvz_{s+1}- \rvz\right\|^2\right]\le & \left(1-\frac{\tau_s}{2\eta}\right)W^2_2(q^\prime_{s}, q_*) -2\tau_s \cdot \left(\mathcal{E}(q^\prime_{s})- \mathcal{E}(q_*)\right)\\
            & + \tau_s^2 \cdot \E_{(\rvz^\prime_{s},\rvz)\sim \gamma^\prime_{s}} \left[\E_{\rvb_s}\left\|\grad g_{\rvb_s}(\rvz^\prime_{s})\right\|^2\right].
        \end{aligned}
    \end{equation}

    Then, we start to upper bound the last term of Eq~\ref{ineq:w2d_bound}, and have
    \begin{equation}
        \label{ineq:w2d_bound_lastterm}
        \begin{aligned}
            &\E_{(\rvz^\prime_{s},\rvz)\sim \gamma^\prime_{s}} \left[\E_{\rvb_s}\left\|\grad g_{\rvb_s}(\rvz^\prime_{s})\right\|^2\right] = \E_{(\rvz^\prime_{s},\rvz)\sim \gamma^\prime_{s}} \left[\E_{\rvb_s}\left\|\grad g_{\rvb_s}(\rvz^\prime_{s}) - \grad g_{\rvb_s}(\rvz)+ \grad g_{\rvb_s}(\rvz)\right\|^2\right]\\
            &\le 2\E_{(\rvz^\prime_{s},\rvz)\sim \gamma^\prime_{s}} \left[\E_{\rvb_s}\left\|\grad g_{\rvb_s}(\rvz^\prime_{s}) - \grad g_{\rvb_s}(\rvz)\right\|^2\right] + 2\E_{(\rvz^\prime_{s},\rvz)\sim \gamma^\prime_{s}}\left[\E_{\rvb_s}\left\|\grad g_{\rvb_s}(\rvz)-\grad g(\rvz)+\grad g(\rvz)\right\|^2\right]\\
            & \le 2 \cdot \left(\frac{3}{2\eta}\right)^2\cdot \E_{(\rvz^\prime_{s},\rvz)\sim \gamma^\prime_{s}} \left\|\rvz^\prime_{s}-\rvz\right\|^2 + 4\E_{(\rvz^\prime_{s},\rvz)\sim \gamma^\prime_{s}} \left[\E_{\rvb_s}\left\|\grad g_{\rvb_s}(\rvz)-\grad g(\rvz)\right\|^2\right]+ 4\E_{\rvz\sim q_*}\left[\left\|\grad g(\rvz)\right\|^2\right].
        \end{aligned}
    \end{equation}
    For the first term, with the definition of $\gamma^\prime_{s}$, we have
    \begin{equation*}
        \E_{(\rvz^\prime_{s},\rvz)\sim \gamma^\prime_{s}} \left\|\rvz^\prime_{s}-\rvz\right\|^2 = W_2^2(q^\prime_{s}, q_*).
    \end{equation*}
    For the second one, suppose we sample $\rvb_s$ uniformly from $\vb$ sharing the same sampler number for all $s\in \{1,2,\ldots, S\}$, i.e., $b_{\mathrm{in}}$.
    Then, for any $\vz\in\R^d$, we have
    \begin{equation*}
        \begin{aligned}
            \E_{\rvb_s}\left\|\grad g_{\rvb_s}(\vz)-\grad g(\vz)\right\|^2 = \E_{\rvb_s}\left[\left\|\grad f_{\rvb_s}(\vz)-\grad f(\vz)\right\|^2\right] \le \frac{\sigma^2}{b_{\mathrm{in}}}
        \end{aligned}
    \end{equation*}
    which follows from Lemma~\ref{lem:minibatch_var}. 
    It then implies
    \begin{equation*}
        \E_{(\rvz^\prime_{s},\rvz)\sim \gamma^\prime_{s}} \left[\E_{\rvb_s}\left\|\grad g_{\rvb_s}(\rvz)-\grad g(\rvz)\right\|^2\right] \le \frac{\sigma^2}{b_{\mathrm{in}}}.
    \end{equation*}
    For the last term, we have
    \begin{equation*}
        \E_{\rvz\sim q_*}\left[\left\|\grad g(\rvz)\right\|^2\right]\le \frac{3d}{2\eta}
    \end{equation*}
    which follows from Lemma~\ref{lem:lem11_vempala2019rapid}.
    In these conditions, Eq~\ref{ineq:w2d_bound_lastterm} can be represented as 
    \begin{equation*}
        \begin{aligned}
            \E_{(\rvz^\prime_{s},\rvz)\sim \gamma^\prime_{s}} \left[\E_{\rvb_s}\left\|\grad g_{\rvb_s}(\rvz^\prime_{s})\right\|^2\right]\le  \frac{9}{2\eta}\cdot W_2^2(q^\prime_{s}, q_*) + \frac{4\sigma^2}{b_\mathrm{in}} + \frac{6d}{\eta}.
        \end{aligned}
    \end{equation*}
    Plugging this inequality into Eq~\ref{ineq:w2d_bound}, we have
    \begin{equation*}
        \begin{aligned}
            &W_2^2(q_{s+1},q_*) \le \left(1-\frac{\tau_s}{2\eta}+\frac{9\tau_s^2}{\eta}\right)\cdot W_2^2(q^\prime_{s}, q_*) -2\tau_s \cdot \left(\mathcal{E}(q^\prime_{s})- \mathcal{E}(q_*)\right)+ \frac{4\tau_s^2 \sigma^2}{b_\mathrm{in}} + \frac{6\tau_s^2 d}{\eta},
        \end{aligned}
    \end{equation*}
    which is equivalent to 
    \begin{equation*}
        2\tau_s \cdot \left(\mathcal{E}(q^\prime_{s})- \mathcal{E}(q_*)\right) \le \left(1-\frac{\tau_s}{2\eta}+\frac{9\tau_s^2}{\eta}\right)\cdot W_2^2(q^\prime_{s}, q_*) - W_2^2(q_{s+1},q_*)+  \frac{4\tau_s^2 \sigma^2}{b_{\mathrm{in}}} + \frac{6\tau_s^2 d}{\eta}.
    \end{equation*}
    By requiring 
    \begin{equation*}
        \frac{9\tau_s^2}{\eta} \le \frac{\tau_s}{4\eta}\quad \Leftrightarrow\quad \tau_s \le \frac{1}{36},
    \end{equation*}
    we have
    \begin{equation}
        \label{ineq:energy_function_bound_1}
        2\tau_s \cdot \left(\mathcal{E}(q^\prime_{s})- \mathcal{E}(q_*)\right) \le \left(1-\frac{\tau_s}{4\eta}\right)\cdot W_2^2(q^\prime_{s}, q_*) - W_2^2(q_{s+1},q_*)+  \frac{4\tau_s^2 \sigma^2}{b_{\mathrm{in}}} + \frac{6\tau_s^2 d}{\eta}.
    \end{equation}

    \paragraph{Entropy functional bound}
    According to Lemma~\ref{lem:lem5_durmus2019analysis}, we have
    \begin{equation*}
        2\cdot \left(\left(1-\frac{\tau_s}{4\eta}\right)^{-1}\cdot \tau_s\right)\cdot \left(\mathcal{H}(q^\prime_{s}) - \mathcal{H}(q_*)\right)\le W_2^2(q_s, q_*) - W_2^2(q^\prime_{s}, q_*),
    \end{equation*}
    which is equivalent to 
    \begin{equation}
        \label{ineq:entropy_function_bound_1}
        2\tau_s \cdot \left(\mathcal{H}(q^\prime_{s}) - \mathcal{H}(q_*)\right)\le \left(1-\frac{\tau_s}{4\eta}\right)\cdot W_2^2(q_s, q_*) - \left(1-\frac{\tau_s}{4\eta}\right)\cdot W_2^2(q^\prime_{s}, q_*).
    \end{equation}

    Therefore, combining Eq~\ref{ineq:energy_function_bound_1} and Eq~\ref{ineq:entropy_function_bound_1}, we have
    \begin{equation}
        \label{ineq:descent_per_iter}
        2\tau_s \cdot \KL{q^\prime_{s}}{q_*}\le \left(1-\frac{\tau_s}{4\eta}\right)\cdot  W_2^2(q_s, q_*) - W_2^2(q_{s+1},q_*)+  \frac{4\tau_s^2 \sigma^2}{b_{\mathrm{in}}} + \frac{6\tau_s^2 d}{\eta}.
    \end{equation}
    Hence, the proof is completed.
\end{proof}

\begin{corollary}
    \label{cor:cor10_durmus2019analysis}
    Using the notations presented in Alg~\ref{alg:sgld_inner}, asume~\ref{con_ass:lips_loss}-\ref{con_ass:var_bound}.
    Define:
    \begin{equation*}
        \tau_s \coloneqq \tau \le \min\left\{\frac{\delta}{16}\cdot  \left(\frac{2\sigma^2\eta}{b_{\mathrm{in}}}+3d\right)^{-1}, \frac{1}{36}\right\},\quad S\ge  \log \frac{2W_2^2(q_1, p_{k+1|k+\frac{1}{2},b}(\cdot|\vx_0, \vb))}{\delta}\cdot 4\eta \tau^{-1},
    \end{equation*}
    where $b_{\mathrm{in}}$ denotes the uniformed minibatch size of sampled in Line 5 of Alg~\ref{alg:sgld_inner}.   
    Then, the underlying distribution of particles at $S$-th iteration, i.e., $q_S$, satisfies  $W_2^2(q_{S}, p_{k+1|k+\frac{1}{2},b}(\cdot|\vx_0, \vb))\le \delta$.
\end{corollary}
\begin{proof}
    Similar to Lemma~\ref{lem:prop2_durmus2019analysis}, the target distribution of the inner loop, i.e., $p_{k+1|k+1/2, b}(\cdot|\vx_0,\vb)$ will be abbreviated as
    \begin{equation*}
        q_{*}(\vz)\coloneqq C_q^{-1}\cdot \exp(-g(\vz)) = C_q^{-1}\cdot  \exp\left(-f_{\vb}(\vz) - \frac{\left\|\vz - \vx_0\right\|^2}{2\eta}\right).
    \end{equation*}
    and we define the minibatch loss as follows
    \begin{equation*}
        g_{\vb_s}(\vz) \coloneqq - \frac{1}{|\vb_s|}\sum_{i\in \vb_s} f_i(\vz) - \frac{\left\|\vz - \vx_0\right\|^2}{2\eta}.
    \end{equation*}
    Then, using Lemma~\ref{lem:prop2_durmus2019analysis} and since the KL divergence is non-negative, for all $s\in\{0,2,\ldots, S-1\}$, we have
    \begin{equation*}
        W_2^2(q_{s+1},q_*)\le \left(1-\frac{\tau_s}{4\eta}\right)\cdot  W_2^2(q_s, q_*)+\frac{4\tau_s^2 \sigma^2}{|\rvb_s|} + \frac{6\tau_s^2 d}{\eta}.
    \end{equation*}
    Following from a direct induction, we have 
    \begin{equation*}
        \begin{aligned}
            W_2^2(q_S, q_*) \le \left[\prod_{s=0}^{S-1} \left(1-\frac{\tau_s}{4\eta}\right)\right]W_2^2(q_0, q_*) + \sum_{i=0}^{S-1} \left(\frac{4\tau_i^2\sigma^2}{|\rvb_i|}+\frac{6\tau_i^2 d}{\eta}\right)\prod_{j=i+1}^{S-1}\left(1-\frac{\tau_j}{4\eta}\right)
        \end{aligned}
    \end{equation*}

    In this condition, we choose uniformed step and mini-batch sizes, i.e., $\tau_s = \tau$, $|\rvb_s|=b_{\mathrm{in}}$, and have
    \begin{equation}
        \label{ineq:w2d_inner_upb}
        \begin{aligned}
            W_2^2(q_S, q_*) \le &\left(1-\frac{\tau}{4\eta}\right)^S \cdot W_2^2(q_0, q_*) + \left(\frac{4\sigma^2}{b_{\mathrm{in}}}+\frac{6d}{\eta}\right)\sum_{i=0}^{S-1} \tau^2\left(1-\frac{\tau}{4\eta}\right)^{i}\\
            \le & \left(1-\frac{\tau}{4\eta}\right)^{S} \cdot W_2^2(q_0, q_*) + \left(\frac{2\sigma^2\eta }{b_{\mathrm{in}}}+3d\right) \cdot 8\tau.
        \end{aligned}
    \end{equation}
    Using that for all $u\in \R_+$, $1-u\le \exp(-u)$, then it has
    \begin{equation}
        \label{ineq:time_choice}
        \left(1-\frac{\tau}{4\eta}\right)^{S}W_2^2(q_1, q_*) \le \exp\left(-\frac{\tau S}{4\eta}\right)W_2^2(q_0, q_*) \le \frac{\delta}{2}.
    \end{equation}
    Without loss of generality, the iteration number of inner loop will be large, which implies the last inequality of Eq~\ref{ineq:time_choice} will establish by requiring 
    \begin{equation*}
        \tau S \ge \log \frac{2W_2^2(q_1, p_{k+1|k+\frac{1}{2},b}(\cdot|\vx_0, \vb))}{\delta} \cdot 4\eta.
    \end{equation*}
    In the following, we choose the value of $\tau S$ to be the lower bound.
    Besides, we require the last term of Eq~\ref{ineq:w2d_inner_upb} to satisfy
    \begin{equation}
        \label{ineq:step_choice}
        \begin{aligned}
            \left(\frac{2\sigma^2\eta }{b_{\mathrm{in}}}+3d\right) \cdot 8\tau \le \frac{\delta}{2} \quad\Leftrightarrow \quad \tau\le \frac{\delta}{16}\cdot  \left(\frac{2\sigma^2\eta}{b_{\mathrm{in}}}+3d\right)^{-1}.
        \end{aligned}
    \end{equation}
    Combining Eq~\ref{ineq:time_choice} and Eq~\ref{ineq:step_choice}, the proof is completed.
\end{proof}

\begin{lemma}
    \label{lem:cor11_durmus2019analysis}
    Using the notations presented in Alg~\ref{alg:sgld_inner}, asume~\ref{con_ass:lips_loss}-\ref{con_ass:var_bound}.
    Define
    \begin{equation*}
        S^\prime \ge  \log \frac{2W_2^2(q_1, p_{k+1|k+\frac{1}{2},b}(\cdot|\vx_0, \vb))}{\delta}\cdot 4\eta \tau^{-1}\quad \mathrm{and}\quad S^\prime \in \mathbb{N}_+,
    \end{equation*}
    for all $s\in[0,S^\prime]$, the step sizes and sample sizes satisfy
    \begin{equation*}
        |\rvb_s| = b_{\mathrm{in}}\quad \mathrm{and}\quad \tau_s \coloneqq \tau \le  \min\left\{\frac{\delta}{16}\cdot  \left(\frac{2\sigma^2\eta}{b_{\mathrm{in}}}+3d\right)^{-1}, \frac{1}{36}\right\}
    \end{equation*}
    in Alg~\ref{alg:sgld_inner}.
    Besides, for $s\in[S^\prime+1, S]$, the step sizes and sampler sizes are
    \begin{equation*}
        |\rvb_s| = b^\prime_{\mathrm{in}}\quad \mathrm{and}\quad \tau_s \coloneqq \tau^\prime \le \min\left\{\frac{\delta}{2} \cdot \left(\frac{2\sigma^2}{b^\prime_{\mathrm{in}}} + \frac{3d}{\eta}\right)^{-1}, \frac{1}{36}\right\}.
    \end{equation*}
    In this condition, if the total iteration number $S$ satisfies
    \begin{equation*}
        S\ge S^\prime + (\tau^\prime)^{-1}\quad \mathrm{and}\quad S\in\mathbb{N}_+,
    \end{equation*}
    then the underlying distribution $\overline{q}_S$ of output particles satisfies $\KL{\overline{q}_S}{p_{k+1|k+\frac{1}{2},b}(\cdot|\vx_0,\vb)}\le \delta$.
\end{lemma}
\begin{proof}
    We first introduce $0< S^\prime < S$ satisfying $S^\prime\in \mathbb{N}_+$, and denote the underlying distribution of output particles as
    \begin{equation*}
        \overline{q}_{S} = \frac{\sum_{i=S^\prime+1}^{S} q^\prime_{i}}{S-S^\prime}\quad \mathrm{where}\quad i\in \mathbb{N}_+
    \end{equation*}
    and $q^\prime_{i}$ denotes the underlying distribution of $\rvz^\prime_{i}$ in Alg~\ref{alg:sgld_inner}.
    Similar to Lemma~\ref{lem:prop2_durmus2019analysis}, the target distribution of the inner loop, i.e., $p_{k+1|k+1/2, b}(\cdot|\vx_0,\vb)$ will be abbreviated as $q_*(\cdot)$.
    Then, we set all step and sample sizes between $S^\prime$-th to $S$-th iteration are uniformed $\tau^\prime$ and $b^\prime_{\mathrm{in}}$.
    In this condition, we have
    \begin{equation}
        \label{ineq:kl_final_upb}
        \begin{aligned}
            & \KL{\overline{q}_S}{q_*}\le \frac{1}{S-S^\prime}\cdot \sum_{i=S^\prime+1}^{S} \KL{q^\prime_{i}}{q_*}\\
            & \le \frac{1}{2\tau^\prime(S-S^\prime)}\cdot \left[ \left(1-\frac{\tau^\prime}{4\eta}\right)W_2^2(q_{S^\prime+1}, q_*) -\sum_{i=S^\prime+2}^{S} \frac{\tau^\prime}{4\eta}\cdot W_2^2(q_{i}, q_*) - W_2^2(q_{S +1},q_*)\right.\\
            &\quad \left.+ (S-S^\prime )\cdot \left(\frac{4(\tau^\prime)^2 \sigma^2}{b^\prime_{\mathrm{in}}}+\frac{6(\tau^\prime)^2 d}{\eta}\right)\right]\\
            & \le \frac{W_2^2(q_{S^\prime+1}, q_*)}{2 \tau^\prime (S-S^\prime)} + \frac{2\tau^\prime \sigma^2}{b^\prime_{\mathrm{in}}} + \frac{3\tau^\prime d}{\eta}
        \end{aligned}
    \end{equation}
    where the first inequality follows from Lemma~\ref{lem:convexity_KL} and the second inequality follows from  Lemma~\ref{lem:prop2_durmus2019analysis}.
    According to Corollary~\ref{cor:cor10_durmus2019analysis}, in Alg~\ref{alg:sgld_inner}, if we set
    \begin{equation*}
        \begin{aligned}
            \tau_s \coloneqq \tau \le  \min\left\{\frac{\delta}{16}\cdot  \left(\frac{2\sigma^2\eta}{b_\mathrm{in}}+3d\right)^{-1}, \frac{1}{36}\right\},\quad S^\prime \ge  \log \frac{2W_2^2(q_1, q_*)}{\delta}\cdot 4\eta \tau^{-1}.
        \end{aligned}
    \end{equation*} 
    for all $s\in[0,S^\prime]$, then we have $W_2^2(q_{S^\prime+1}, q_*)\le \delta$.
    In this condition, by requiring 
    \begin{equation*}
        \tau^\prime ( S- S^\prime)\ge 1,\quad \mathrm{and}\quad \tau^\prime \le  \frac{\delta}{2} \cdot \left(\frac{2\sigma^2}{b^\prime_{\mathrm{in}}} + \frac{3d}{\eta}\right)^{-1},
    \end{equation*}
    the first and the second term of Eq~\ref{ineq:kl_final_upb} will satisfies
    \begin{equation*}
        \frac{W_2^2(q_{S^\prime+1}, q_*)}{2 \tau^\prime (S-S^\prime)} \le \frac{\delta}{2},\quad \mathrm{and}\quad \frac{2\tau^\prime \sigma^2}{b^\prime_{\mathrm{in}}} + \frac{3\tau^\prime d}{\eta}\le \frac{\delta}{2}.
    \end{equation*}
    Hence, the proof is completed.
\end{proof}

\begin{theorem}[Formal version of Theorem~\ref{thm:conv_gra_comp_innerSGLD_informal}]
    \label{thm:conv_gra_comp_innerSGLD}
    Suppose~\ref{con_ass:lips_loss}-\ref{con_ass:var_bound} hold. 
    With the following parameter settings
    \begin{equation*}
        \begin{aligned}
            &\eta_k = \frac{1}{2L},\quad K = \frac{L}{\alpha_*}\cdot \log \frac{(1+L^2)d}{4\alpha_* \epsilon^2},\quad \delta_k = \frac{2\epsilon^2 \alpha_*}{L}\cdot \left(\log \frac{(1+L^2)d}{4\alpha_* \epsilon^2}\right)^{-1}\\
            &b_o = \min\left\{\frac{\sigma^2}{4 \alpha_* \epsilon^2}\cdot \log \frac{(1+L^2)d}{4\alpha_* \epsilon^2},n\right\},
        \end{aligned}
    \end{equation*}
    for Alg~\ref{alg:sps}, if we choose Alg~\ref{alg:sgld_inner} as the inner sampler shown in Line 5 Alg~\ref{alg:sps}, set 
    \begin{equation*}
        \begin{aligned}
            &\tau =  \min\left\{\frac{\alpha_*\epsilon^2}{16} \cdot \left( \left(\sigma^2 + 3Ld\right)\cdot \log \frac{(1+L^2)d}{4\alpha_* \epsilon^2}\right)^{-1}, \frac{1}{36}\right\},\\
            &\tau^\prime = \min\left\{\frac{\alpha_*\epsilon^2}{4L}\cdot \left(\left(\sigma^2 +3Ld \right)\cdot \log \frac{(1+L^2)d}{4\alpha_* \epsilon^2 }\right)^{-1}, \frac{1}{36}\right\},\\
            & S^\prime(\vx_0, \vb) = \left(\log \left( \frac{\left\|\grad f_{\vb}(\vzero)\right\|^2 + L + L\left\|\vx_0\right\|^2}{L\alpha_* \epsilon^2}\right) +  \log \log  \frac{(1+L^2)d}{4\alpha_* \epsilon^2 }\right)\cdot \frac{4}{L \tau}\\
            & S(\vx_0, \vb) = \left(\log \left(\frac{\left\|\grad f_{\vb}(\vzero)\right\|^2 + L + L\left\|\vx_0\right\|^2}{L\alpha_* \epsilon^2}\right) +  \log \log  \frac{(1+L^2)d}{4\alpha_* \epsilon^2 }\right)\cdot \frac{4}{L \tau} + (\tau^\prime)^{-1},\\
            &\tau_s = \tau \quad \mathrm{when}\quad s\in[0,S^\prime(\vx_0, \vb)]\\
            &\tau_s = \tau^\prime \quad \mathrm{when}\quad s\in[S^\prime(\vx_0, \vb)+1, S(\vx_0, \vb)-1]
        \end{aligned}
    \end{equation*}
    and $1$ inner minibatch size, i.e., $b_{\mathrm{in}}=1$, then the underlying distribution of returned particles $\hat{p}_K$ in Alg~\ref{alg:sps} satisfies $\TVD{\hat{p}_{K+1}}{p_*}<3\epsilon$. 
    In this condition, the expected gradient complexity will be 
    \begin{equation*}
        \frac{34L^3(\sigma^2 + 3d)}{\alpha_*^3 \epsilon^2}\cdot \log(24L^2)\cdot \log^2 \frac{(1+L^2)d}{4\alpha_* \epsilon^2}\cdot \log \frac{30L^2\left(M+\sigma^2 + d+1+\left\|\grad f(\vzero)\right\|^2\right)}{\alpha_*\epsilon^2},
    \end{equation*}
    which can be abbreviated as $\tilde{\Theta}(\kappa^3 \epsilon^{-2}\cdot (d+\sigma^2))$.
\end{theorem}
\begin{proof}
    For the detailed implementation of Alg~\ref{alg:sps} with Alg~\ref{alg:sgld_inner}, 
    we consider the following settings.
    \begin{itemize}
        \item For all $k\in\{0,1,\ldots, K-1\}$, the mini-batch $\rvb_k$ in Alg~\ref{alg:sps} Line 2 has a uniformed norm which is denoted as $|\rvb_k| = b_o$.
        \item For all $k\in\{0,1,\ldots, K-1\}$, the conditional probability densities $p_{k+1|k+1/2,b}(\cdot|\rvx_{k+1/2},\vb_k)$ in Alg~\ref{alg:sps} Line 4 formulated as Eq~\ref{def:transition_kernel_stage2} share the same $L$-2 regularized coefficients, i.e., $\eta_k^{-1}$.
        \item For all $k\in\{0,1,\ldots, K-1\}$, the inner sampler shown in Alg~\ref{alg:sps} Line 5 is chosen as Alg~\ref{alg:sgld_inner}.
    \end{itemize}
    \paragraph{Errors control of outer loops.}
    With these conditions, we have
    \begin{equation*}
        \TVD{\hat{p}_{K}}{p_*} \le \sqrt{\frac{1}{2}\sum_{i=0}^{K-1} \delta_i} +  \sigma \sqrt{\frac{K\eta}{2b_o}} + \sqrt{\frac{(1+L^2)d}{4\alpha_*}} \cdot \left(1+\alpha_* \eta\right)^{-K}
    \end{equation*}
    which follows from Theorem~\ref{thm:comp_conv_sps}. 
    For achieving $\TVD{p_{K+1}}{p_*} \le \tilde{O}(\epsilon)$, we start with choosing the step size $\eta$ and the iteration number $K$ in Alg~\ref{alg:sps}.
    By requiring
    \begin{equation}
        \label{ineq:eta_K_choice_alg1}
        \eta \le \frac{1}{2L}\quad \mathrm{and}\quad K\ge (\alpha_*\eta)^{-1 }\cdot \log \frac{(1+L^2)d}{4\alpha_* \epsilon^2} = \frac{2L}{\alpha_*}\cdot \log \frac{(1+L^2)d}{4\alpha_* \epsilon^2},
    \end{equation}
    we have
    \begin{equation*}
        \begin{aligned}
            \left(1+\alpha_* \eta\right)^{2K}\ge \exp(\alpha_*\eta K) \ge \frac{(1+L^2)d}{4\alpha_* \epsilon^2}\quad \Rightarrow\quad \exp(-\alpha_* K\eta) \le \epsilon,
        \end{aligned}
    \end{equation*}
    where the first inequality follows from $1+u\ge \exp(u/2)$ when $u\le 1$.
    The last equation of Eq~\ref{ineq:eta_K_choice_alg1} establishes when $\eta$ is chosen as its upper bound.
    Besides by requiring
    \begin{equation}
        \label{ineq:b_choice_alg1}
        b_o\ge \min\left\{\frac{K\eta \sigma^2 }{2\epsilon^2}, n\right\} = \min\left\{\frac{\sigma^2}{\alpha_* \epsilon^2}\cdot \log \frac{(1+L^2)d}{4\alpha_* \epsilon^2},n\right\},
    \end{equation}
    we have $\sigma \sqrt{K\eta/(2b_o)}\le \epsilon$.
    The last equation of Eq~\ref{ineq:b_choice_alg1} requires the choice of $\eta$ and $K$ in Eq~\ref{ineq:eta_K_choice_alg1} to be their upper and lower bound respectively.
    For simplicity, we consider inner samplers for all iterations share the same error tolerance, i.e., $\delta_k = \delta$ for all $k\in\{1,2,\ldots, K\}$.
    By requiring,
    \begin{equation}
        \label{ineq:delta_choice_alg1}
        \delta \le \frac{2\epsilon^2}{K} = \frac{\epsilon^2 \alpha_*}{L}\cdot \left(\log \frac{(1+L^2)d}{4\alpha_* \epsilon^2}\right)^{-1}
    \end{equation}
    we have $\sqrt{\frac{1}{2}\sum_{i=0}^{K-1} \delta_i} \le \epsilon$.
    The last inequality of Eq~\ref{ineq:delta_choice_alg1} holds when $K$ is chosen as its lower bound in Eq~\ref{ineq:eta_K_choice_alg1}.
 
    \paragraph{Errors control of inner loops.}
    Then, we start to consider the hyper-parameter settings of the inner loop and the total gradient complexity.
    According to Theorem~\ref{thm:comp_conv_sps}, we require the underlying distribution of output particles of the inner loop, i.e., $\hat{p}_{k+1|k+\frac{1}{2},b}(\cdot|\vx_0,\vb)$, satisfies
    \begin{equation}
        \label{ineq:inner_kl_con}
        \KL{\hat{p}_{k+1|k+\frac{1}{2},b}(\cdot|\vx_0,\vb)}{{p}_{k+1|k+\frac{1}{2},b}(\cdot|\vx_0,\vb)} \le \delta \le  \frac{\epsilon^2 \alpha_*}{L}\cdot \left(\log \frac{(1+L^2)d}{4\alpha_* \epsilon^2}\right)^{-1}
    \end{equation}
    for all $\vx_0\in\R^d$ and $\vb\subseteq \{1,2,\ldots,n\}$.
    Then, to achieve Eq~\ref{ineq:inner_kl_con}, Lemma~\ref{lem:cor11_durmus2019analysis} will decompose the total inner iterations of Alg~\ref{alg:sgld_inner}, i.e., $s\in[0,S(\vx_0,\vb)]$ into two stages.
    
    For the first stage, we consider 
    \begin{equation}
        \label{ineq:tau_choice_alg2}
        \begin{aligned}
            \tau_s \coloneqq \tau \le  \min\left\{\frac{\delta}{16}\cdot  \left(\frac{2\sigma^2\eta}{b_\mathrm{in}}+3d\right)^{-1}, \frac{1}{36}\right\} = \min\left\{\frac{\alpha_*\epsilon^2}{16} \cdot \left( \left(\sigma^2 + 3Ld\right)\cdot \log \frac{(1+L^2)d}{4\alpha_* \epsilon^2}\right)^{-1}, \frac{1}{36}\right\}
        \end{aligned}
    \end{equation}
    for $s\in[0, S^\prime(\vx_0, \vb)]$ where 
    \begin{equation}
        \label{ineq:sprime_choice_alg2}
        S^\prime(\vx_0,\vb) \ge \left(\log \frac{2L\cdot W_2^2(q_0, p_{k+1|k+\frac{1}{2},b}(\cdot|\vx_0, \vb))}{\alpha_* \epsilon^2} +  \log \log  \frac{(1+L^2)d}{4\alpha_* \epsilon^2 }\right)\cdot \frac{2}{L \tau}\quad \mathrm{and}\quad S^\prime(\vx_0,\vb) \in \mathbb{N}_+.
    \end{equation}
    It should be noted that the last equation of Eq~\ref{ineq:tau_choice_alg2} only establishes when $\delta$ and $\eta$ are chosen as their upper bounds, and $b_{\mathrm{in}}=1$.
    
    For the second stage, we consider 
    \begin{equation}
        \label{ineq:tauprime_choice_alg2}
        \begin{aligned}
            \tau_s \coloneqq \tau^\prime \le  \min\left\{\frac{\delta}{2} \cdot \left(\frac{2\sigma^2}{b^\prime_{\mathrm{in}}} + \frac{3d}{\eta}\right)^{-1}, \frac{1}{36}\right\} = \min\left\{\frac{\alpha_*\epsilon^2}{4L}\cdot \left(\left(\sigma^2 +3Ld \right) \log \frac{(1+L^2)d}{4\alpha_* \epsilon^2 }\right)^{-1}, \frac{1}{36}\right\}.
        \end{aligned}
    \end{equation}
    for $s\in [S^\prime(\vx_0,\vb)+1, S(\vx_0,\vb)-1]$ where 
    \begin{equation}
        \label{ineq:S_choice_alg2}
        \begin{aligned}
            S(\vx_0,\vb)\ge &S^\prime(\vx_0,\vb) + (\tau^\prime)^{-1}\\
            = & \left(\log \frac{2L\cdot W_2^2(q_0, p_{k+1|k+\frac{1}{2},b}(\cdot|\vx_0, \vb))}{\alpha_* \epsilon^2} +  \log \log  \frac{(1+L^2)d}{4\alpha_* \epsilon^2 }\right)\cdot \frac{32\sigma^2 + 96Ld}{L\alpha_*\epsilon^2}\cdot \log \frac{(1+L^2)d}{4\alpha_* \epsilon^2}\\
            &  + \frac{4L\sigma^2+ 12L^2 d}{\alpha_* \epsilon^2}\cdot  \log \frac{(1+L^2)d}{4\alpha_* \epsilon^2 }.
        \end{aligned}
    \end{equation}
    It should be noted that the last equation of Eq~\ref{ineq:S_choice_alg2} only establishes when $\delta$ and $\eta$ are chosen as their upper bounds, and $b^\prime_{\mathrm{in}}=1$.

    Since the choice of $S(\vx_0,\vb)$ depend on the upper bound of $W_2^2(q_1, p_{k+1|k+\frac{1}{2},b}(\cdot|\vx_0, \vb))$, we start to bound it.
    Line 3 of Alg~\ref{alg:sgld_inner} has presented that $q_0$ is a Gaussian-type distribution with $\eta^{-1}$-strong convexity, then we have $q_0$ also satisfies $\eta^{-1}$-LSI due to Lemma~\ref{lem:strongly_lsi}, which implies
    \begin{equation*}
        W_2^2(q_0, p_{k+1|k+\frac{1}{2},b}(\cdot|\vx_0, \vb)) \le 2\eta \KL{q_0}{p_{k+1|k+\frac{1}{2},b}(\cdot|\vx_0, \vb)}\le \eta^2 \FI{q_0}{p_{k+1|k+\frac{1}{2},b}(\cdot|\vx_0, \vb)}.
    \end{equation*}
    Noted that the relative Fisher information satisfies
    \begin{equation*}
        \begin{aligned}
            &\FI{q_0}{p_{k+1|k+\frac{1}{2},b}(\cdot|\vx_0, \vb)} = \int q_0(\vz) \left\|\grad \log \frac{q_0(\vz)}{p_{k+1|k+\frac{1}{2},b}(\vz|\vx_0, \vb)}\right\|^2\der \vz\\
            & = \int q_0(\vz) \left\|\grad f_{\vb}(\vz)-\grad f_{\vb}(\vzero) + \grad f_{\vb}(\vzero)-\grad f(\vzero) + \grad f(\vzero)\right\|^2\der \vz \\
            &\le 3L^2 \E_{\rvz\sim q_0}[\|\rvz\|^2] + 3\left\|\grad f_{\vb}(\vzero)-\grad f(\vzero)\right\|^2 + 3\left\|\grad f(\vzero)\right\|^2\\
            & = 3L^2(\eta + \left\|\vx_0\right\|^2) + 3\left\|\grad f_{\vb}(\vzero)-\grad f(\vzero)\right\|^2 + 3\left\|\grad f(\vzero)\right\|^2.
        \end{aligned}
    \end{equation*}
    where the first inequality follows from~\ref{con_ass:lips_loss} with respect to $f_{\vb}$, and the last equation follows from the explicit form of the mean and variance of Gaussian-type $q_0$.
    Taking the expectation for both sides, we have
    \begin{equation}
        \label{ineq:exp_w2d_init}
        \begin{aligned}
            &\E_{\rvx_0, \rvb}\left[W_2^2(q_0, p_{k+1|k+\frac{1}{2},b}(\cdot|\rvx_0, \rvb))\right] \le 3\eta^2 \cdot \left(L^2\eta + L^2\E_{\rvx_0}\left[\left\|\rvx_0\right\|^2\right]+ \E_{\rvb}\left[\left\|\grad f_{\vb}(\vzero)-\grad f(\vzero)\right\|^2\right]+ \left\|\grad f(\vzero)\right\|^2\right)\\
            &\le \E_{\rvx_0}\left[\left\|\rvx_0\right\|^2\right]  + \frac{1}{2L} + \frac{\E_{\rvb}\left[\left\|\grad f_{\vb}(\vzero)-\grad f(\vzero)\right\|^2\right]}{L^2}+ \frac{\left\|\grad f(\vzero)\right\|^2}{L^2}\\
            &\le \E_{\rvx_0}\left[\left\|\rvx_0\right\|^2\right] + (2L^2)^{-1}\cdot \left(2\left\|\grad f(\vzero)\right\|^2 + L+ 2\sigma^2/|\rvb| \right)\le  \E_{\rvx_0}\left[\left\|\rvx_0\right\|^2\right] + \frac{2\left\|\grad f(\vzero)\right\|^2 + L+ 2\sigma^2}{2L^2}
        \end{aligned}
    \end{equation}
    where the second inequality follows from the choice of $\eta$, the third inequality follows from Lemma~\ref{lem:minibatch_var}, and the last inequality establishes since $|\rvb|\ge 1$.
    To solve this problem, we start with upper bounding the second moment, i.e., $M_k$ of $p_{k}$ for any $k\in[1,K]$.
    For calculation convenience, we suppose $L\ge 1$, $\delta<1$ without loss of generality and set 
    \begin{equation*}
        \begin{aligned}
        C_m \coloneqq 
        4\eta \delta + \frac{6\sigma^2}{b_o} + \left(\frac{6}{\eta^2}+4\right)M + \frac{6d}{\eta}\le 2+6\sigma^2+(24L^2+4)M+12Ld.
        \end{aligned}
    \end{equation*}
    In this condition, following from Lemma~\ref{lem:2ndmoment_bound}, we have 
    \begin{equation*}
        \begin{aligned}
            M_{k+1}& \le \frac{6}{\eta_k^2}\cdot M_k + 4\eta_k \delta_k + \frac{6\sigma^2}{|\rvb_k|}+ \left(\frac{6}{\eta_k^2}+4\right)M  + \frac{6d}{\eta_k} = 24L^2M_k +C_m,
        \end{aligned}
    \end{equation*}  
    which implies
    \begin{equation}
        \label{ineq:2nd_moment_K}
        \begin{aligned}
            M_k \le & \left(24L^2\right)^{k} M + C_m\cdot \left(1+24L^2 + \ldots + \left(24L^2\right)^{k-1}\right) \le \left(24L^2\right)^{k}\cdot \left(M + \frac{C_m}{24L^2 - 1}\right)\\
            \le & \left(24L^2\right)^K\cdot \left(M+ 2+6\sigma^2+(24L^2+4)M+12Ld\right).
        \end{aligned}
    \end{equation}
    Additionally, Lemma~\ref{lem:2ndmoment_bound} also demonstrates that
    \begin{equation*}
        M_{k+\frac{1}{2}} \le M_{k}+\eta_k d \le \left(24L^2\right)^K\cdot \left(M+ 2d+6\sigma^2+(24L^2+4)M+12Ld\right)
    \end{equation*}
    for all $k\in [0,K-1]$. 
    Plugging Eq~\ref{ineq:2nd_moment_K} into Eq~\ref{ineq:exp_w2d_init}, we have
    \begin{equation*}
        \begin{aligned}
            &\E_{\rvx_0, \rvb}\left[W_2^2(q_0, p_{k+1|k+\frac{1}{2},b}(\cdot|\rvx_0, \rvb))\right] \\
            &\le \left(24L^2\right)^{K-1}\cdot \left(M+ C_m + 2\left\|\grad f(\vzero)\right\|^2 + L+ 2\sigma^2\right)\\
            &\le \left(24L^2\right)^{K-1}\cdot 30L^2\cdot \left(M+\sigma^2 + d+ 1+ \left\|\grad f(\vzero)\right\|^2\right),
        \end{aligned}
    \end{equation*}
    which implies
    \begin{equation}
        \label{ineq:fin_w2d_init_bound}
        \begin{aligned}
            &\E_{\rvx_0, \rvb}\left[\log  (2L\cdot W_2^2(q_0, p_{k+1|k+\frac{1}{2},b}(\cdot|\vx_0, \vb)))\right] \le \log \left(\E\left[2L\cdot W_2^2(q_0, p_{k+1|k+\frac{1}{2},b}(\cdot|\vx_0, \vb)\right]\right)\\
            & \le \log \left[\left(24L^2\right)^{K}\cdot \left(M+\sigma^2 + d+ 1+ \left\|\grad f(\vzero)\right\|^2\right)\right] \\
            &= K\cdot \log (24L^2)+\log \left(30L^2\cdot\left(M+\sigma^2 + d+ 1+ \left\|\grad f(\vzero)\right\|^2\right)\right)\\
            & \le \frac{L}{\alpha_*} \log \frac{(1+L^2)d}{4\alpha_* \epsilon^2}\cdot \log (24L^2) + \log \left(30L^2\cdot\left(M+\sigma^2 + d+ 1+ \left\|\grad f(\vzero)\right\|^2\right)\right),
        \end{aligned}
    \end{equation}
    where the first inequality follows from Jensen's inequality and the last inequality follows from the parameters' choice shown in Eq~\ref{ineq:eta_K_choice_alg1}.
    By choosing $S(\vx_0,\vb)$ to its lower bound and taking the expectation for both sides of Eq~\ref{ineq:S_choice_alg2}, we have
    \begin{equation*}
        \small
        \begin{aligned}
            \E_{\rvx_0, \rvb}\left[S(\vx_0,\vb)\right] \le &\frac{32\sigma^2+96Ld}{L\alpha_*\epsilon^2}\cdot \log \frac{(1+L^2)d}{4\alpha_* \epsilon^2}\cdot \log\log \frac{(1+L^2)d}{4\alpha_* \epsilon^2}+ \frac{4L\sigma^2+ 12 L^2 d}{\alpha_* \epsilon^2}\cdot  \log \frac{(1+L^2)d}{4\alpha_* \epsilon^2 }\\
            & + \frac{32\sigma^2 + 96Ld}{L\alpha_*\epsilon^2}\cdot\log \frac{(1+L^2)d}{4\alpha_* \epsilon^2}\cdot 
            \E\left[\log \left(\frac{2L\cdot W_2^2(q_1, p_{k+1|k+\frac{1}{2},b}(\cdot|\vx_0, \vb))}{\alpha_* \epsilon^2}\right)\right]\\
            \le & \frac{32\sigma^2+96Ld}{L\alpha_*\epsilon^2}\cdot \log \frac{(1+L^2)d}{4\alpha_* \epsilon^2}\cdot \log\log \frac{(1+L^2)d}{4\alpha_* \epsilon^2}+ \frac{4L\sigma^2+ 12L^2 d}{\alpha_* \epsilon^2}\cdot  \log \frac{(1+L^2)d}{4\alpha_* \epsilon^2 }\\
            & + \frac{32\sigma^2 + 96Ld}{L\alpha_*\epsilon^2}\cdot\log \frac{(1+L^2)d}{4\alpha_* \epsilon^2}\cdot \left(\log \frac{30L^2\left(M+\sigma^2 + d+1+\left\|\grad f(\vzero)\right\|^2\right)}{\alpha_*\epsilon^2} + \frac{L}{\alpha_*} \log \frac{(1+L^2)d}{4\alpha_* \epsilon^2}\cdot \log (24L^2)\right)\\
            \le & \frac{4L\sigma^2+ 12L^2 d}{\alpha_* \epsilon^2}\cdot  \log \frac{(1+L^2)d}{4\alpha_* \epsilon^2 }  + \frac{32\sigma^2 + 96Ld}{L\alpha_*\epsilon^2}\cdot\log \frac{(1+L^2)d}{4\alpha_* \epsilon^2}\\
            &\cdot 2\cdot \frac{L}{\alpha_*}\cdot \log \frac{30L^2\left(M+\sigma^2 + d+1+\left\|\grad f(\vzero)\right\|^2\right)}{\alpha_*\epsilon^2}\cdot \log (24L^2)\\
            \le & \frac{34L^2(\sigma^2 + 3Ld)}{\alpha_*^2 \epsilon^2}\cdot \log(24L^2)\cdot \log \frac{(1+L^2)d}{4\alpha_* \epsilon^2}\cdot \log \frac{30L^2\left(M+\sigma^2 + d+1+\left\|\grad f(\vzero)\right\|^2\right)}{\alpha_*\epsilon^2},
        \end{aligned}
    \end{equation*}
    for all $\rvx_0\sim p_{k+1/2}$.
    Hence, the total gradient complexity will be 
    \begin{equation*}
        K\cdot \E_{\rvx_0, \rvb}\left[S(\vx_0,\vb)\right] = \tilde{O}(\kappa^3 \epsilon^{-2}\cdot \max\{\sigma^2, Ld\}),
    \end{equation*}
    and the proof is completed.
\end{proof}

\subsection{Warm-started MALA Inner Samplers}
\label{app_sec:innermala}

We define the Re\'nyi divergence between two distributions as
\begin{equation*}
    \mathcal{R}_r(p\|q) = \frac{1}{r-1}\log \int \left(\frac{p(\vx)}{q(\vx)} \right)^r\cdot q(\vx)\der \vx,
\end{equation*}
since it will be widely used in the following section.
Then, we provide a detailed theoretical analysis.

\begin{lemma}
    \label{lem:var_thm41_altschuler2023faster}
    Suppose~\ref{con_ass:lips_loss} holds and Alg~\ref{alg:ULD_inner} is implemented with following hyper-parameters' settings:
    \begin{equation*}
        \gamma = \sqrt{3/\eta},\quad \tau=\tilde{\Theta}\left(\frac{\delta \eta^{1/2}}{d^{1/2}r^{1/2}}\right),\quad \mathrm{and}\quad S= \tilde{\Theta}\left(\frac{d^{1/2}r^{1/2}}{\delta}\log\left(\|\vx_0\|^2+(\eta\|\grad f_{\vb}(\vzero)\|)^2\right)\right),
    \end{equation*}
    the underlying distribution $q_S$ of the output particle i.e., $\rvz_S$ will satisfy
    \begin{equation*}
        \mathcal{R}_r(q_S\|q_*)\le \delta^2,
    \end{equation*}
    where $\mathcal{R}_r$ denotes Re\'nyi divergence with order $r$.
\end{lemma}
\begin{proof}
    We suppose the $\mathsf{InnerULD}$ is implemented as Alg~\ref{alg:ULD_inner}.
    We denote the underlying distribution of $(\rvz_s,\rvv_s)$ as $q^\prime_s$ and its marginal distribution w.r.t. $\rvz_s$ is denoted as $q_s$.
    Since, we only consider Alg~\ref{alg:ULD_inner} rather than its outer loops, the target distribution of Alg~\ref{alg:ULD_inner} can be abbreviated as 
    \begin{equation*}
        q_*(\vz)\propto \exp(-g(\vz)),\quad  q_*^\prime(\vz,\vv) \propto \exp\left(-g(\vz)-\frac{\|\vv\|^2}{2}\right),\quad \mathrm{where}\quad g(\vz)\coloneqq -\log p_{k+1|k+\frac{1}{2},b}(\vz|\vx_0, \vb).
    \end{equation*}
    Combining Lemma~\ref{lem:sc_sm_inner_targets} and the choice of the step size, i.e., $\eta\le 1/2L$, we have
    \begin{equation*}
        (2\eta)^{-1}\cdot \mI \preceq \grad^2 g(\vz) = \grad^2 q_*(\vz) \preceq (3/2\eta)\cdot \mI.
    \end{equation*}
    By data-processing inequality, we have
    \begin{equation*}
        \mathcal{R}_r(q_S\|q_*)\le \mathcal{R}_r(q^\prime_S\|q^\prime_*).
    \end{equation*}
    By the weak triangle inequality of Re\'nyi divergence, i.e., Lemma 7 in~\cite{vempala2019rapid}, we have 
    \begin{equation*}
        \mathcal{R}_r(q_S^\prime\|q_*^\prime)\le \frac{r-1/2}{r-1} \cdot \mathcal{R}_{2r}(q_S^\prime\|\tilde{q}_*^\prime) + \mathcal{R}_{2r-1}(\tilde{q}_*^\prime\|q_*).
    \end{equation*}
    It can be noted that $\frac{r-1/2}{r-1}$ will be bounded by $2$ when $q\ge 3/2$ and $\tilde{q}_*^\prime$ denotes the underlying distribution of output particles if we initialize $q_0^\prime$ with $q_*^\prime$.
    Then, by combining Lemma~\ref{lem:thm4.4_altschuler2023faster}, Lemma~\ref{lem:rmk4.2_altschuler2023faster} and Lemma~\ref{lem:lem4.8_altschuler2023faster}, we conclude that 
    \begin{equation*}
        \mathcal{R}_r(q_S^\prime\|q_*^\prime)\le  \delta^2  
    \end{equation*}
    if ULD is run with friction parameter $\gamma$, step size $\tau$, and iteration complexity $N$ that satisfy:
    \begin{equation*}
        \gamma = \sqrt{3/\eta},\quad \tau\lesssim \frac{\delta\eta^{3/4}}{d^{1/2}r^{1/2}T^{1/2}},\quad \mathrm{and}\quad S\gtrsim \frac{\sqrt{\eta}}{\tau}\log \left(\left(d\eta+\|\vx_0-\vz_*\|^2\right)\cdot \frac{r \eta^{1/2}}{\delta^2\tau^3}\right).
    \end{equation*}
    By recalling that $T=N\tau$, solving for these choices of parameters, and omitting logarithmic factors, we conclude that it suffices to run ULD with the following choices of parameters:
    \begin{equation}
        \label{ineq:uld_hyper_choice}
        \gamma = \sqrt{3/\eta},\quad \tau=\tilde{\Theta}\left(\frac{\delta \eta^{1/2}}{d^{1/2}r^{1/2}}\right),\quad \mathrm{and}\quad S= \tilde{\Theta}\left(\frac{d^{1/2}r^{1/2}}{\delta}\log \|\vx_0-\vz_*\|^2\right)
    \end{equation}
    where $\vz_*$ is the minimizer of $g$.
    Besides, the minimizer of $g$ satisfies
    \begin{equation*}
        \begin{aligned}
            \grad g(\vz_*) = \grad f_{\vb}(\vz_*)+\eta^{-1}\cdot \left(\vz_* - \vx_0\right) = \vzero \quad \Leftrightarrow \quad \vx_0 = \eta \grad f_{\vb}(\vz_*) + \vz_*,
        \end{aligned}
    \end{equation*}
    which implies
    \begin{equation*}
        \begin{aligned}
            \|\vx_0\| = \left\|\eta \grad f_{\vb}(\vz_*)+\vz_*\right\|\ge \|\vz_*\| - \eta\|\grad f_{\vb}(\vz_*)\|\quad \Leftrightarrow \quad \|\vx_0\|+\eta \|\grad f_{\vb}(\vz_*)\|\ge \|\vz_*\|.
        \end{aligned}
    \end{equation*}
    In this condition, it has
    \begin{equation*}
        \begin{aligned}
            \|\vz_*\|\le \|\vx_0\| + \eta\|\grad f_{\vb}(\vz_*) - \grad f_{\vb}(\vzero)+\grad f_{\vb}(\vzero)\|\le \|\vx_0\|+L\eta\|\vz_*\|+\eta \|\grad f_{\vb}(\vzero)\|
        \end{aligned}
    \end{equation*}
    where the second inequality follows from~\ref{con_ass:lips_loss}.
    Since, we require $L\eta\le 1/2$, then the previous inequality is equivalent to 
    \begin{equation*}
        \|\vz_*\|\le 2\|\vx_0\| + 2\eta\|\grad f_{\vb}(\vzero)\|.
    \end{equation*}
    Plugging this results into Eq~\ref{ineq:uld_hyper_choice}, the hyper-parameter choice of Alg~\ref{alg:ULD_inner} can be concluded as
    \begin{equation*}
        \gamma = \sqrt{3/\eta},\quad \tau=\tilde{\Theta}\left(\frac{\delta \eta^{1/2}}{d^{1/2}r^{1/2}}\right),\quad \mathrm{and}\quad S= \tilde{\Theta}\left(\frac{d^{1/2}r^{1/2}}{\delta}\log\left(\|\vx_0\|^2+(\eta\|\grad f_{\vb}(\vzero)\|)^2\right)\right).
    \end{equation*}
\end{proof}

\begin{lemma}[Variant of Theorem 1 of~\cite{wu2022minimax}]
    \label{lem:var_thm1_wu2022minimax}
    Using the notations presented in Alg~\ref{alg:mala_inner},
    suppose~\ref{con_ass:lips_loss} holds and Alg~\ref{alg:mala_inner} is implemented when
    \begin{equation*}
        \tau = \Theta\left( \eta d^{-1/2}\log^{-2}\left(\max\left\{d, \frac{\chi^2(q_0\|q_*)}{\delta^2}\right\}\right)\right),\quad\mathrm{and}\quad S = \Theta\left(d^{1/2}\log^3 \left(\frac{\chi^2(q_0\|q_*)}{\delta^2}\right)\right).
    \end{equation*}
    Then, underlying distribution $q_S$ of the output particle i.e., $\rvz_S$ will satisfy
    \begin{equation*}
        \TVD{q_S}{q_*} \le \delta.
    \end{equation*}
\end{lemma}
\begin{proof}
    We suppose the $\mathsf{InnerMALA}$ is implemented as Alg~\ref{alg:mala_inner}.
    We denote the underlying distribution of $(\rvz_s,\rvv_s)$ as $q^\prime_s$ and its marginal distribution w.r.t. $\rvz_s$ is denoted as $q_s$.
    Since, we only consider Alg~\ref{alg:mala_inner} rather than its outer loops, the target distribution of Alg~\ref{alg:mala_inner} can be abbreviated as 
    \begin{equation*}
        q_*(\vz)\propto \exp(-g(\vz)),\quad  q_*^\prime(\vz,\vv) \propto \exp\left(-g(\vz)-\frac{\|\vv\|^2}{2}\right),\quad \mathrm{where}\quad g(\vz)\coloneqq -\log p_{k+1|k+\frac{1}{2},b}(\vz|\vx_0, \vb).
    \end{equation*}

    Theorem 1 of ~\cite{wu2022minimax} upper bound the total variation distance between the underlying distribution of output particles and the target distribution as 
    \begin{equation*}
        \TVD{q_S}{q_*} \le H_s  + \frac{H_s}{s}\cdot \exp\left(-\frac{S \Phi_s}{2}\right)
    \end{equation*}
    where $H_s$ is defined as
    \begin{equation*}
        H_s\coloneqq \sup\left\{\left|q_0(A) - q_*(A)\right|: q_*(A)\le s\right\}
    \end{equation*}
    and $\Phi_s$ denotes the $s$-conductance.
    The final step size and gradient complexity will depend on the warm-start $M$ defining as $H_s\le Ms$.
    Since, we use $\chi^2$ distance to define the warm-start in our analysis. 
    We have additionally the following inequality.
    \begin{equation*}
        \begin{aligned}
            \left|q_0(A) - q_*(A)\right| = \left|\int \vone_A \left(\frac{\der q_0}{\der q_*}-1\right)\der q_*\right|\le \sqrt{\int \vone_A\der\pi\cdot \int \left(\frac{\der q_0}{\der q_*}-1\right)^2 \der q_*}\le \sqrt{q_*(A)\chi^2(q_0\|q_*)},
        \end{aligned}
    \end{equation*} 
    which means $H_s\le \sqrt{s\chi^2(q_0\|q_*)}$.
    In this condition, we have
    \begin{equation*}
        \TVD{q_S}{q_*} \le \sqrt{s\chi^2(q_0\|q_*)} +  \sqrt{\frac{\chi^2(q_0\|q_*)}{s}} \cdot \exp\left(-\frac{S \Phi_s}{2}\right)
    \end{equation*}
    By requiring 
    \begin{equation*}
        s = \frac{\delta^2}{4\chi^2(q_0\|p_*)}\quad \mathrm{and}\quad S=\frac{2}{\Phi_s}\log\left(\frac{8\chi^2(q_0\|p_*)}{\delta^2}\right),
    \end{equation*}
    we can achieve $\TVD{q_S}{p_*}\le \epsilon$.
    Besides, we can obtain the $M$ by
    \begin{equation}
        \label{ineq:choice_M}
        \begin{aligned}
            M\ge \frac{H_s}{s} \quad\Leftarrow\quad M\ge \sqrt{\frac{\chi^2(q_0\|q_*)}{s}} = \frac{2\chi^2(q_0\|q_*)}{\delta}.
        \end{aligned}
    \end{equation}
    Since the target distribution $q_*$ is $(1/2\eta)$-strongly convex and $(3/2\eta)$-smooth when $\eta\le 1/(2L)$ due to Lemma~\ref{lem:sc_sm_inner_targets},
    plugging the choice of $M$ shown in Eq~\ref{ineq:choice_M} into Theorem 1 of ~\cite{wu2022minimax}, we know the step size should be
    \begin{equation*}
        \begin{aligned}
            \tau = \Theta\left( \eta d^{-1/2}\log^{-2}\left(\max\left\{d, \frac{\chi^2(q_0\|q_*)}{\delta^2}\right\}\right)\right)
        \end{aligned}
    \end{equation*}
    and the gradient complexity will be 
    \begin{equation*}
        S = \Theta\left(d^{1/2}\log^3 \left(\frac{\chi^2(q_0\|q_*)}{\delta^2}\right)\right).
    \end{equation*}
    Hence, the proof is completed.
\end{proof}

\begin{corollary}
    \label{cor:inner_mala_conv}
    Suppose~\ref{con_ass:lips_loss} holds, we implement Alg~\ref{alg:ULD_inner} with 
    \begin{equation*}
        \gamma = \sqrt{3/\eta},\quad \tau=\tilde{\Theta}\left(\frac{\eta^{1/2}}{d^{1/2}}\right),\quad \mathrm{and}\quad S= \tilde{\Theta}\left(d^{1/2}\log\left(\|\vx_0\|^2+(\eta\|\grad f_{\vb}(\vzero)\|)^2\right)\right),
    \end{equation*}
    and implement Alg~\ref{alg:mala_inner} with 
    \begin{equation*}
        \tau = \Theta\left( \eta d^{-1/2}\log^{-2}\left(\max\left\{d, \delta^{-1}\right\}\right)\right),\quad\mathrm{and}\quad S = \Theta\left(d^{1/2}\log^3 \left(1/\delta\right)\right).
    \end{equation*}
    The underlying distribution $q_S$ of the output particle of Alg~\ref{alg:mala_inner} will have
    \begin{equation*}
        \KL{q_S}{q_*}\le \delta,
    \end{equation*}
    and the total gradient complexity will be 
    \begin{equation*}
        \tilde{\Theta}\left(|\vb|d^{1/2}\left(\log\left(\|\vx_0\|^2+(\eta\|\grad f_{\vb}(\vzero)\|)^2\right)+\log^3(1/\delta)\right)\right).
    \end{equation*}
\end{corollary}
\begin{proof}
    Using the notations in Alg~\ref{alg:mala_inner},
    by Lemma~\ref{lem:var_thm41_altschuler2023faster}, Alg~\ref{alg:ULD_inner} can outputs a distribution $q_0$ satisfying 
    \begin{equation*}
        \mathcal{R}_3(q_0\|q_*)\le \log 2,
    \end{equation*}
    which implies
    \begin{equation*}
        \chi^2(q_0\|q_*)\le \exp\left(\mathcal{R}_2(q_0\|q_*)\right)-1 \le \exp\left(\mathcal{R}_3(q_0\|q_*)\right)-1 \le 1.
    \end{equation*}
    It should be noted that the second inequality follows from the monotonicity of Re\'nyi divergence.
    In this condition, the gradient complexity of Alg~\ref{alg:ULD_inner} should be 
    \begin{equation*}
        |\vb|\times S^\prime= \tilde{\Theta}\left(|\vb|d^{1/2}\log\left(\|\vx_0\|^2+(\eta\|\grad f_{\vb}(\vzero)\|)^2\right)\right),
    \end{equation*}
    where $S^\prime$ denotes the iteration number of Alg~\ref{alg:ULD_inner}, i.e., Line 2 of Alg~\ref{alg:mala_inner}.
    With the warm start in $\chi^2$ divergence, we invoke Lemma~\ref{lem:var_thm1_wu2022minimax} and achieve 
    \begin{equation*}
        \TVD{q_S}{q_*} \le \delta^2/5.
    \end{equation*}
    with the following gradient complexity
    \begin{equation*}
        |\vb|\times S= \Theta\left(|\vb|d^{1/2}\log^3 \left(1/\delta\right)\right).
    \end{equation*}
    Then, we start upper bound the KL divergence between $q_S$ and $q_*$ and have 
    \begin{equation*}
        \begin{aligned}
            \KL{q_S}{q_*}\le & \chi^2(q_S\|q_*) = \int \left(\frac{q_S(\vz)}{q_*(\vz)}-1\right)^2 q_*(\vz)\der \vz\le  \sqrt{\int \left|\frac{q_S(\vz)}{q_*(\vz)}-1\right| q_*(\vz)\der \vz \cdot \int \left|\frac{q_S(\vz)}{q_*(\vz)}-1\right|^3 q_*(\vz)\der \vz }\\
            \le & \sqrt{\TVD{q_S}{q_*}\cdot \left(\int \left|\frac{q_S(\vz)}{q_*(\vz)}\right|^3\der \vz +1\right)} = \sqrt{\TVD{q_S}{q_*}\cdot \left(\exp\left(2\mathcal{R}_3(q_S\|q_*)\right)+1\right)}\\
            \le & \sqrt{\TVD{q_S}{q_*}\cdot \left(\exp\left(2\mathcal{R}_3(q_0\|q_*)\right)+1\right)} \le \delta,
        \end{aligned}
    \end{equation*}
    where the second inequality follows from Cauchy–Schwarz inequality, the second equation follows from the definition of Re\'nyi divergence, and the last inequality follows from data-processing inequality.
    Therefore, to ensure the convergence of KL divergence, i.e., 
    \begin{equation*}
        \KL{q_S}{q_*}\le \delta,
    \end{equation*}
    the total complexity of this warm start MALA will be
    \begin{equation*}
        \tilde{\Theta}\left(|\vb|d^{1/2}\left(\log\left(\|\vx_0\|^2+(\eta\|\grad f_{\vb}(\vzero)\|)^2\right)+\log^3(1/\delta)\right)\right).
    \end{equation*}
    Hence, the proof is completed.
\end{proof}

\begin{theorem}
    \label{thm:conv_gra_comp_innerMALA}
    Suppose \ref{con_ass:lips_loss}-\ref{con_ass:var_bound} hold. 
    With the following parameter settings 
    \begin{equation*}
        \begin{aligned}
            &\eta_k = \frac{1}{2L},\quad K = \frac{L}{\alpha_*}\cdot \log \frac{(1+L^2)d}{4\alpha_* \epsilon^2},\quad \delta_k = \frac{2\epsilon^2 \alpha_*}{L}\cdot \left(\log \frac{(1+L^2)d}{4\alpha_* \epsilon^2}\right)^{-1}\\
            &b_o = \min\left\{\frac{\sigma^2}{4 \alpha_* \epsilon^2}\cdot \log \frac{(1+L^2)d}{4\alpha_* \epsilon^2},n\right\},
        \end{aligned}
    \end{equation*}
    for Alg~\ref{alg:sps}, if we choose Alg~\ref{alg:mala_inner} as the inner sampler shown in Line 5 of Alg~\ref{alg:sps}, set 
    \begin{equation*}
        \gamma = \sqrt{6L},\quad \tau=\tilde{\Theta}\left(\frac{1}{\sqrt{2Ld}}\right),\quad \mathrm{and}\quad S= \tilde{\Theta}\left(d^{1/2}\log\left(\|\vx_0\|^2+\frac{\|\grad f_{\vb}(\vzero)\|^2}{2L^2}\right)\right).
    \end{equation*}
    for Alg~\ref{alg:ULD_inner}, and
     \begin{equation*}
        \begin{aligned}
        &\tau = \Theta \left(\frac{1}{2L\sqrt{d}}\cdot \log^{-2}\left(\max\left\{d, \frac{L}{2\alpha_*\epsilon^2}\log \frac{(1+L^2)d}{4\alpha\epsilon^2}\right\}\right)\right),\\
        &\mathrm{and}\quad S=\Theta\left(d^{1/2}\log^3 \left(\frac{L}{2\alpha_*\epsilon^2}\log \frac{(1+L^2)d}{4\alpha\epsilon^2}\right)\right).
        \end{aligned}
    \end{equation*}
    for Alg~\ref{alg:mala_inner}, then the underlying distribution of returned particles $p_K$ in Alg~\ref{alg:sps} satisfies $\TVD{p_{K+1}}{p_*}<3\epsilon$. 
    In this condition, the expected gradient complexity will be $\tilde{\Theta}\left(\kappa^3 d^{1/2}\sigma^2\epsilon^{-2}\right)$.
\end{theorem}
\begin{proof}
    We provide this proof with a similar proof roadmap shown in Theorem~\ref{thm:conv_gra_comp_innerSGLD}.
    Specifically, we show the detailed implementation of Alg~\ref{alg:sps} with Alg~\ref{alg:sgld_inner} in the following.
    \begin{itemize}
        \item For all $k\in\{0,1,\ldots, K-1\}$, the mini-batch $\rvb_k$ in Alg~\ref{alg:sps} Line 2 has a uniformed norm which is denoted as $|\rvb_k| = b_o$.
        \item For all $k\in\{0,1,\ldots, K-1\}$, the conditional probability densities $p_{k+1|k+1/2,b}(\cdot|\rvx_{k+1/2}\vb_k)$ in Alg~\ref{alg:sps} Line 4 formulated as Eq~\ref{def:transition_kernel_stage2} share the same $L$-2 regularized coefficients, i.e., $\eta^{-1}$.
        \item For all $k\in\{0,1,\ldots, K-1\}$, the inner sampler shown in Alg~\ref{alg:sps} Line 5 is chosen as Alg~\ref{alg:mala_inner}.
    \end{itemize}
    By requiring 
    \begin{equation*}
        \eta \le \frac{1}{2L}\quad \mathrm{and}\quad K\ge (2\alpha_*\eta)^{-1 }\cdot \log \frac{(1+L^2)d}{4\alpha_* \epsilon^2} = \frac{L}{\alpha_*}\cdot \log \frac{(1+L^2)d}{4\alpha_* \epsilon^2},
    \end{equation*}
    we have
    \begin{equation*}
        \sqrt{\frac{(1+L^2)d}{4\alpha_*}} \cdot \left(1+\alpha_* \eta\right)^{-K} \le \sqrt{\frac{(1+L^2)d}{4\alpha_*}} \cdot \exp(-\alpha_* K\eta) \le \epsilon.
    \end{equation*}
    Besides by requiring
    \begin{equation*}
        b_o\ge \min\left\{\frac{K\eta \sigma^2 }{2\epsilon^2}, n\right\} = \min\left\{\frac{\sigma^2}{4 \alpha_* \epsilon^2}\cdot \log \frac{(1+L^2)d}{4\alpha_* \epsilon^2},n\right\},
    \end{equation*}
    we have $\sigma \sqrt{K\eta/(2b_o)}\le \epsilon$.
    Additionally, by requiring,
    \begin{equation*}
        \delta \le \frac{2\epsilon^2}{K} = \frac{2\epsilon^2 \alpha_*}{L}\cdot \left(\log \frac{(1+L^2)d}{4\alpha_* \epsilon^2}\right)^{-1}.
    \end{equation*}
    With these conditions, we have
    \begin{equation*}
        \TVD{p_{K}}{p_*} \le \sqrt{\frac{1}{2}\sum_{i=0}^{K-1} \delta_i} +  \sigma \sqrt{\frac{K\eta}{2b_o}} + \sqrt{\frac{(1+L^2)d}{4\alpha_*}} \cdot \left(1+\alpha_* \eta\right)^{-K}\le 3\epsilon
    \end{equation*}
    which follows from Theorem~\ref{thm:comp_conv_sps}.
    \paragraph{Errors control of inner loops.}
    To determine the hyper-parameter settings of Alg~\ref{alg:ULD_inner} and Alg~\ref{alg:mala_inner}, we can plug the choice of outer loops step size $\eta$ and inner loops error tolerance $\delta$, i.e., 
    \begin{equation*}
        \eta = \frac{1}{2L}\quad \mathrm{and}\quad \delta = \frac{2\epsilon^2 \alpha_*}{L}\cdot \left(\log \frac{(1+L^2)d}{4\alpha_* \epsilon^2}\right)^{-1}
    \end{equation*}
    into Corollary~\ref{cor:inner_mala_conv}.
    In this condition,  for Alg~\ref{alg:ULD_inner}, we set
    \begin{equation*}
        \gamma = \sqrt{6L},\quad \tau=\tilde{\Theta}\left(\frac{1}{\sqrt{2Ld}}\right),\quad \mathrm{and}\quad S= \tilde{\Theta}\left(d^{1/2}\log\left(\|\vx_0\|^2+\frac{\|\grad f_{\vb}(\vzero)\|^2}{2L^2}\right)\right).
    \end{equation*}
    For Alg~\ref{alg:mala_inner}, we set
    \begin{equation*}
        \begin{aligned}
        &\tau = \Theta \left(\frac{1}{2L\sqrt{d}}\cdot \log^{-2}\left(\max\left\{d, \frac{L}{2\alpha_*\epsilon^2}\log \frac{(1+L^2)d}{4\alpha\epsilon^2}\right\}\right)\right),\\
        &\mathrm{and}\quad S=\Theta\left(d^{1/2}\log^3 \left(\frac{L}{2\alpha_*\epsilon^2}\log \frac{(1+L^2)d}{4\alpha\epsilon^2}\right)\right).
        \end{aligned}
    \end{equation*}
    Then, the underlying distribution $q_S$ of the output particle of Alg~\ref{alg:mala_inner} will satisfy
    \begin{equation*}
        \KL{q_S}{q_*}\le \frac{2\epsilon^2 \alpha_*}{L}\cdot \left(\log \frac{(1+L^2)d}{4\alpha_* \epsilon^2}\right)^{-1} = \delta,
    \end{equation*}
    and the total gradient complexity will be 
    \begin{equation*}
        \tilde{\Theta}\left(b_od^{1/2}\left(\log\left(\|\vx_0\|^2+(\eta\|\grad f_{\vb}(\vzero)\|)^2\right)+\log^3(1/\delta)\right)\right).
    \end{equation*}
    Since $\log (1/\delta)$ will only provide additional log terms which will be omitted in $\tilde{\Theta}$, we only consider the following inequality, i.e.,
    \begin{equation}
        \label{ineq:inner_mala_each_comp}
        \begin{aligned}
            &\E_{\rvx_0, \rvb}\left[b_od^{1/2}\log\left(\|\vx_0\|^2+(\eta\|\grad f_{\vb}(\vzero)\|)^2\right)\right] \le b_o d^{1/2}\log\left(\E\left[\left\|\rvx_0\right\|^2\right] + \eta^2\E\left[\left\|\grad f_{\rvb}(\vzero)\right\|^2\right]\right)\\
            & \le b_o d^{1/2}\log \left(\E\left[\left\|\rvx_0\right\|^2\right]+ 2\eta^2\left\|\grad f(\vzero)\right\|^2 + 2\eta^2 \E\left[\left\|\grad f_{\rvb}(\vzero)-\grad f(\vzero)\right\|^2\right]\right)\\
            & \le \frac{\sigma^2d^{1/2}}{4 \alpha_* \epsilon^2}\cdot \log \frac{(1+L^2)d}{4\alpha_*\epsilon^2}\cdot \log \left(\E\left[\left\|\rvx_0\right\|^2\right]+ \frac{\left\|\grad f(\vzero)\right\|^2}{2L^2} + \frac{\sigma^2}{2L^2}\right)
        \end{aligned}
    \end{equation}
    the first inequality follows from Jensen's inequality, the second follows from triangle inequality, and the last follows from~\ref{con_ass:var_bound}.
    Here, we should note that the underlying distribution of random variable $\rvx_0$ is $p_{k+1/2}$.
    Hence, the second moment bound, i.e., $M_{k+1/2}$ of $p_{k+1/2}$ for any $k\in[0,K-1]$ is required.

    To solve this problem, we start with upper bounding the second moment, i.e., $M_k$ of $p_{k}$ for any $k\in[1,K]$.
    For calculation convenience, we suppose $L\ge 1$, $\delta<1$ without loss of generality and set 
    \begin{equation*}
        \begin{aligned}
        C_m \coloneqq 
        4\eta \delta + \frac{6\sigma^2}{b_o} + \left(\frac{6}{\eta^2}+4\right)M + \frac{6d}{\eta}\le 2+6\sigma^2+(24L^2+4)M+12Ld.
        \end{aligned}
    \end{equation*}
    In this condition, following from Lemma~\ref{lem:2ndmoment_bound}, we have 
    \begin{equation*}
        \begin{aligned}
            M_{k+1}& \le \frac{6}{\eta_k^2}\cdot M_k + 4\eta_k \delta_k + \frac{6\sigma^2}{|\rvb_k|}+ \left(\frac{6}{\eta_k^2}+4\right)M  + \frac{6d}{\eta_k} = 24L^2M_k +C_m,
        \end{aligned}
    \end{equation*}  
    which implies
    \begin{equation*}
        \begin{aligned}
            M_k \le & \left(24L^2\right)^{k} M + C_m\cdot \left(1+24L^2 + \ldots + \left(24L^2\right)^{k-1}\right) \le \left(24L^2\right)^{k}\cdot \left(M + \frac{C_m}{24L^2 - 1}\right)\\
            \le & \left(24L^2\right)^K\cdot \left(M+ 2+6\sigma^2+(24L^2+4)M+12Ld\right).
        \end{aligned}
    \end{equation*}
    Additionally, Lemma~\ref{lem:2ndmoment_bound} also demonstrates that
    \begin{equation*}
        M_{k+\frac{1}{2}} \le M_{k}+\eta_k d \le \left(24L^2\right)^K\cdot \left(M+ 2d+6\sigma^2+(24L^2+4)M+12Ld\right)
    \end{equation*}
    for all $k\in [0,K-1]$. 
    Plugging the following inequality, i.e.,
    \begin{equation*}
        \begin{aligned}
            \frac{\sigma^2d^{1/2}}{4 \alpha_* \epsilon^2}\cdot \log\E\left[\left\|\rvx_0\right\|^2\right] \le \frac{Ld^{1/2}\sigma^2}{4\alpha_*^2 \epsilon^2}\log \frac{(1+L^2)d}{4\alpha_*\epsilon^2}\log24L^2\log \left(M+ 2d+6\sigma^2+(24L^2+4)M+12Ld\right)
        \end{aligned}
    \end{equation*}
    into the RHS of Eq~\ref{ineq:inner_mala_each_comp} and omitting trivial log terms, we know the gradient complexity for each $k$ will be $\tilde{\Theta}\left(\kappa^2 d^{1/2}\sigma^2\epsilon^{-2}\right)$.
    After multiplying the total iteration number of Alg~\ref{alg:sps}, i.e., $K$, the final gradient complexity will be $\tilde{\Theta}\left(\kappa^3 d^{1/2}\sigma^2\epsilon^{-2}\right)$.
    Hence, the proof is completed.
\end{proof}

\section{Lemmas for Errors from Initialization of Inner Samplers}

\begin{proof}[Proof of Lemma~\ref{lem:2ndmoment_bound}]
    We first suppose the second moment of $\hat{p}_{k}$ is upper bounded and satisfies $\mathbb{E}_{\hat{p}_k}[\|\rvx\|^2]\le m_k$.

    According to Alg~\ref{alg:sps} Line 3, we have the closed form of the random variable $\hat{\rvx}_{k+1/2}$ is
    \begin{equation*}
        \hat{\rvx}_{k+\frac{1}{2}} = \hat{\rvx}_k + \sqrt{\eta_k}\xi,\quad \mathrm{where}\quad \xi\sim \mathcal{N}(\vzero, \mI).
    \end{equation*}
    Noted that $\xi$ is independent with $\hat{\rvx}_k$, hence, we have
    \begin{equation}
        \label{ineq:2ndmoment_ite_1over2}
        M_{k+\frac{1}{2}}\coloneqq \mathbb{E}\left[\left\|\hat{\rvx}_{k+\frac{1}{2}}\right\|^2\right] = \mathbb{E}\left[\left\|\hat{\rvx}_{k}\right\|^2\right]+ \eta_k\cdot d\le M_k+\eta_k\cdot d.
    \end{equation}
    Then, considering the second moment of $\rvx_{k+1}$, we have
    \begin{equation}
        \label{eq:2ndmoment_ite}
        \begin{aligned}
            \E\left[\left\|\hat{\rvx}_{k+1}\right\|^2\right]  = &\int \hat{p}_{k+1}(\vx)\cdot \left\|\vx\right\|^2 \der\vx \\
            = &\int \left(\int \hat{p}_{k+\frac{1}{2}}(\vy)\cdot \sum_{\vb \in \{1,2,\ldots, n \}}\hat{p}_{k+1|k+\frac{1}{2},b}(\vx|\vy,\vb)\cdot p_b(\vb)\right)\cdot \left\|\vx\right\|^2 \der\vx\\
            = & \sum_{\vb \in \{1,2,\ldots, n \}}\left(p_b(\vb)\cdot \int \hat{p}_{k+\frac{1}{2}}(\vy)\cdot \left(\int \hat{p}_{k+1|k+\frac{1}{2},b}(\vx|\vy,\vb) \cdot \left\|\vx\right\|^2 \der\vx\right) \der \vy\right)
        \end{aligned}
    \end{equation}
    Then, we focus on the innermost integration, suppose $\hat{\gamma}_{\vy}(\cdot,\cdot)$ as the optimal coupling between $\hat{p}_{k+1|k+\frac{1}{2},b}(\cdot|\vy)$ and ${p}_{k+1|k+\frac{1}{2},b}(\cdot | \vy)$. 
    Then, we have
    \begin{equation}
        \label{ineq:pati_to_closed_2ndmoment}
        \begin{aligned}
            &\int \hat{p}_{k+1|k+\frac{1}{2},b}(\vx|\vy)\left\|\vx\right\|^2 \der \vx - 2\int p_{k+1|k+\frac{1}{2},b}(\vx|\vy)\left\|\vx\right\|^2 \der\vx \\
            &\le \int \hat{\gamma}_{\vy}(\hat{\vx}, \vx) \left(\left\|\hat{\vx}\right\|^2 - 2\left\|\vx\right\|^2 \right) \der (\hat{\vx},\vx) \le \int \hat{\gamma}_{\vy}(\hat{\vx}, \vx) \left\|\hat{\vx}-\vx\right\|^2 \der (\hat{\vx},\vx) = W_2^2\left(\hat{p}_{k+1|k+\frac{1}{2},b}, {p}_{k+1|k+\frac{1}{2},b}\right).
        \end{aligned}
    \end{equation}
    Since ${p}_{k+1|k+\frac{1}{2},b}$ is strongly log-concave, i.e.,
    \begin{equation*}
        -\grad^2_{\vx^\prime} p_{k+1|k+\frac{1}{2},b}(\vx^\prime|\vx,\vb) = \grad^2 f_{\vb}(\vx^\prime) + \eta^{-1}\mI\succeq \left(-L + \eta_k^{-1}\right)\mI \succeq (2\eta_k)^{-1}\cdot \mI,
    \end{equation*}
    the distribution ${p}_{k+1|k+\frac{1}{2},b}$ also satisfies $(2\eta_k)^{-1}$ log-Sobolev inequality due to Lemma~\ref{lem:strongly_lsi}.
    By Talagrand's inequality, we have
    \begin{equation}
        \label{ineq:talagrand_closed_inner}
        W_2^2\left(\hat{p}_{k+1|k+\frac{1}{2},b}, {p}_{k+1|k+\frac{1}{2},b}\right) \le 4\eta_k \KL{\hat{p}_{k+1|k+\frac{1}{2},b}}{{p}_{k+1|k+\frac{1}{2},b}}\le 4\eta_k \delta_k.
    \end{equation}
    Plugging Eq~\ref{ineq:pati_to_closed_2ndmoment} and Eq~\ref{ineq:talagrand_closed_inner} into Eq~\ref{eq:2ndmoment_ite}, we have
    \begin{equation}
        \label{eq:2ndmoment_ite_mid}
        \E\left[\left\|\hat{\rvx}_{k+1}\right\|^2\right] \le \sum_{\vb \in \{1,2,\ldots, n \}}\left(p_b(\vb)\cdot \int \hat{p}_{k+\frac{1}{2}}(\vy)\cdot \left(4\eta_k\delta_k + 2\int p_{k+1|k+\frac{1}{2},b}(\vx|\vy)\left\|\vx\right\|^2 \der\vx \right) \der \vy\right).
    \end{equation}
    To upper bound the innermost integration, we suppose the optimal coupling between $p_*$ and $p_{k+1|k+\frac{1}{2},b}(\cdot |\vy)$ is $\gamma_{\vy}(\cdot,\cdot)$. Then it has
    \begin{equation}
        \label{ineq:closed_to_target_2ndmoment}
        \begin{aligned}
            & \int p_{k+1|k+\frac{1}{2},b}(\vx|\vy)\left\|\vx\right\|^2 \der\vx - 2\int p_*(\vx)\left\|\vx\right\|^2 \der \vx\\
            & \le \int \gamma_{\vy}(\vx^\prime, \vx)\left(\left\|\vx^\prime\right\|^2 - 2\left\|\vx\right\|^2\right)\der (\vx^\prime, \vx) \le \int \gamma_{\vy}(\vx^\prime, \vx)\left\|\vx^\prime-\vx\right\|^2\der (\vx^\prime, \vx) = W_2^2(p_*, p_{k+1|k+\frac{1}{2},b})
        \end{aligned}
    \end{equation}
    Since ${p}_{k+1|k+\frac{1}{2},b}$ satisfies LSI with constant $(2\eta_k)^{-1}$. 
    By Talagrand's inequality and LSI, we have
    \begin{equation*}
        \begin{aligned}
            & W_2^2(p_*, p_{k+1|k+\frac{1}{2},b}) \le  4\eta_k \KL{p_*}{p_{k+1|k+\frac{1}{2},b}} \\
            & \le 4\eta^2_k \int p_*(\vx)\cdot \left\|\grad \log \frac{p_*(\vx)}{p_{k+1|k+\frac{1}{2},b}(\vx|\vy,\vb)}\right\|^2 \der \vx = 4\eta_k^2 \int p_*(\vx)\cdot \left\|\grad f_{\vb}(\vx) -\grad f_(\vx) + \frac{\vx-\vy}{\eta_k} \right\|^2 \der \vx\\
            & \le 12\eta_k^2 \cdot \left[\int p_*(\vx) \left\|\grad f_{\vb}(\vx)-\grad f(\vx)\right\|^2 \der \vx + \eta_k^{-2}\int p_*(\vx) \left\|\vx\right\|^2 \der \vx + \eta_k^{-2}\left\|\vy\right\|^2\right].
        \end{aligned}
    \end{equation*}
    Combining this inequality with Eq~\ref{ineq:closed_to_target_2ndmoment}, we have
    \begin{equation*}
        \begin{aligned}
            &\int p_{k+1|k+\frac{1}{2},b}(\vx|\vy)\left\|\vx\right\|^2 \der\vx \le  12\eta_k^2 \int p_*(\vx) \left\|\grad f_{\vb}(\vx)-\grad f(\vx)\right\|^2 \der \vx + 12M + 12\left\|\vy\right\|^2 + 2M.
        \end{aligned}
    \end{equation*}
    Plugging this inequality into Eq~\ref{eq:2ndmoment_ite_mid}, we have
    \begin{equation}
        \label{eq:2ndmoment_ite_mid2}
        \begin{aligned}
            \E\left[\left\|\hat{\rvx}_{k+1}\right\|^2\right] & \le 4\eta_k \delta_k + \sum_{\vb\subseteq \{1,2,\ldots, n\}} 24\eta_k^2 \cdot p_b(\vb) \int \hat{p}_{k+\frac{1}{2}}(\vy) \cdot \left(\int p_*(\vx)\left\|\grad f_{\vb}(\vx)-\grad f(\vx)\right\|^2\der\vx\right) \der\vy\\
            & \quad + 28M + \sum_{b\subseteq \{1,2,\ldots, n\}} 24\cdot p_b(\vb) \int \hat{p}_{k+\frac{1}{2}}(\vy)\left\|\vy\right\|^2 \der \vy.
        \end{aligned}
    \end{equation}
    According to~\ref{con_ass:var_bound}, suppose we sample $\rvb$ uniformly from $\{1,2,\ldots,n\}$, then for any $\vx\in\R^d$ we have
    \begin{equation*}
        \begin{aligned}
            \E_{\rvb}\left[\left\|\frac{1}{|\rvb|}\sum_{i=1}^{|\rvb|}  \left(\grad f(\vx) - \grad f_{\rvb_i}(\vx)\right)\right\|^2\right] = &\frac{1}{|\rvb|^2}\sum_{i=1}^{|\rvb|}\sum_{j=1}^{|\rvb|}\E\left[(\grad f_{\rvb_i}(\vx)-\grad f(\vx))^\top (\grad f_{\rvb_j}(\vx)-\grad f(\vx))\right]\\
            = & \frac{1}{|\rvb|^2}\sum_{i=1}^{|\rvb|}\E\left[\left\|\grad f_{\rvb_i}(\vx)-\grad f(\vx)\right\| ^2\right] = \frac{\sigma^2}{|\rvb|}.
        \end{aligned}
    \end{equation*}
    Plugging this equation into the second term of RHS of Eq~\ref{eq:2ndmoment_ite_mid}, we have
    \begin{equation*}
        \begin{aligned}
            & \sum_{\vb\subseteq \{1,2,\ldots, n\}} p_b(\vb) \int \hat{p}_{k+\frac{1}{2}}(\vy) \cdot \left(\int p_*(\vx)\left\|\grad f_{\vb}(\vx)-\grad f(\vx)\right\|^2\der\vx\right) \der\vy\\
            & = \int \hat{p}_{k+\frac{1}{2}}(\vy) \int p_*(\vx) \E_{\rvb}\left[\left\|\grad f_{\rvb}(\vx)-\grad f(\vx)\right\|^2\right] \der\vx\der\vy  = \frac{\sigma^2}{|\rvb|}.
        \end{aligned}
    \end{equation*}
    Besides, for the last term of RHS of Eq~\ref{eq:2ndmoment_ite_mid}, we have
    \begin{equation*}
        \sum_{b\subseteq \{1,2,\ldots, n\}} p_b(\vb) \int \hat{p}_{k+\frac{1}{2}}(\vy)\left\|\vy\right\|^2 \der \vy  = M_{k+\frac{1}{2}}.
    \end{equation*}

    With these conditions, Eq~\ref{eq:2ndmoment_ite_mid2} can be reformulated as
    \begin{equation}
        \label{eq:2ndmoment_ite_fin}
        \begin{aligned}
            M_{k+1}\coloneqq \E\left[\left\|\hat{\rvx}_{k+1}\right\|^2\right] & \le 4\eta_k \delta_k + \frac{24\eta_k^2\sigma^2}{|\rvb|}+ 28M + 24 M_{k+\frac{1}{2}}\\
            & \le 24\cdot M_k + 4\eta_k \delta_k + \frac{24\eta_k^2 \sigma^2}{|\rvb|}+ 28M  + 24\eta_k d.
        \end{aligned}
    \end{equation}
    where the last inequality follows from Eq~\ref{ineq:2ndmoment_ite_1over2}.
    Hence, the proof is completed.
\end{proof}

\begin{remark}
    According to Lemma~\ref{lem:2ndmoment_bound}, when $L\le 1/5$, We plug the following hyper-parameters settings, i.e.,
    \begin{equation*}
        \begin{aligned}
            \eta_k = \frac{1}{2L},\quad \delta_k\le \frac{Ld}{2},\quad \mathrm{and}\quad |\rvb_k|\ge \frac{6\sigma^2}{d},
        \end{aligned}
    \end{equation*}
    into Eq~\ref{eq:2ndmoment_ite_fin}, then we have
    \begin{equation*}
        M_{k+1}\le  M_k + 5(d+M)\quad \Rightarrow\quad M_K \le M+ K\cdot 5(d+M) \le 6K(d+M),
    \end{equation*}
    which is the second moment bound along the update of Alg~\ref{alg:sps}.
\end{remark}

\section{Auxiliary Lemmas}

\begin{lemma}
    \label{lem:minibatch_var}
    Suppose a function $f$ can be decomposed as a finite sum, i.e., $ f(\vx) = 1/n\sum_{i=1}^{n}f_i(\vx)$ where~\ref{con_ass:var_bound} is satisfied.
    If we uniformly sample a minibatch $\rvb$ from $\{1,2,\ldots, n\}$ which constructs a minibatch loss shown in Eq~\ref{def:minibatch_loss}, then for any $\vx\in\R^d$, we have
    \begin{equation*}
        \EE_{\rvb}\left[\left\|\grad f_{\rvb}(\vx) - \grad f(\vx) \right\|^2\right]\le \frac{\sigma^2}{|\rvb|}
    \end{equation*}
\end{lemma}
\begin{proof}
    For minibatch variance, we have
    \begin{equation*}
        \begin{aligned}
            \E_{\rvb}\left[\left\|\frac{1}{|\rvb|}\sum_{i\in \rvb}  \left(\grad f(\vx) - \grad f_{i}(\vx)\right)\right\|^2\right] = &\frac{1}{|\rvb|^2}\E\left[\sum_{i\in \rvb}\sum_{j\in\rvb}(\grad f_{i}(\vx)-\grad f(\vx))^\top (\grad f_{j}(\vx)-\grad f(\vx))\right]\\
            = & \frac{1}{|\rvb|^2}\E\left[\sum_{i\in \rvb}\left\|\grad f_{i}(\vx)-\grad f(\vx)\right\| ^2\right] = \frac{\sigma^2}{|\rvb|}.
        \end{aligned}
    \end{equation*}
    Hence, the proof is completed.
\end{proof}

\begin{lemma}[Variant of Lemma 10 in~\cite{cheng2018convergence}]
    \label{lem:strongly_lsi}
    Suppose $-\log p_*$ is $m$-strongly convex function, for any distribution with density function $p$, we have
    \begin{equation*}
        \KL{p}{p_*}\le \frac{1}{2m}\int p(\vx)\left\|\grad\log \frac{p(\vx)}{p_*(\vx)}\right\|^2\der\vx.
    \end{equation*}
    By choosing $p(\vx)=g^2(\vx)p_*(\vx)/\mathbb{E}_{p_*}\left[g^2(\rvx)\right]$ for the test function $g\colon \R^d\rightarrow \R$ and  $\mathbb{E}_{p_*}\left[g^2(\rvx)\right]<\infty$, we have
    \begin{equation*}
        \mathbb{E}_{p_*}\left[g^2\log g^2\right] - \mathbb{E}_{p_*}\left[g^2\right]\log \mathbb{E}_{p_*}\left[g^2\right]\le \frac{2}{m} \mathbb{E}_{p_*}\left[\left\|\grad g\right\|^2\right],
    \end{equation*}
    which implies $p_*$ satisfies $m$-log-Sobolev inequality.
\end{lemma}

\begin{lemma}[Theorem 3 in~\cite{chen2022improved}]
    \label{lem:thm3_chen2022improved}
    Assume that $p_*\propto \exp(-f_*)$ satisfies~\ref{con_ass:lsi}. 
    For any $\eta>0$, and any initial distribution $p_1$ the $k$-th iterate $p_k$ of the proximal sampler with step size $\eta_k$ satisfies
    \begin{equation*}
        \KL{p_{k+1}}{p_*}\le \KL{p_k}{p_*}\cdot \left(1+\alpha_* \eta_k\right)^{-2},
    \end{equation*}
    which means it has
    \begin{equation*}
        \KL{p_{k+1}}{p_*}\le \KL{p_0}{p_*}\cdot \prod_{i=1}^k \left(1+\alpha_* \eta_k\right)^{-2}.
    \end{equation*}
\end{lemma}

\begin{lemma}
    \label{lem:init_error_bound}
    Suppose $p_*\propto \exp(-f_*)$ defined on $\R^d$ satisfies $\alpha_*$-log-Sobolev inequality where $f_*$ satisfies~\ref{con_ass:lips_loss}, $p_0$ is the standard Gaussian distribution defined on $\R^d$, then we have
    \begin{equation*}
        \KL{p_0}{p_*} \le \frac{(1+L^2)d}{2\alpha_*}.
    \end{equation*}
\end{lemma}
\begin{proof}
    According to the definition of LSI, we have
    \begin{equation*}
        \begin{aligned}
            \KL{p_0}{p_*}\le & \frac{1}{2\alpha_*}\int p_1(\vx)\left\|\grad \log \frac{p_1(\vx)}{p_*(\vx)}\right\|^2 \der \vx = \frac{1}{2\alpha_*}\int p_1(\vx)\left\|-\vx + \grad f_*(\vx)\right\|^2 \der \vx\\
            \le & \frac{1}{2\alpha_*}\int p_1(\vx) \left(\|\vx\|^2 + L^2\|\vx\|^2\right) \der \vx = \frac{(1+L^2)d}{2\alpha_*} 
        \end{aligned}
    \end{equation*}
    where the second inequality follows from the $L$-smoothness of $f_*$ and the last equation establishes since $\mathbb{E}_{p_0}[\|\vx\|^2]=d$ is for the standard Gaussian distribution $p_0$ in $\R^d$.
\end{proof}

\begin{lemma}[Convexity of KL divergence]
    \label{lem:convexity_KL}
    Suppose $\{q_i\}_{i\in\{1,2,\ldots, n\}}$ and $p$ are probability densities defined on $\R^d$ and $\{w_i\}_{i\in\{1,2,\ldots, n\}}$ are real numbers satisfying
    \begin{equation*}
        \forall i\in\{1,2,\ldots, n\}\quad w_i\in [0,1]\quad \mathrm{and}\quad \sum_{i=1}^n w_i = 1. 
    \end{equation*}
    It has
    \begin{equation*}
        \KL{\sum_{i=1}^n w_i q_i}{p}\le \sum_{i=1}^n w_i \KL{q_i}{p}.
    \end{equation*}
\end{lemma}
\begin{proof}
    We first consider the case when $n=2$, which means it is only required to prove 
    \begin{equation}
        \label{ineq:convexity_2term}
        \KL{\lambda q_1 + (1-\lambda)q_2}{p}\le \lambda \KL{q_1}{p} + (1-\lambda)\KL{q_2}{p}
    \end{equation}
    for any $\lambda \in [0,1]$.
    In this condition, we have
    \begin{equation}
        \label{eq:kl_refor_con}
        \begin{aligned}
            \KL{\lambda q_1 + (1-\lambda)q_2}{p} = &  \int (\lambda q_1(\vx)+ (1-\lambda) q_2(\vx)) \log (\lambda q_1(\vx)+ (1-\lambda) q_2(\vx)) \der\vx  \\
            & - \int (\lambda q_1(\vx)+ (1-\lambda) q_2(\vx)) \log p(\vx)\der \vx.
        \end{aligned}
    \end{equation}
    Since $\varphi(u)\coloneqq u\log u$ satisfies convexity, i.e.,
    \begin{equation*}
        \grad^2 \varphi (u) = u^{-1} > 0\quad \forall u>0,
    \end{equation*}
    which implies 
    \begin{equation*}
        \lambda q_1(\vx)+(1-\lambda)q_2(\vx) \log \left(\lambda q_1(\vx)+(1-\lambda)q_2(\vx)\right) \le  \lambda q_1(\vx)\log q_1(\vx) + (1-\lambda)q_2(\vx)\log q_2(\vx),
    \end{equation*}
    then RHS of Eq~\ref{eq:kl_refor_con} satisfies
    \begin{equation*}
        \begin{aligned}
            \mathrm{RHS} \le &\int \lambda q_1(\vx)\log q_1(\vx)\der \vx - \int \lambda q_1(\vx)\log p(\vx)\der \vx \\
            & +\int (1-\lambda) q_2(\vx)\log q_2(\vx)\der \vx - \int (1-\lambda) q_2(\vx)\log p(\vx)\der \vx = \lambda \KL{q_1}{p}+ (1-\lambda)\KL{q_2}{p}.
        \end{aligned}
    \end{equation*}
    Then, for $n>2$ case, we suppose 
    \begin{equation}
        \label{ineq:convexity_prev}
        \KL{\sum_{i=1}^{n-1} w_i q_i}{p}\le \sum_{i=1}^{n-1} w_i \KL{q_i}{p}.
    \end{equation}
    Then, by setting
    \begin{equation*}
        \overline{q}\coloneqq \frac{\sum_{i=1}^{n-1}w_i q_i}{1- w_n} = \frac{\sum_{i=1}^{n-1}w_i q_i}{\sum_{i=1}^{n-1} w_i},
    \end{equation*}
    then we have
    \begin{equation*}
        \begin{aligned}
            \KL{\sum_{i=1}^{n-1} w_i q_i}{p} = &\KL{(1-w_n)\overline{q}+ w_n q_n}{p}\le (1-w_n)\KL{\overline{q}}{p}+w_n\KL{q_n}{p}\\
            \le & (1-w_n)\sum_{i=1}^{n-1} \frac{w_i}{1-w_n}\KL{q_i}{p}+ w_n\KL{q_n}{p} = \sum_{i=1}^n w_i \KL{q_i}{p},
        \end{aligned}
    \end{equation*}
    where the first inequality follows from Eq~\ref{ineq:convexity_2term} and the last inequality follows from Eq~\ref{ineq:convexity_prev}.
    Hence, the proof is completed.
\end{proof}

\begin{lemma}[Lemma 11 in~\cite{vempala2019rapid}]
    \label{lem:lem11_vempala2019rapid}
    Suppose the density function satisfies $p\propto \exp(-f)$ where $f$ is $L$-smooth, i.e.,~\ref{con_ass:lips_loss}. 
    Then, it has
    \begin{equation*}
        \E_{\rvx\sim p}\left[\left\|\grad f(\rvx)\right\|^2\right]\le Ld.
    \end{equation*}
\end{lemma}

\begin{lemma}[Lemma 5 in~\cite{durmus2019analysis}]
    \label{lem:lem5_durmus2019analysis}
    Suppose the underlying distributions of random variables $\rvx$ and $\rvx+\sqrt{2\tau}\xi$ are $p$ and $p^\prime$ respectively, where $\xi \sim \mathcal{N}(\vzero, \mI)$.
    If $p, p_* \mathcal{P}_2(\R^d)$ and $\E_{p_*}[\log p_*]<\infty$, then it has
    \begin{equation*}
        2\tau \cdot \left(\E_{\rvx \sim p^\prime}\left[\log p^\prime(\rvx)\right] - \E_{\rvx \sim p_*}\left[\log p_*(\rvx)\right]\right) \le W_2^2(p,p_*) - W_2^2(p^\prime, p_*).
    \end{equation*}
\end{lemma}

\begin{definition}[Definition of Orlicz–Wasserstein metric]
    \label{def: orl_was_metric}
    The Orlicz–Wasserstein metric between distributions $p$ and $q$ is
    \begin{equation*}
        W_{\psi}(p,q)\coloneqq \inf_{(\rvx,\rvy)\sim \Gamma(p,q)} \left\|\rvx - \rvy\right\|_{\psi} 
    \end{equation*}
    where 
    \begin{equation*}
        \left\|\rvx\right\|_{\psi}\coloneqq \inf\left\{\lambda>0 : \E \left[\psi \left(\frac{\|\rvx\|}{\lambda}\right)\le 1\right]\right\}.
    \end{equation*}
\end{definition}

\begin{lemma}[Theorem 4.4 in~\cite{altschuler2023faster}]
    \label{lem:thm4.4_altschuler2023faster}
    Suppose $q_* \propto \exp(-g)$ where $g$ is $\mu$-strongly-convex and $L$-smooth. 
    Let $\mathbf{P}$ denote the Markov transition kernel for underdamped Langevin dynamics (ULD) when run with friction paramter $\gamma  = \sqrt{2L}$ and step size $\tau\lesssim 1/(\kappa\sqrt{L})$. 
    Then, for any target accuracy $0<\epsilon\le \sqrt{\log 2/(i-1)}$, any Re\'nyi divergence order $i\ge 1$ and any two initial distributions $q^\prime_0, q_1^\prime\in \mathcal{P}(\R^{2d})$,
    \begin{equation*}
        \mathcal{R}_i(\mathbf{P}^N q^\prime_0\|\mathbf{P}^N q^\prime_1)\le \epsilon^2,
    \end{equation*}
    if the number of ULD iteration is 
    \begin{equation*}
        N\gtrsim \frac{\sqrt{L}}{\mu \tau}\log \left(\frac{2W_{\psi}(q_0,q_*)}{L^{1/2}\epsilon^2 \tau^3}\right),
    \end{equation*}
    where $q_0$ is the marginal distribution of $q_0^\prime$ w.r.t. the first $d$ dimensions and $W_{\psi}$ is defined as Definition~\ref{def: orl_was_metric}.
\end{lemma}
\begin{lemma}[Remark 4.2 in~\cite{altschuler2023faster}] 
    \label{lem:rmk4.2_altschuler2023faster}
    Suppose $q_* \propto \exp(-g)$ where $g$ is $\mu$-strongly-convex and $L$-smooth. 
    We run underdamped Langevin dynamics (ULD) when with friction paramter $\gamma  = \sqrt{2L}$, step size $\tau\lesssim 1/(\kappa\sqrt{L})$ and initialize the distribution with 
    \begin{equation*}
        q^\prime_0 = \delta_{\vx} \otimes \mathcal{N}(\vzero, \mI),
    \end{equation*}
    then it has 
    \begin{equation*}
        W_{\psi}(q_0,q_*) \lesssim \sqrt{d/\mu}+\left\|\vx-\vx_*\right\|
    \end{equation*}
    where $\vx_*$ denotes the minimizer of $g$.
\end{lemma}

\begin{lemma}[Lemma 4.8 in~\cite{altschuler2023faster}]
    \label{lem:lem4.8_altschuler2023faster}
    Suppose $q_*(\vz) \propto \exp(-g(\vz))$ where $g$ is $\mu$-strongly-convex and $L$-smooth. 
    Let $q_*^\prime(\vz,\vv) \propto \exp(-g(\vz)-\|\vv\|^2/2)$.
    Let $\mathbf{P}$ denote the Markov transition kernel for underdamped Langevin dynamics (ULD) when run with friction paramter $\gamma  \asymp \sqrt{L}$ and step size 
    \begin{equation*}
        \tau\lesssim L^{-3/4}d^{-1/2}i^{-1}(T\log N)^{-1/2},
    \end{equation*}
    where $N$ is the total number of iterations and $T=N\tau$ is the total elapsed time.
    Then, 
    \begin{equation*}
        \mathcal{R}_i(\mathbf{P}^N q^\prime_*\| q^\prime_*)\le L^{3/2}d \tau^2 iT.
    \end{equation*}
\end{lemma}

\section{Additional Experiments}
\label{sec:appendix_exp}
Due to space limitations, we defer some experimental details in Section~\ref{sec:exp} to this part.

In our experiments, we fix the number of stochastic gradient usage at $12000$. As the primary goal of our experiments is to verify our theory, we set the inner batch size, i.e., $b_s=1$. Additionally, to be more comparable with SGLD, we set $S^\prime = S-1$. Under these conditions, we primarily focus on tuning three other hyper-parameters. Among them, the inner step size $\tau$ is chosen from the set $\{0.2, 0.4, 0.6, 0.8, 1.0, 1.2, 1.4\}$, which somewhat corresponds to the step size in SGLD. The inner iteration $S$ is chosen from $\{20, 40, 80\}$, which also determines $K=12000/S$. The outer step size $\eta$ is a special hyper-parameter in SPS-SGLD, which corresponds to the coefficient of quadratic regularizer in RGO. As our theory requires it to be larger than $\tau$ in our theory, we choose it from $\{1.0, 4.0, 10.0\}$ in our experiments. The optimal hyper-parameters obtained through grid search are presented in Table~\ref{tab:hyper_params}.
\begin{table}[ht!]
\centering
\begin{tabular}{|l|ccccc|}
\hline
\diagbox{Hyper-Params}{Dimensions} & $d=10$ & $d=20$ & $d=30$ & $d=40$ & $d=50$ \\
\hline
Inner step size $\tau$ & $0.4$ & $0.4$ & $0.4$ & $0.4$ & $0.4$ \\
Inner iteration number $S$ & $40$ & $20$ & $20$ & $80$ & $80$  \\
Outer step size $\eta$ & $4.0$ & $4.0$ & $10.0$ & $10.0$ & $10.0$\\
\hline
\end{tabular}
\caption{Hyper-parameter settings for different dimension tasks based on the grid search.}
\label{tab:hyper_params}
 \end{table}

For the choice of these hyper-parameters, the inner step size somewhat corresponds to the step size in SGLD and can be set in the same order of magnitude. The outer step size $\eta$ is a special hyper-parameter in SPS-SGLD, it requires to be larger than $\tau$ in our theory and experiments. Furthermore, our theory indicates that the inner iteration number, i.e., $S$, is in the same order of magnitude as $\eta/\tau$. This principle of the hyper-parameter choice can be roughly verified by the optimal hyper-parameter settings shown in Table~\ref{tab:hyper_params}. Moreover, we conduct a grid search for $b_s$ under our experimental settings.
\begin{table}[ht!]
\centering
\begin{tabular}{|l|ccccc|}
\hline
\diagbox{Inner batch size}{Dimensions} & $d=10$ & $d=20$ & $d=30$ & $d=40$ & $d=50$ \\
\hline
$b_s=1$ & $0.105$ & $0.063$ & $0.064$ & $0.060$ & $0.055$ \\
$b_s=5$ & $0.143$ & $0.078$ & $0.081$ & $0.074$ &	$0.082$ \\
$b_s=10$ & $0.138$ & $0.092$ & $0.086$ & $0.122$ & $0.110$\\
$b_s=20$ & $0.175$ & $0.107$ & $0.090$ & $0.142$ & $0.117$\\
\hline
\end{tabular}
\caption{The marginal accuracy results under different $b_s$ settings.}
\label{tab:vary_bs_results}
\end{table}
It is worth noting that since we fix the gradient usage, increasing the inner batch size will cause the iteration number to decrease sharply. Consequently, the overall performance in our experiments is worse than that observed with the $b_s=1$ setting.

Although we only provide gradient complexity in our theory, both SGLD and SPS-SGLD are first-order samplers, with the primary computational cost stemming from the number of gradient calculations referred to as gradient complexity in our paper. Consequently, we can assert that SGLD and SPS-SGLD have nearly the same computational cost when the number of gradient calls is fixed, which is set at 12k in our experiments. To substantiate this claim, we present the wall clock time under 12k gradient calls (normalizing SPS-SGLD wall clock time to 1) in the table below.
\begin{table}[ht!]
\centering
\begin{tabular}{|l|ccccc|}
\hline
\diagbox{Algorithms}{Dimensions} & $d=10$ & $d=20$ & $d=30$ & $d=40$ & $d=50$ \\
\hline
SPS-SGLD & $1$ & $1$ & $1$ & $1$ & $1$ \\
SGLD & $0.971$ & $0.968$ & $0.981$ & $0.970$ &	$0.969$ \\
\hline
\end{tabular}
\caption{The wall clock time comparison between SPS-SGLD and SGLD.}
\label{tab:wall_clock_time_sps_sgld}
\end{table}

Moreover, we add some other baselines, e.g., such as AB-SGLD and CC-SGLD proposed by~\citet{das2023utilising}. We selected these variants because they achieved the best theoretical results, apart from our own. With target distributions set as shown in Section 5, the total variation distance performance for different algorithms is presented below. 
\begin{table}[ht!]
\centering
\begin{tabular}{|l|ccccc|}
\hline
\diagbox{Algorithms}{Dimensions} & $d=10$ & $d=20$ & $d=30$ & $d=40$ & $d=50$ \\
\hline
SPS-SGLD & $0.105$ & $0.063$ & $0.064$ & $0.060$ & $0.055$ \\
CC-SGLD & $0.143$ & $0.125$ & $0.105$ & $0.121$ & $0.114$\\
AB-SGLD & $0.154$ & $0.129$ & $0.121$ & $0.120$ & $0.119$\\
vanila-SGLD & $0.176$ & $0.144$ & $0.122$ & $0.131$ & $0.134$ \\
\hline
\end{tabular}
\caption{The marginal accuracy results comparison among SPS-SGLD and other SGLD variants.}
\label{tab:ma_sps_sgld_more}
\end{table}
The results demonstrate that SPS-SGLD significantly outperforms CC-SGLD and AB-SGLD. Furthermore, such SGLD variants can also be incorporated as inner samplers within our framework, potentially enhancing the performance of SPS-type methods even further. Additionally, we would be happy to modify the name to distinguish it from SGLD variants, such as CC-SGLD and AB-SGLD.


\end{document}